\documentclass[sn-mathphys-num]{sn-jnl}% Math and Physical Sciences Numbered Reference Style 
\usepackage{graphicx}%
\usepackage{multirow}%
\usepackage{amsmath,amssymb,amsfonts}%
\usepackage{amsthm}%
\usepackage{mathrsfs}%
\usepackage[title]{appendix}%
\usepackage{xcolor}%
\usepackage{textcomp}%
\usepackage{manyfoot}%
\usepackage{booktabs}%
\usepackage{algorithm}%
\usepackage{algorithmicx}%
\usepackage{algpseudocode}%
\usepackage{listings}%
\usepackage{booktabs}
\usepackage{graphicx}
\newtheorem{Definition}{Definition}

\usepackage{textgreek}
\usepackage{romannum}
\usepackage{tabularx}
\usepackage{longtable}
\usepackage{graphicx}
\usepackage{bm}
\usepackage{subcaption}

\usepackage{microtype}
\usepackage{graphicx}
\usepackage{booktabs} % for professional tables
\usepackage{adjustbox}

\usepackage{hyperref}

\usepackage{amsmath}
\usepackage{amssymb}
\usepackage{mathtools}
\usepackage{amsthm}
\usepackage{multirow}

\usepackage{tabularx} % 请在导言区添加此包
\usepackage{ragged2e} % 提供更好的对齐控制

\usepackage{siunitx}  % 必选：对齐小数点并控制精度
\theoremstyle{thmstyleone}%
\newtheorem{theorem}{Theorem}%  meant for continuous numbers
\newtheorem{proposition}[theorem]{Proposition}% 

\theoremstyle{remark}
\theoremstyle{thmstyletwo}%
\newtheorem{example}{Example}%
\newtheorem{remark}{Remark}%
\newtheorem{lemma}{Lemma}
\newtheorem{assumption}{Assumption}
\newtheorem{corollary}{Corollary}
\theoremstyle{thmstylethree}%

\begin{document}
	
	\pagenumbering{arabic}
	\setcounter{page}{1}
	\pagestyle{plain}

%\title[A Limit Theory of Foundation Models: \\
%A Perspective of Deciphering the Phenomenon of Intelligent Emergence and Scaling Laws]{A Limit Theory of Foundation Models: \\ A Perspective of Deciphering the Phenomenon of Intelligent Emergence and Scaling Laws}

\title[A Limit Theory of Foundation Models: \\
A Mathematical Approach to Understanding Emergent Intelligence and Scaling Laws]{A Limit Theory of Foundation Models: \\ A Mathematical Approach to Understanding Emergent Intelligence and Scaling Laws}


\author[1]{\fnm{Jun} \sur{Shu}}\email{junshu@mail.xjtu.edu.com}

\author[1]{\fnm{Junxiong } \sur{Jia}}\email{ jjx323@mail.xjtu.edu.cn}

\author[1]{\fnm{Deyu} \sur{Meng}}\email{dymeng@mail.xjtu.edu.cn}

\author[1]{\fnm{Zongben} \sur{Xu}*}\email{zbxu@mail.xjtu.edu.cn}
\equalcont{Corresponding author.}

\affil[1]{\orgdiv{School of Mathematics and Statistics}, \orgname{Xi'an Jiaotong University}}
% \orgaddress{\city{Xi'an}, \postcode{910049}, \state{Shaanxi Province}, \country{China}}}

%\affil[2]{ \orgname{Pazhou Lab (Huangpu)}}
	 %\orgaddress{\city{Guangzho}, \postcode{510700}, \state{ Guangdong Province}, \country{China}}}

%\affil[3]{\orgdiv{Pazhou Lab (Huangpu),}, \orgname{Organization}, \orgaddress{\street{Street}, \city{City}, \postcode{610101}, \state{State}, \country{Country}}}


\abstract{
	%Emergent intelligence has played a major role in the modern AI development, manifesting as previously unobserved capabilities enabled by scaling model parameters. While existing studies primarily rely on empirical observations to characterize this phenomenon, a rigorous theoretical framework remains underexplored.
%Our key observation is that intelligence arises as an emergent phenomenon characterized by a transition from finite knowledge to effectively infinite knowledge. This perspective naturally motivates the possibility of utilizing limit tool to formalize intelligent emergence, as that finding irrational numbers in mathematics through the limit of rational numbers. 
%Specifically, we introduce a performance function \(\mathcal{E}(N, P, K)\), dependent on data size $N$, model size $P$ and training steps $K$, to quantify intelligence behavior. Then existence of the limit \(\lim_{N,P,K \to \infty} \mathcal{E}(N,P,K)\) characterizes the conditions for intelligent emergence of foudation models, as well as the stability and practical utility of the underlying architecture. Once the limit exists, the parameter-limit architecture (referred to as the limit architecture) embodies the emergent intelligence, which rationally corresponds to the learning behavior of this limit system. Moreover, we show that the scaling laws of foundation models are approximately governed by the convergence rate of \(\mathcal{E}(N, P, K)\) to its asymptotic value $\mathcal{E}(\infty,\infty,\infty)$.

Emergent intelligence has played a major role in the modern AI, manifesting as previously unobserved capabilities enabled by scaling model parameters. While existing studies primarily rely on empirical observations to characterize these phenomena, a rigorous theoretical framework remains underexplored.
This study attempts to develop a mathematical approach to formulate emergent intergence from the perspective of limit theory. Specifically, we introduce a performance function \(\mathcal{E}(N, P, K)\), dependent on data size $N$, model size $P$ and training steps $K$, to quantify intelligence behavior. We posit that intellgence emerges as a transition from finite to effectively infinite knowledge, and thus the emergent intelligence occurs only when the limit \(\mathcal{E}(\infty,\infty,\infty) = \lim_{N,P,K \to \infty} \mathcal{E}(N,P,K) \) exists, and whenever it exists, the emergent abilities corresponds to the limiting behavior. We define a parameter-limit infinite-dimensional system, called as the limit architecture, to embody the emergent abilities of foundation models, and we show that emergent intelligence rationally corresponds to the learning behaivor of this limit architecture.
By employing the tools from nonlinear Lipschitz operator theory, particularly the characteristic number $\mathrm{Lip}(T)$ of Lipschitz operator, the Lipschitz dual operator theory and the spectral property of compact operators, we prove that the necessary and sufficient conditions for existence of the limit architecture are (\romannumeral1) $\mathrm{Lip}(T_i) \leq 1$ for all basic functional blocks $T_i$, $i > K_0$ for some integer $K_0$; and (\romannumeral2) there exists a generalized projection operator $T$ such that $\|T_i - T\| \leq \epsilon_i$, where $\sum_{i=1}^\infty \epsilon_i < \infty$. Furthermore, we cast the scaling laws of foundation models as the convergence rates of \(\mathcal{E}(N, P, K)\) to its asymptotic value $\mathcal{E}(\infty,\infty,\infty)$. Under general conditions, we establish the scaling laws of foundation models by leveraging the Lipschitz operator theory and empirical process theory.

%are approximately governed by the convergence rate of \(\mathcal{E}(N, P, K)\) to its asymptotic value $\mathcal{E}(\infty,\infty,\infty)$.

%Our key observation is that intelligence arises as an emergent phenomenon characterized by a transition from finite knowledge to effectively infinite knowledge. This perspective naturally motivates the possibility of utilizing limit tool to formalize intelligent emergence, as that finding irrational numbers in mathematics through the limit of rational numbers. 
%Then existence of the limit \(\lim_{N,P,K \to \infty} \mathcal{E}(N,P,K)\) characterizes the conditions for intelligent emergence of foudation models, as well as the stability and practical utility of the underlying architecture. Once the limit exists, the parameter-limit architecture (referred to as the limit architecture) embodies the emergent intelligence, which rationally corresponds to the learning behavior of this limit system. Moreover, we show that 
%

%This limit theory helps reveal that emergent. Furthermore, we derive the scaling law of foundation models by leveraging the tools of Lipschitz operator and covering number.
\qquad   
The established theories show that: (\romannumeral1) the emergent intelligence of foundation models appears subject to the existence of the limit architecture, and it is governed by three key factors-the number of training steps, the size of the training data, and the model size, all of which are in turn crucially determined by properties of the constituting functional blocks $\{T_i\}$; (\romannumeral2) the emergent intelligence occurs only at the critical case when $\mathrm{Lip}(T)=1$. 
This is a finding consistent with those empirically observed by Simon Vock and Christian Meisel \cite{vock2025critical} and M{\"u}ller et al. \cite{muller2025critical}; (\romannumeral3) the emergent intelligence is essentially determined by an infinite-dimensional system, yet can be effectively approximated in practice through a finite-dimensional architecture, offering a solid theoretical evidence for the practical applicability of foundation models in general, and the universal weight subspace hypothesis proposed by Kaushik et al. \cite{kaushik2025universal} in particular. 
(\romannumeral4) The established scaling laws shed some lights on the question posed by Simon et al.~\cite{simon2026there}: 
``can we develop a theory of scaling laws that both explains why power laws arise and predicts their exponents a priori''.
We apply the established theories to the foundation models with GPT-like architectures, showing that the GPT-2-like foundation models do exhibit intelligent emergence, whereas GPT-1-like models do not, and for GPT-2 like models, their scaling law with respect to model size follows a power-law.
We provide a series of experiments to support the rationality and correctness of the established theories.
%Our empirical results all corroborate these theoretical findings.

%By introducing tools from nonlinear Lipschitz operator theory, we establish necessary and sufficient conditions for the existence of the limit architecture, and demonstrate that an infinite-dimensional limit architecture can be effectively approximated by finite composition of stacked basic functional blocks (a finite-dimensional architecture). Empirical results further corroborate these theoretical findings.

 }

\maketitle

\section{Introduction}\label{sec1}

Foundation models~\cite{bommasani2021opportunities} have revolutionized the AI community across a wide range of domains, including natural language processing~\cite{brown2020language,achiam2023gpt,touvron2023llama,team2023gemini,chowdhery2023palm}, computer vision~\cite{ramesh2021zero,chen2020generative,kirillov2023segment,liu2025grounding}, multimodal learning~\cite{wang2024visionllm,liu2024visual,team2023gemini,brooks2024video,radford2021learning}, graph learning~\cite{hou2022graphmae,tian2023heterogeneous}, and programming languages~\cite{moura2021lean}. These models demonstrate remarkable capabilities in understanding human instructions, exhibiting knowledge and logical reasoning, and even aligning with human values, thereby paving a promising path toward general artificial intelligence~\cite{bubeck2023sparks}.
The crux behind these successes lies the phenomenon of emergent intelligence in foundation models~\cite{wei2022emergent,srivastava2023beyond,ganguli2022predictability}. The concept of emergence has its origins in physics and biology~\cite{anderson1972more,hwang2012emergent,johnson2002emergence}. A central insight is captured by the principle of ``More Is Different'' articulated by Nobel laureate Philip Anderson~\cite{anderson1972more}. This idea suggests that ``quantitative changes in a system can lead to qualitative changes in behavior'' adapted from \cite{steinhardt2022future}.

Recently, Wei et al. \cite{wei2022emergent} provided a more clear definition of emergent ability in large language models (LLMs) as: ``abilities that are not present in smaller-scale models but are present in large-scale models; thus they cannot be predicted by
simply extrapolating the performance improvements on smaller-scale models''. Such emergent abilities were first observed in the GPT-3 family \cite{brown2020language}, and further motivated a costly effort to train ever larger models on ever larger datasets, in the hope of unlocking new capabilities.
It is now widely recognized that scaling up foundation models (e.g., data size, model size, and training steps) can consistently lead to improved performance and sample efficiency across a wide range of downstream tasks. Empirically, the generalization loss of foundation models has been observed to follow a predictable power-law relationship with respect to data size, model size and training steps, a property commonly referred to as the neural scaling law~\cite{kaplan2020scaling}.
In practice, to reduce the cost of discovering effective training strategies, researchers typically validate ideas through small-scale experiments and then extrapolate their performance to larger regimes using scaling laws. By leveraging such reliable extrapolation, it becomes possible to iterate efficiently at smaller scales while identifying the approaches most likely to perform optimally in full-scale training. As a result, neural scaling laws have become a standard principle in the development of state-of-the-art foundation models, including Chinchilla \cite{hoffmann2022training}, PaLM \cite{chowdhery2023palm}, GPT-4~\cite{achiam2023gpt}, LLaMA
\cite{touvron2023llama}, Claude \cite{anthropic2024claude}, Gemini \cite{team2024gemini}, Qwen \cite{bai2023qwen}, DeepSeek \cite{guo2025deepseek},
and many others~\cite{zhai2022scaling,bachmann2024scaling,muennighoff2023scaling,besiroglu2024chinchilla,sorscher2022beyond,sardanabeyond2024,ruan2024observational,gadre2024language,du2024understanding}.

Although emergent abilities and neural scaling laws have been empirically validated across a wide range of foundation models, theoretical understanding of these phenomena remains very limited. In this paper, we attempt to propose a rigorous mathematical framework to analyze emergent intelligence and scaling laws.
Our central observation is that intelligence arises as an emergent phenomenon characterized by a transition from finite knowledge to effectively infinite knowledge. This perspective naturally motivates the possibility of using limit theory to formalize emergent intelligence, as that finding irrational numbers in mathematics through limit of rational numbers.
Specifically, we introduce a three-dimensional function $\mathcal{E}(N, P, K)$ to quantify the intelligent behavior (i.e., generalization capability) of foundation models, where \(N\) represents the data size (i.e., the number of data points for training foundation models), \(P\) represents the model size (i.e., the number of model parameters), and \(K\) represents the training steps of optimization algorithms (closely related to the amount of computation).
To understand emergent intelligence, it is essential to study the asymptotic behavior of $\mathcal{E}(N, P, K)$ in the limit as $N \to \infty$, $P \to \infty$, and $K \to \infty$, corresponding to the regime of infinite knowledge, infinite architecture and infinite training steps.
Under this formulation, the emergent intelligence can then be rationally characterized by the value and existence of the limit
\(
\lim_{N,P,K \to \infty} \mathcal{E}(N,P,K).
\)
Furthermore, when this limit exists (denoted by $\mathcal{E}(\infty,\infty,\infty)$), the scaling laws of foundation models can then be explicitly evaluated by the rate at which $\mathcal{E}(N, P, K)$ converges to $\mathcal{E}(\infty,\infty,\infty)$.

%As a result, the scaling law of foundation models could be characterized by the limit behavior/speed of \(\mathcal{E}(N, P, K)\) when \(N \to \infty, P \to \infty, K \to \infty\), which could be further utilized to interpret intelligent emergent ability of foundation models. 

%To depict the foundation models with infinite parameters, we introduce an infinite-dimensional systems called ``limit architecture'', defined as the infinite composition of stacked basic functional blocks. Besides, the properties and behaviors of the limit architecture, as quantified by its generalization performance \(\varepsilon(\infty, \infty, \infty)\), elucidate the fundamental mechanisms underlying intelligent emergence in foundation models.
%We introduce a metric called Lip number $\text{Lip}(T)$ to characterize the existence of the limit architecture, and show the necessary and sufficient condition for its existence in foundation models. 
%Furthermore, we proposed a standard error decomposition of \(\varepsilon(N, P, K)\) - \(\varepsilon(\infty, \infty, \infty)\), dividing it into the weight error \(\varepsilon(N, P, K) -\varepsilon(N, P,\infty)\), architecture error \(\varepsilon(N, P, \infty) - \varepsilon(N, \infty, \infty)\), and the sample error \(\varepsilon(N, \infty, \infty) - \varepsilon(\infty, \infty, \infty)\). 
%We formalize the conditions under which these errors arise and establish corresponding upper bounds, which enable us to derive explicit scaling laws in terms of data, model, and computational budget.

To characterize foundation models with increasing-parameters, we introduce an infinite-dimensional learning system called as the \emph{limit architecture}, that is defined as the infinite composition of stacked basic functional blocks. The properties and behavior of this limit architecture, as quantified by its generalization performance $\varepsilon(\infty, \infty, \infty)$, provide fundamental insights into the mechanisms underlying emergent intelligence in foundation models.
We introduce the Lip number, $\mathrm{Lip}(T)$, as a key metric to characterize the existence of the limit architecture, and establish the necessary and sufficient conditions under which such a limit exists.
Moreover, we propose a standard error decomposition of
\(
\varepsilon(N, P, K) - \varepsilon(\infty, \infty, \infty),
\)
which separates the total error into three components: the \emph{optimization (weight) error}
\(
\varepsilon(N, P, K) - \varepsilon(N, P, \infty),
\)
the \emph{architecture error}
\(
\varepsilon(N, P, \infty) - \varepsilon(N, \infty, \infty),
\)
and the \emph{statistical (sample) error}
\(
\varepsilon(N, \infty, \infty) - \varepsilon(\infty, \infty, \infty).
\)
We further formalize the conditions under which these error components arise and the corresponding upper bounds can be estimated, which enable us to derive explicit scaling laws with respect to data size, model size, and training steps, respectively.

In summary, our theoretical findings are mainly four-fold:

\begin{itemize}
	\item \textbf{The criteria for existence of emergent intelligence.}
	We show that the existence of emergent intelligence critically depends on the properties of the basic functional blocks. Specifically, let $\{T_i\}, i=1,2,\cdots,$ denote the sequence of basic functional blocks in a foundation model. These blocks satisfy the self-mapping and quasi-asymptotic regularity conditions, and there exists a constant $K_0$ such that $\mathrm{Lip}(T_i) \leq 1$ for all $i > K_0$. Moreover, the sequence $\{T_i\}$ converges to a generalized projection operator $T$, i.e., $T_i \to T$, with $\|T_i - T\| \in \ell_1$, which corresponds to a \textit{condensing property}. We then prove that these conditions are necessary and sufficient for the existence of emergent intelligence, and under the conditions,
	the limit architecture of a foundation model can be well approximated by a finite composition of stacked basic functional blocks. The established theory supports positively the universal weight subspace hypothesis proposed by Kaushik et al. \cite{kaushik2025universal}, and the discretization hypothesis proposed by Simon et al. \cite{simon2026there}.
	
	\item \textbf{The existence of scaling laws.} %We use the standard error decomposition tool to analyze the discrepancy between $\mathcal{E}(N,P,K)$ and $\mathcal{E}(\infty,\infty,\infty)$, and illustrate that the scaling laws of training steps and model size follow exponential law, and scaling law of data size follows a power law. %Furthermore, when the basic functional blocks possess favorable properties, the scaling law may transition from a power-law regime to an exponential rate.	
	We employ the standard error decomposition to analyze the discrepancy between $\mathcal{E}(N,P,K)$ and $\mathcal{E}(\infty,\infty,\infty)$, and provide upper bound estimations for the weight error, architecture error and sample error respectively. The obtained bounds reveal that the scaling law of training steps follows an exponential law, and scaling laws of model and data size follow a power law, consistent with those experimentally observed by Kaplan et al.~\cite{kaplan2020scaling} and Hoffmann et al.~\cite{hoffmann2022training}, among others.
	 Additionally, the obtained estimations show that the exponents of scaling laws are essentially determined by the condition number of loss function, the characteristic number $\mathrm{Lip}(T)$ of basic functional blocks and the metric entropy of the hypothesis class. This findings partially answer the question of how to predict the exponents of scaling laws a priori, as recently raised by Simon et al.~\cite{simon2026there}.

	\item \textbf{Theoretical analysis on GPT-like architectures.} We apply the established theories to the models with GPT-like architectures, showing that GPT-2-like models yield a stable architecture, namely, their limit architecture exists, whereas GPT-1-like models do not. We further show that GPT-2-like models do exhibit intelligent emergence, and their scaling law of model size follows a power law.

	\item \textbf{Empirical validation.} We present a series of experiments to justify the correctness and rationality of developed theories. More specifically,  
	we provide experiments to demonstrate that the Lip number does effectively characterizes the properties of basic functional blocks and can serve as a useful criterion for evaluating the usability of network architectures, i.e., yielding a stable architecture. Empirical analysis on state-of-the-art open-source foundation models  supports also the condensing property as a core condition for the existence of emergent intelligence.
\end{itemize}

The rest of the paper is organized as follows. Related work is reviewed in Section~\ref{eqs}. The formulation of emergent intelligence and scaling laws from a limit-theoretic perspective are presented in Section~\ref{sec2}. Section~\ref{sec3} introduces the main theoretical tools, including the standard error decomposition framework, the limit architecture, and the nonlinear Lipschitz operators theory. The necessary and sufficient conditions for existence of the limit architecture are established in Section~\ref{limitsa}. Sections~\ref{sec4} and~\ref{sec5} present the existence conditions for emergent intelligence and the corresponding scaling laws. Experimental validations of the theoretical results are provided in Sections~\ref{es} and~\ref{condensing}. We finally conclude the research with some useful remarks.

\section{Related Work} \label{eqs}

%\subsection{Multidisciplinary Insights for Intelligent Emergence and Inapplicability for Foundation Models}

\subsection{Emergence Ability in Foundation Models.}
Emergence as an idea has long been discussed in domains such as system science, physics and biology \cite{anderson1972more,hwang2012emergent,johnson2002emergence}.
For example, Anderson \cite{anderson1972more} states that ``emergent abilities for complex systems are those properties that (\romannumeral1) emerge at each level of complexity and (\romannumeral2) cannot be understood simply by analyzing the single components `behavior' ''. Hopfield \cite{hopfield1982neural} studies emergent properties in simple-structured neural networks with neurons having elementary properties, and states that ``computational properties of use to biological organisms or the construction of
computers can emerge as collective properties of systems having a large number of simple equivalent components (or neurons)''. For foundation models, Wei et al., \cite{wei2022emergent} add two important concepts based on previous definitions, including the unpredictability and the magnitude of the performance increase, and state that ``An ability is emergent if it is not present in smaller models but is present in larger models. Emergent abilities would not have been directly predicted by
extrapolating a scaling law from small-scale models. When visualized via a scaling curve, emergent abilities show a clear pattern-performance that is near-random until a
certain critical threshold of scale is reached, after which performance increases to substantially above random''. This definition would be the first LLM-specific definition and has been most used in the current academic literatures \cite{ganguli2022predictability,schaeffer2023emergent,lu2024emergent}. 
Recent work \cite{schaeffer2023emergent} shows that the seemingly sharp and unpredictable emergence of capabilities is largely driven by the use of nonlinear evaluation metrics, such as accuracy. In contrast, when performance is assessed using linear metrics (e.g., token edit distance), model improvements tend to be smooth and predictable.

Two different approaches have been adopted to analyze emergent intelligence. In complex system science \cite{holland1998emergence}, specific dynamic models are used to characterize the behavior of emergence. Scaling is often applied in complex systems to demonstrate how ``more is different'', and emergent abilities could be understood through an analogy to phase transitions, e.g., \cite{o2015backbones,kempes2019scales,chen2023sudden,nakaishi2024critical}. These approaches often analyze emergent property with temporal variable, while foundation models primarily focus on variables like model size $P$, data size $D$ and training steps $T$. This renders such approached unsuitable for direct use in the present context. 
Besides, the ``broken symmetry" assumption in physics \cite{anderson1972more}, grounded in a specific ``material basis", also does not lend itself directly to modeling the emergent intelligence of foundation models.

For foundation models, some studies explore to use different tools to help understand emergent abilities, e.g., analyzing loss function and training dynamics \cite{du2024understanding}, quantization\cite{liu2024emergent}, task complexity \cite{wu2024u}, etc. Besides, emergent abilities are manifested through specific capabilities, such as in-context learning \cite{chan2022data,shi2024larger,raventos2023pretraining}, symbolic abstraction \cite{al2025emergence}, RL-enhanced reasoning \cite{guo2025deepseek,snell2024scaling,shao2024deepseekmath}, LLMs-powered agents \cite{huang2024understanding,zhao2024expel}, etc. 
All those approaches are predominantly empirical, with relatively little work on rigorous mathematical foundations. 
In this paper, we make an attempt to develop a preliminary mathematical framework for formalizing emergent intelligence from the perspective of limit theory.

\subsection{Scaling Law in Foundation Models}
Scaling laws have been proposed to analyze the behavior of deep learning across diverse domains and tasks. Investigations into the relationship between generalization error, training data size, and model capacity date back to before the era of deep learning \cite{caponnetto2007optimal}.
For deep neural networks, Hestness et al. \cite{hestness2017deep} observe that the performance of networks improves according to a power-law scaling behavior across a variety of domains, including machine translation, language modeling, image processing, and speech recognition. Importantly, these scaling relationships are shown to persist across model improvements. Recently, Kaplan et al. \cite{kaplan2020scaling} pushes the scale of these studies further into large language models, studying power laws for models up to 1.5B parameters trained on 23B tokens to determine the optimal allocation of a fixed computing budget.
Subsequent studies \cite{hoffmann2022training} revisit this issue and find that Kaplan scaling law significantly underestimates the amount of data required for optimal training, although substantial procedural differences make it difficult to pinpoint the source of this discrepancy. 
As a result, many subsequent LLMs \cite{muennighoff2023scaling,touvron2023llama} have been trained following the Chinchilla scaling laws \cite{hoffmann2022training}.
Besiroglu et al.~\cite{besiroglu2024chinchilla} revisit the methodology of Hoffmann et al.~\cite{hoffmann2022training} and report that the confidence intervals in the original study appear implausibly narrow.
In a related line of research, Porian et al.~\cite{porian2024resolving} examine discrepancies between the scaling laws proposed by Kaplan et al.~\cite{kaplan2020scaling} and those derived under the Chinchilla framework~\cite{hoffmann2022training}.

Since then, researchers have investigated various aspects of scaling in language models.
Hernandez et al.~\cite{hernandez2021scaling} analyze scaling laws governing transfer across distributions in fine-tuning settings.
Aghajanyan et al.~\cite{aghajanyan2023scaling} explore scaling behaviors in multimodal foundation models.
Poli et al.~\cite{poli2024mechanistic} extend scaling studies to hybrid architectures such as Mamba~\cite{gu2024mamba}, demonstrating the effectiveness of this emerging model family. Krajewski et al. \cite{krajewski2024scaling} characterize differences in scaling properties between dense transformers and mixture of expert models. 
Gadre et al. \cite{gadre2025language} show that the loss of over-trained models, trained past compute-optimality, is predictable, and propose a scaling law relating loss to average downstream task performance.
\cite{chen2025revisiting} revisits scaling laws by examining the impact of data quality and training strategies on model performance.
A key limitation of existing scaling laws is their disregard for inference costs, which dominate the long-term expenses of utilizing large models in real-world applications. Sardana et al. \cite{sardana2024beyond} modify the Chinchilla scaling laws to account for both the computational and real-world costs of inference. The growing adoption of LLMs in reasoning systems also highlights the need for scaling frameworks that explicitly account for inference costs \cite{snell2024scaling,brown2024large,luo2024improve,guan2025rstar}.

Recently, several theoretical studies of scaling laws are proposed. The scaling law for linear models with structured (nonisotropic) covariates, including infinite-width kernel regression, have been analyzed using tools from statistical physics and random matrix theory \cite{bordelon2020spectrum,canatar2021spectral,simon2023the,bahri2024explaining,hastie2022surprises,mei2022generalization}. For examples, Mei et al. \cite{mei2022generalization} analyze a linear model with random projections of isotropic covariates. Bahri et al. \cite{bahri2024explaining} consider a linear teacher-student model under a power-law spectrum assumption on the covariates, and they show that the test loss of the ordinary least square estimator decreases following a power law in data size when the model size is infinite. Furthermore, some researchers analyze (precise) generalization error of simple
closed-form estimators such as ridge regression \cite{cui2021generalization,maloney2022solvable,defilippis2024dimension,atanasov2024scaling}.

While they study the limiting effects of model size and data size, some works also study the dependence of training steps.
Bordelon et al. \cite{bordelon2024dynamical} consider a linear random feature model and analyze the test loss of the solution found by (batch) gradient flow. They
focus on the bottleneck regimes where two of the data size, model size and traning steps are infinite, and they show that the risk has a power-law decay in the remaining quantity. Lin et al. \cite{lin2024scaling} consider a linear predictor trained by one-pass SGD with geometrically decaying stepsizes, and they demonstrate that small variance error is due to the implicit regularization effect of SGD. High dimensional limits of SGD have been analyzed with Volterra integral equations in the offline case \cite{paquette2021sgd} or with recursive matrix equations in the online case \cite{varre2021last}. Based on this theoretical foundation, \cite{paquette20244+} uses phase-plane analysis to characterize scaling laws in the compute-limited, infinite-data regime. Li et al. \cite{li2025functional} establish a functional scaling law that captures the full loss trajectory under arbitrary learning rate schedules. Besides, existing scaling analyses for the additive setting \cite{michaud2023quantization,nam2024exactly} explicitly decompose the loss into an independent sum, simplifying the analysis due to task decoupling. 
On the other hand, some researches study the (online) SGD training dynamics and sample complexity of learning a two-layer neural network, e.g., \cite{renemergence,benarous2025learning}. Recently, Olsen et al. \cite{olsen2025sgd} developed a continuous-time SDE framework that connects SGD dynamics to spectral evolution of weight matrices, providing the theoretical explanation for the empirically observed ``bulk+tail'' spectral structure in trained networks. 
However, relatively little theory is known about how these three factors—model size, data size, and training steps—jointly determine scaling behavior in deep architectures. In this paper, we aim to address this gap.

%Arora and Goyal \cite{arora2023theory} derived a theory characterizing how LMs' complex skills can be derived as a composition of base skills.

%For example, in the field of natural language processing, foundation models drive the successes of applications such as ChatGPT, GPT-4 \cite{achiam2023gpt}, and Llama \cite{touvron2023llama}. Similarly, foundation models, have reported considerable success in , graph learning \cite{hou2022graphmae,tian2023heterogeneous}, programming language \cite{moura2021lean} applications, etc.

\section{Formulation and Modelling for Emergent Intelligence and Scaling Laws}\label{sec2}

%\subsection{Why we adopt the perspective of limit theory?}
%
%In this paper, we introduce the limit tools in mathematics to formulating intelligent emergence. In fact, the concept of limits has been a powerful tool for expanding the boundaries of understanding throughout the history of mathematics. For example, the irrational numbers could be represented by the limit of rational numbers, e.g., 
%\begin{align*}
%	e & = \lim_{n \to \infty}\sum_{i=0}^{n} \frac{1}{i!} = \lim_{x \to \infty} \left( 1 + \frac{1}{x} \right)^x; \quad
%	\pi = \lim_{n \to \infty} 4 \sum_{i=0}^{n} \frac{(-1)^i}{2i+1}; \\
%	\sqrt{2} & = \lim_{n \to \infty} x_n, \; x_n = 1 + \frac{1}{1 + x_{n-1}}, \; x_0 = 1.
%\end{align*}
%
%
%
%As it can be seen, the irrational numbers are the emergent behavior exhibited by the limits of rational numbers, and the propety of the irrational numbers could be characterized by the limits of rational numbers. 
%Actually, emergent behavior could be understood as those characteristics/properties that arise beyond what is known or expected. In this paper, we consider that emergence is the process of infinitely approaching the known or expected, yet never fully reaching it. Therefore, emergence represents the limit behavior of known states. Based on this understanding, we explore to formulate intelligent emergence and scaling laws from the perspective of limits.
%

%\subsection{Modelling of Intelligent Emergence and Scaling Laws}

\subsection{Prerequisites}
To obtain a well-trained foundation model, it often needs to pre-determine the network architecture and the size of neural network, collect massive data available for the training process, and then choose a proper optimization algorithm to train the model.
Formally, we can use a triplet $\mathcal{M} = \{f_W, \mathcal{D}, \mathcal{A}\}$ to represent a foundation model, where \(f_W, \mathcal{D}, \mathcal{A}\) denote the network architecture, training dataset and learning algorithm of the foundation models, respectively. Table \ref{ss} summarizes the main notations used in this paper.

\begin{table}[htb]
	\centering
	\begin{tabular}{ll}
		\hline
		\textbf{Symbol} & \textbf{Description} \\
		\hline
		
		$\mathcal{M}$& Foundation model\\
		$f_W$  & Network architecture of foundation model with parameter $W$\\
		$f_0$ &  Initial architecture \\
		$f^*$ & Limit architecture \\
		$\mathcal{D}$  & Training dataset of foundation model  \\ 
		$\mathcal{A}$  & Learning algorithm of foundation model \\
		$N$ & Data size \\
		$P$ & Model size  \\
		$K$ & Training steps  \\
		$T_i$ & $i$-th basic functional block in the network \\
		$\mathcal{L}_2(\mathbb{R}^{m}, \mathbb{R}^{n})$ & The space of square-integrable functions  \\
		$X,D$ &  Uniformly convex Banach space and its closed convex subset\\
		$E,E^*$ &  Bacach space and its dual space. \\
		$\mathbb{D}(T), \mathbb{R}(T)$ & The domain and range of the operator $T$ \\
		$\mathscr{L}(D)$ &  The set of all Lipschitz opertors satisfying self-mapping\\
		$L$ & The number of basic functional blocks \\
		
		$W_k$ & Network parameters at $k$-th algorithm iteration \\
		$W^*$ & Optimal network parameter trained on dataset $\mathcal{D}$ \\
		
		$\ell, \partial\ell$ & Loss function and its gradient with respect to network parameters\\		
		$R(f_W)$ & Population risk of model $f_W$\\
		$E(N, P, T)$ & Excess risk of foundation model  \\
		$\mathcal{E}(N, P, T)$ & Performance function of foundation model \\
		
		$\delta(N,P,K)$ & Neural scaling law of foundation model\\
		
		$\mathrm{Lip}(T)$ & Lip number of operator $T$ \\
		$\mathrm{L}(T)$ & glb-Lipschitz constant of operator $T$ \\
		$T'$ & Fr\'{e}chet derivative of operator $T$\\
		
		$\rho(\cdot)$ & Spectral radius of a linear operator \\
		
		$\mathrm{Fix}(T)$ & Fixed point set of operator $T$ \\
		
		\hline
	\end{tabular} \vspace{2mm}
	\caption{Main notations used in this paper.} \label{ss}
\end{table}

Figure \ref{fig1} shows the model architecture of the GPT-family \cite{radford2018improving,radford2019language,brown2020language}, which implies that a large-scale foundation model is always equipped with periodic structural functional blocks. With this observation, we assume that the model architecture \(f_\mathcal{W}\) consists of $L$ basic functional blocks. Generally, the $i$-th basic functional block could be defined as a functional operator $T_i: \mathbb{R}^{n_i} \to \mathbb{R}^{n_{i+1}}$, where the first input layer satisfies $n_1 = d$. Note that $T_i$ is a nonlinear Lipschitz operator defined over the space $\mathbb{R}^{n_i}$, and maps it into the space $\mathbb{R}^{n_{i+1}}$.
Thus the architecture \(f_{W}\) of this foundation model can be represented as a composition of $L$ functional operators, that is, 
\begin{align} 
f_{W}(x) = T_{L}\circ T_{{L-1}} \circ \cdots \circ T_{1}\circ f_0(x),  \ \forall x \in \mathbb{R}^{d}\label{Eqarch},
\end{align}
with parameters $W = [w_1,w_2,\cdots,w_L]$, where $w_i$ is the parameter of $T_i$, and $f_0$ is the initial architecture, that is the input embedding. 
%For the simplicity, we also let $f_{{K}}=\{T_i\}_{i=1}^K$.
%where $T_{w_{i}}: D \to D, i \in [K]$, and $D$ is a subset of $\mathbb{R}^d$.  
Without loss of generality, we assume the width of the model $f_{{W}}(x)$ is finite, and let $P$ represent the total number of model parameters, i.e., $P = \sum_{i=1}^K |w_i|$, where $|w_i|$ denotes the cardinality of $w_i$. Here, we have dropped explicit dependence of $W$ on $P$ for brevity.
%To avoid abuse of notation, we use $f_{P}(x)$ to denote the foundation models with model parameters $P$ in the following.

The dataset $\mathcal{D} = \{(x_i,y_i)\}_{i=1}^N$, drawn i.i.d. from a probability distribution $\mathbb{P}(x,y)$, which simulates the input and output
of the model to be estimated, and the size of dataset is $|\mathcal{D}|=N$. 
For a given loss function $\ell$, we apply empirical risk minimization principle to define the model \(f_{W}\) with training dataset $\mathcal{D}$, that is
\[
W^* = \arg\min_{{W}\in \mathcal{W}} \ell(f_{W},\mathcal{D}) \overset{\triangle}{=} \frac{1}{N}\sum_{i=1}^{N} \ell(f_{W}(x_i),y_i).
\]
One often employs a gradient descent algorithm to train the model \(f_{W}\) via
\begin{align}\label{eqsss}
	{W}_{k+1} = \mathcal{A}({W}_{k},\mathcal{D},\partial \ell, \gamma ), k=0, 1, \cdots,K, \end{align}
where $\mathcal{A}$ denotes any gradient descent algorithm, ${W}_{k}$ is model parameter of the $t$-th iteration step,
$\partial \ell$ denotes the gradient of loss function with respect to the model parameter ${W}$, and $\gamma$ is the learning rate. 
Here, we assume $W^* = \lim\limits_{K\to \infty} W_K$ under some mild condtions commonly used in the optimization community. Particularly, we assume $\lim_{K \to \infty}\partial \ell(f_{W_K},\mathcal{D})= \partial \ell(f_{W^*},\mathcal{D}) = 0$.

Let $R(f_{W}) = \mathbb{E}_{(x,y)\sim \mathbb{P}(x,y)} \ell(f_{{W}}(x),y)$, and then the excess risk is defined by
\begin{align}
E(N,P,K) = R(f_{{W}_K}) - \inf_{f_{W}}R(f_{{W}}).
\end{align}
By definition, the excess risk $E(N,P,K)$ depends on the data size \(N\), the model size \(P\) and the training steps \(K\).
%In the following, we assume that the optimal model parameter ${W}^*$ exists, i.e., ${W}^* = \arg\min_{{W}} R(f_{{W}}) $, and satisfies $  \mathbb{E}_{(x,y)\sim P(x,y)} \partial \ell(f_{{W}}(x),y) = 0$, and we denote $\partial \ell({W}^*) = 0$ for simplicity.

\begin{figure}[t]
	\centering
	\includegraphics[width=0.8\textwidth]{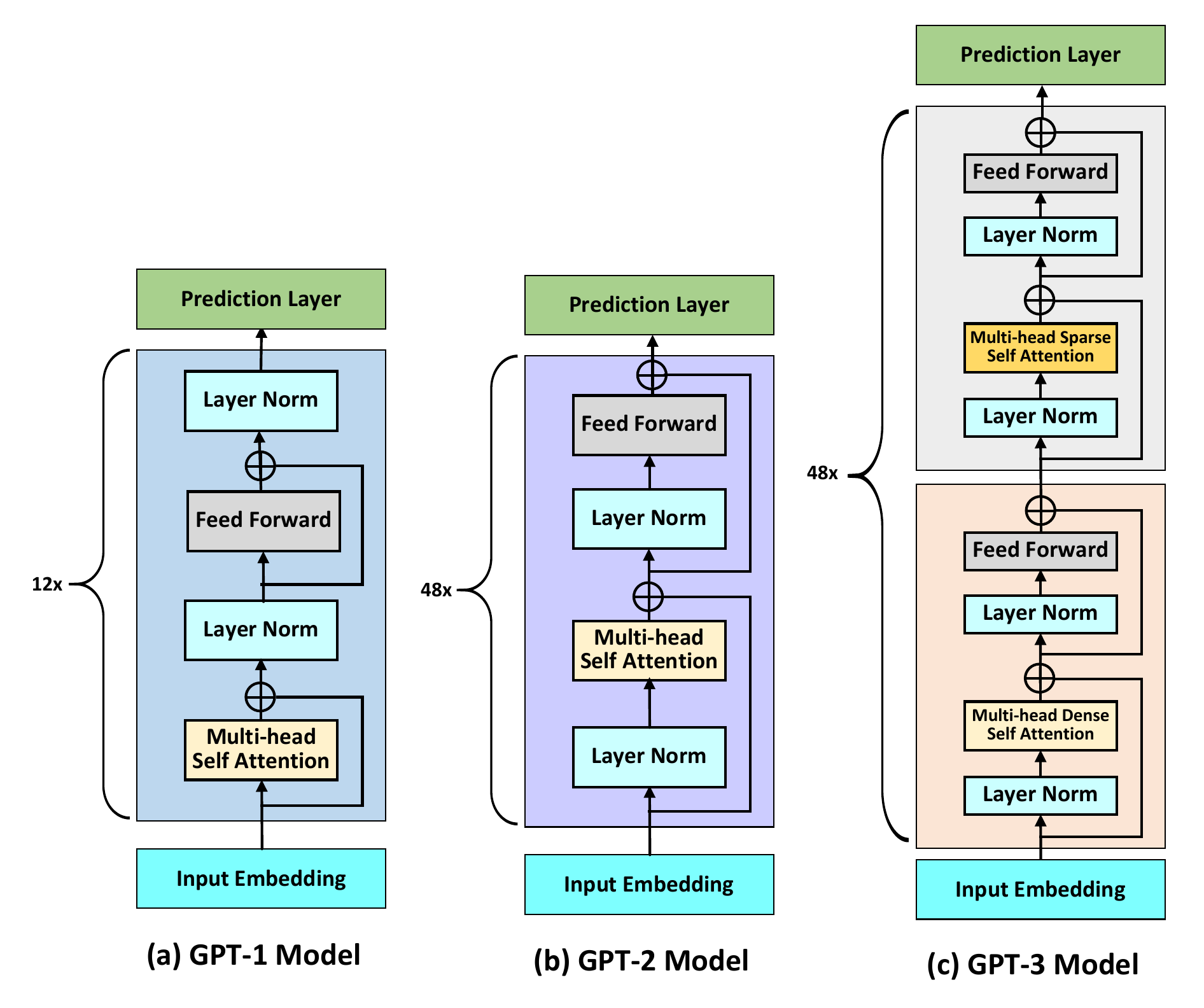}
	\caption{Illustration of GPT family models. (a) GPT-1 model \cite{radford2018improving} is stacked by original Transformer architecture (see Example \ref{ex7}) with 12 layers; (b) GPT-2 model \cite{radford2019language} is stacked by revised Transformer architecture (see Example \ref{ex8}) with 48 layers; (c) GPT-3 model \cite{brown2020language} use the same model and architecture as GPT-2, with the exception that we use alternating dense and locally banded sparse attention patterns in the layers of the transformer, similar to the Sparse Transformer \cite{child2019generating}.}\label{fig1}
\end{figure}

\subsection{Performance Function of Foundation Models}

We assume that a performance function $\mathcal{E}(N,P,K)$ for a fixed foundation model $\mathcal{M} = \{f_W, \mathcal{D}, \mathcal{A}\}$ is closely related to the excess risk function \(E(N,P,K)\), and it satisfies

(\romannumeral1) the measurability condition, i.e.,
\begin{align}\label{aaa}
	\mathcal{E}(N,P,K) = g(E(N,P,K)), \end{align}
where $g:\mathbb{R}^+_0 \to \mathbb{R}^+_0$ is a right-continuous function, and $\mathbb{R}^+_0 = \{x \in \mathbb{R} | x \geq 0\}$.

(\romannumeral2) the interchange condition for the limit function, i.e.,
\begin{align}
	\lim\limits_{N,P,K \to \infty} \mathcal{E}(N,P,K)  & = \lim\limits_{N,P \to \infty} \lim\limits_{T \to \infty} \mathcal{E}(N,P,K),   \label{eqc1}\\
	\lim\limits_{N,P \to \infty} \mathcal{E}(N,P,\infty) &  = \lim\limits_{N \to \infty} \lim\limits_{P \to \infty} \mathcal{E}(N,P,\infty). \label{eqc2}
\end{align}

The above condition (\ref{aaa}) says that the value of performance function $\mathcal{E}(N,P,K)$ for the foundation model is closely (linearly or nonlinearly) related to that of the excess risk function. 
When $g$ is an identity function, the performance function degenerates into the excess risk. In particular, the performance function $\mathcal{E}(N,P,K)$ could be discrete or discontinous. For example, 
	\begin{align*} 
		\mathcal{E}(N,P,K) & = \text{Classification Accuracy} = \mathbb{E}_{(x,y)\sim \mathbb{P}(x,y)} \mathbf{1}(\arg\max_k [f_{W_K} (x)]_k = y) 		
		\\
	& \approx \mathbb{E}_{(x,y)\sim \mathbb{P}(x,y)} {\rm{Pr}} (\arg\max_k [f_{W_K} (x)]_k = y) \\
	& = \mathbb{E}_{(x,y)\sim \mathbb{P}(x,y)} \exp(\log ({\rm{Pr}} (\arg\max_k [f_{W_K} (x)]_k = y)))  \\
	& \approx \mathbb{E}_{(x,y)\sim \mathbb{P}(x,y)} \exp(-\ell(f_{W_K} (x),y)) := \exp(- R(f_{{W}_K}))\\
	& = C  \exp(- E(N,P,K)),
	\end{align*}
	where $C = \exp(\inf_{f_{W}}R(f_{{W}}))$. In this case, the performance function $\mathcal{E}(N,P,K)$ can appear sharp and unpredictable changes with increasing scale, i.e., emergent intelligence may occurs. More related discussions and examples please refer to \cite{schaeffer2024emergent}.

\begin{remark}
	\normalfont
	Observe that
	\begin{align*}
		& \lim\limits_{N,P,K \to \infty} \mathcal{E}(N,P,K)   = \lim\limits_{N,P,K \to \infty} g( E(N,P,K))	=\lim\limits_{N,P \to \infty} g \left(\lim\limits_{K \to \infty}E(N,P,K)\right)
		\\
		& =    \lim\limits_{N,P \to \infty} g \left( E(N,P,\infty)\right) =\lim\limits_{N,P \to \infty} \mathcal{E}(N,P,\infty) =  \lim\limits_{N,P \to \infty} \lim\limits_{K \to \infty} \mathcal{E}(N,P,K), 
	\end{align*}
	where the second equality holds since $g$ is right-continuous (as $K$ increase, $E(N,P,K)$ is non-increasing), and thus it is rational for Eq.(\ref{eqc1}) to hold.

%Observing that
%	\begin{align*}
%	& \lim\limits_{N,P,K \to \infty} \mathcal{E}(N,P,K)   =  \lim\limits_{N,P \to \infty} \lim\limits_{T \to \infty} \mathcal{E}(N,P,K) =\lim\limits_{N,P \to \infty} \mathcal{E}(N,P,\infty) 	\\
%	 & =    \lim\limits_{N,P \to \infty} g \left( E(N,P,\infty)\right) =\lim\limits_{N,P \to \infty} g \left(\lim\limits_{T \to \infty}E(N,P,K)\right) = \lim\limits_{N,P,K \to \infty} g( E(N,P,K))	, 
%	  \end{align*}
%	 where the second equality holds since $g$ is right-continuous (as $T$ increase, $\mathcal{R}(N,P,K)$ is nonincreasing), and thus it is rational for Eq.(\ref{eqc1}) to hold. 
%In fact, $N,P$ are often pre-determined during the execution of the learning process, and we usually firstly take the limit for training steps $T$.
	 
	Similarly, one has
	  \begin{align*}
	& \lim\limits_{N,P \to \infty} \mathcal{E}(N,P,\infty)   = \lim\limits_{N,P \to \infty} g \left( E(N,P,\infty)\right)=\lim\limits_{N \to \infty} g \left( \lim\limits_{P \to \infty} E(N,P,\infty) \right)\\
	&  =  \lim\limits_{N \to \infty} g \left( E(N,\infty,\infty) \right) =  \lim\limits_{N \to \infty} \mathcal{E}(N,\infty,\infty) = \lim\limits_{N \to \infty} \lim\limits_{P \to \infty} \mathcal{E}(N,P,\infty),
	\end{align*}
		 where the second equality holds since $g$ is right-continuous (as $N,P$ increase, $E(N,P,\infty)$ is nonincreasing), and thus it is rational for Eq.(\ref{eqc2}) to hold. 
	
\end{remark}

\subsection{Formulation of Emergent Intelligence and Scaling Laws}
\begin{Definition} [Emergent Intelligence]
If the limit $\lim\limits_{N,P,K \to \infty} \mathcal{E}(N,P,K)$ exists in any sense, we say that the foundation model $\mathcal{M} = \{f_W, \mathcal{D}, \mathcal{A}\}$ displays intelligent emergenence, and $\mathcal{E}(\infty,\infty,\infty) = \lim\limits_{N,P,K \to \infty} \mathcal{E}(N,P,K)$ is the emergent intelligence (performance) exhibited by $\mathcal{M}$.
%
%
% and denote the limit as \(\mathcal{E}(\infty,\infty,\infty)\), i.e.,
%\begin{align}
%	\mathcal{E}(\infty,\infty,\infty) = \lim\limits_{N,P,K \to \infty} \mathcal{E}(N,P,K).
%\end{align}
\end{Definition}

Our key insight is that intelligence emerges as a transition from finite to effectively infinite knowledge, architecture and training steps, and thus we can cast intelligent emergence as the existence of the limit \(\lim_{N,P,K \to \infty} \mathcal{E}(N,P,K)\). 
Note that the existence of such limit also underlies the applicability and stability of foundation models regarding the concerned tasks. In fact, a trained foundation model is applicable only 
when \(N\) is sufficiently large (indicating extensive knowledge acquisition), \(P\) is sufficiently large (indicating a large scale of model parameters), and \(T\) is sufficiently large (indicating sufficiently well-trained for the models), the model can exhibit a very stable learning behavior, hopefully, the expected favorable performance.
This is characterized exactly by the existence of limit \(\lim_{N,P,K \to \infty} \mathcal{E}(N,P,K)\), or by definition, for any $\varepsilon>0$, and for any sufficient large $N,P,K$, one has
 \[\mathcal{E}(\infty,\infty,\infty) - \varepsilon \leq \inf_{N, P,T} \mathcal{E}(N,P,K) \leq \sup_{N, P,T} \mathcal{E}(N,P,K) \leq   \mathcal{E}(\infty,\infty,\infty) + \varepsilon.\]
Here \(\mathcal{E}(\infty,\infty,\infty)\) quantifies the performance of the foundation model after displaying emergence, i.e., achieving an ideal and even unexpected performance. It is noted that when the range of $\mathcal{E}(N,P,K)$ is discrete or discontinous, the performance of the foundation model might exhibit ``abrupt performance improvement'', i.e., emergent abilities.

\begin{Definition} [Scaling Laws]
		We use $\delta(N,P,K)$ to represent the rate at which $\mathcal{E}(N,P,K)$ converges to $\mathcal{E}(\infty,\infty,\infty)$, that is,
	\begin{equation}
		|\mathcal{E}(N,P,K) - \mathcal{E}(\infty,\infty,\infty)| \leq \delta(N,P,K). 
	\end{equation}
	$\delta(N,P,K)$ is called a neural scaling law, which predicts a path for improving the models via scaling up data size $N$, model size $P$, and training steps $K$.

\end{Definition}

Typical forms of scaling laws contain exponential law, power law, etc. Taking the scaling law of $P$ as an example, if $\mathcal{E} = A^P, A \in (0,1)$, we say it is an exponential law (as $A^P = \exp(-P\log(1/A))$); and if $\mathcal{E} = \frac{1}{P^C}, C>0$, it is a power law.
The scaling laws characterize the relationship between the performance of a foundation model and the scaling of key factors, like data size $N$, model size $P$, and training steps $K$. 
When fixing two of $N,P,K$, the convergence rate refers to the scaling law for the remaining factor. For example, the scaling law of training steps is the convergence rate of $K$ when $N,P$ are fixed, the scaling law of model size is the convergence rate of $P$ when $N,K$ are fixed, and the scaling law of data size is the convergence rate of $N$ when $P,K$ are fixed.

In essence, emergent intelligence in foundation models appears subject to existence of the limit
\(
\lim_{N,P,K \to \infty} \mathcal{E}(N,P,K).
\)
Whenever this limit exists, the scaling laws correspond to the characterization of the convergence rates governing how $\mathcal{E}(N,P,K)$ approaches its asymptotic limit $\mathcal{E}(\infty,\infty,\infty)$.

%In a nutshell, the existence of intelligent emergence for foundation models turns into the problem of discriminating the existence of the limit $\lim\limits_{N,P,K \to \infty} \mathcal{E}(N,P,K)$; if the limit exists, then the scaling laws turns into the estimation of convergence rates for performance function $\mathcal{E}(N,P,K)$ approaching to the limit $\mathcal{E}(\infty,\infty,\infty)$.

%we formulate  and  in  as the problem of the existence of the limit and the estimation of convergence rates for performance function $\mathcal{E}(N,P,K)$.

\section{Research Tools} \label{sec3}

\subsection{Standard Error Decomposition}
To study existence of the limit of performance function $\mathcal{E}(N,P,K)$, we employ the following error decomposition framework commonly adopted in statistical learning theory (SLT) community \cite{mohri2018foundations},
\begin{align} \label{eqdecom}
	& \mathcal{E}(N,P,K) - \mathcal{E}(\infty,\infty,\infty)   \notag \\
	= &	\underbrace{\mathcal{E}(N,P,K) - \mathcal{E}(N,P,\infty)}_{\text{weight error}} + \underbrace{\mathcal{E}(N,P,\infty) - \mathcal{E}(N,\infty,\infty)}_{\text{architecture error}} + \underbrace{\mathcal{E}(N,\infty,\infty) - \mathcal{E}(\infty,\infty,\infty)}_{\text{sample error}}.
\end{align}
As seen, to estimate $\mathcal{E}(N,P,K) - \mathcal{E}(\infty,\infty,\infty)$, one can bound the weight (optimization) error, the architecture error and the sample (statistical) error, respectively, and then combine these bounds together to estimate the final error.

%\begin{remark}
Traditional SLT framework assumes that the function class $\mathcal{F} = \{f_W, W \in \mathcal{W}\}$ is given, where $f_W$ is defined as in Eq.(\ref{Eqarch}). This implies that function $f_W$ in $\mathcal{F}$ exists, and the model size is finite. Here we consider the model size $P \to \infty$, leading to some essential difficulties. Specifically, the architecture error is subject to the existence of $f_W$ when model size $P \to \infty$, and the convergence rate of the model with finite parameters approximates the model with infinite parameters. Existing works on such study mainly consider shallow neural networks with infinite parameters (e.g., 2-layer neural network with infinite width \cite{jacot2018neural,mei2018mean}), while the theoretical analysis of deep neural networks with infinite depth has been lack of research. The central difficulty is whether a neural network with infinite depth defines a stable architecture. The key challenge lies in analyzing the infinite-depth limit, namely, whether infinite-depth neural networks constitute a meaningful and stable architecture.
Besides, the sample error quantifies the deviation between the solution learned from finite data and the infinite-data associated with the limit architecture. Existing studies on this side primarily analyze the estimation error of finite-parameter models, whereas the infinite-parameter regime remains largely unexplored \cite{de2023convergence,liu2024deep}. In this work, we provide a theoretical analysis of the architecture error and the sample error in Sections \ref{modelscaling} and \ref{datascaling}, respectively. Compared with classical machine learning theory, the new theories to be developed will provide convergence analysis for infinitely deep networks, along with corresponding generalization bounds.

%\end{remark}

%\begin{remark}
%	The existence of $\mathcal{E}(N,P,\infty)$ is guaranteed by the existence of optimal weight $\mathcal{W}^*$ of foundation model.
%\end{remark}

\subsection{Limit Architecture of Foundation Models}

Without loss of generality, we suppose that the width of $f_W$ remains finite, while its depth tends to infinity as the model size goes to infinity.
To investigate existence of the limit $\mathcal{E}(\infty,\infty,\infty)$, we first study whether a foundation model exists as its size tends to infinity. We define a limit architecture as the final convergent foundation model when the model size tends to infinity.

\begin{Definition}[Limit Architecture of a Foundation Model]
	If the limit $T^* = \lim_{L \to \infty} \prod_{i=1}^{L} T_i$ exists, then we call $f^* = T^*\circ f_0$ a limit architecture of the foundation model. 
	
	%$\lim_{P \to \infty} f_{{W}}(x)$ exists, let $f^*(x) = \lim_{P \to \infty} f_{{W}}(x)$, and we call $f^*$ %where $W$ denotes the model parameter of the limit architecture.   %and $f_0$ in Eq.(\ref{Eqarch}) the initial architecture.
\end{Definition}

If the limit architecture $f^* = \left(\prod_{i=1}^\infty T_i\right) f_0$ exists, we say that the foundation model induces a \textbf{stable architecture}, which defines an infinite-dimensional learning system
\[
y = f^*(x), \quad x \in \mathbb{R}^d,\; y \in \mathbb{R}^n.
\]
It has infinite-dimensional parameters. The sample error defined in Eq.~(\ref{eqdecom}) characterizes the deviation between the solution obtained from finite amount of data and the oracle of the infinite system. We will present a upper bound estimations on such sample error in Sections \ref{datascaling}.

By the above definition, the limit architecture $f^*$ of a foundation model could be represented as 
\begin{align}\label{eqlimit}
	f^* = \left(\prod_{i=1}^\infty T_i\right) f_0,
\end{align}
or equivalently, defined by
\[f_{i+1} = T_{i+1}f_{i}, i=0,1,2,\cdots, \]
where $f_0$ is the initial architecture, and $T_i \in \mathcal{L}_2(\mathbb{R}^{n_i}, \mathbb{R}^{n_{i+1}})$ is the nonlinear Lipschitz operator from function $f_i \in \mathbb{R}^{n_i}$ to function $f_{i+1} \in \mathbb{R}^{n_{i+1}}$.
Thus the existence of limit architecture is convert to the convergence of infinite operator products $\left(\prod_{i=1}^\infty T_i\right)$ in $L^2$ space, and, it is seen that whenever the limit architecture exists, $f^*$ is nothingelse but the common fixed points of the countable family of nonlinear operators $\{T_i\}_{i=1}^\infty$. This motivates us to design a feasible architecture in practice through searching the common fixed points of a family of nonlinear operators.
In the following, we introduce a mathematical tool to investigate the convergence of infinite operator products.

%From examples \ref{exlinear}, we can see that the necessary condition for the existence of the limit architecture $f^*_{\mathcal{W}}$ is $\|T_i\| \in \mathcal{L}_2(\mathbb{R}^{n_i}, \mathbb{R}^{n_{i+1}})$. Based on this observation, we assume $\mathcal{N} = \{f_{\mathcal{W}} \in \mathcal{L}_2, \mathcal{W} = \{w_i\}_{i=1}^\infty\}$, where $w_i$ denote the model weight of basic block $T_i$. In the following, we define the limit architecture in the function space $\mathcal{N}$. However, under what conditions does the limit architecture of $f_{\mathcal{W}}$ exist within $\mathcal{N}$? The next section provides a criterion for this existence.

\subsection{Nonlinear Lipschitz Operators} \label{lips}

In the present study, we assume $X$ is a uniformly convex Banach space, and $D$ is a nonempty closed, bounded and convex subset of $X$, which are typically chosen as $\mathbb{R}^{d}$ when analyzing neural architectures. Assume that $\mathcal{L}_2(\mathbb{R}^{m}, \mathbb{R}^{n})$ is an \( L^2 \)-integrable function space 
defined on domain \(\mathbb{R}^m\) with range in \( \mathbb{R}^n \). In other words, when a function $T \in \mathcal{L}_2(\mathbb{R}^{m}, \mathbb{R}^{n}), x \in \mathbb{R}^{m}$, then $T(x) = (T_1(x),T_2(x),\cdots,T_n(x))$ satisfies $T_i(x) \in \mathcal{L}_2(\mathbb{R}^{m}, \mathbb{R})$ for each $i$. The basic block $T_i$ of neural network $f_W$ could be studied in such an \( L^2 \)-integrable function space.

 We say that an operator $T: D \to X$ is a Lipschitz operator, if there exists a distance function $d(\cdot,\cdot)$ such that
%defined on $\mathcal{L}_2(\mathbb{R}^{m}, \mathbb{R}^{n})$,  
\begin{align} 
	\mathbb{K} d(x,y) \leq d(Tx,Ty) \leq \mathbb{L} d(x,y), \forall x,y \in D,
\end{align}
where $\mathbb{K} \geq 0,  \mathbb{L}<\infty$ are constants independent of $x$ and $y$. We use \( \mathscr{L}(D,X) = \{ T : D \to X \mid T \text{ is a Lipschitz operator} \} \) to denote the family of all Lipschitz continuous operators. In our study, we assume that $T$ satisfies the self-mapping condition, i.e., $T \in \mathscr{L}(D)$, which maps \( D \) into itself.

%\begin{figure}[t]
%	\centering
%	\includegraphics[width=0.5\textwidth]{./figure/nonexpansive}
%	\caption{Geometric Interpretation of the quasi-asymptotic regularity property (Eq.(\ref{s})) of a nonexpansive operator.}\label{fignon}
%\end{figure}

Let
\begin{align*}
	L(T) =& \sup \left\{ \frac{d(Tx,Ty)}{d(x,y)} \Big| x \neq y, x,y \in \mathbb{D}(T)   \right\}
	%\ell(T) =& \inf \left\{ \frac{d(Tx,Ty)}{d(x,y)} \Big| x \neq y, x,y \in \mathbb{D}(T)  \right\},
\end{align*}
be the lub-Lipschitz constants of $T$ \cite{soderlind1986bounds}, where $\mathbb{D}(T) $ is the domain of $T$.
If $L(T)\leq 1$, $T$ is said to be a non-expansive operator; if $L(T)< 1$, $T$ is a contraction operator. Let $\mathrm{Fix}(T)=\{x\in D: Tx = x\}$ denote the fixed point set of $T$. As is well known, if $T\in \mathscr{L}(D)$ is nonexpansive, then $T$ has fixed points \cite{Browder1965,Gohde1965,Kirk1965}. 
We say that an operator $T: D \to X$ is a quasi-Lipschitz operator, if there exists $\mathbb{K}' \geq 0,  \mathbb{L}'<\infty$ such that
%defined on $\mathcal{L}_2(\mathbb{R}^{m}, \mathbb{R}^{n})$,  
\begin{align*} 
	\mathbb{K}' d(x,p) \leq d(Tx,Tp) \leq \mathbb{L}' d(x,y), \forall x\in D, p \in \mathrm{Fix}(T).
\end{align*}
Especially, if $T$ satisfies $$d(Tx,p) \leq d(x,p), \forall x\in D, p \in \mathrm{Fix}(T),$$ we say $T$ is a quasi-non-expansive operator. A quasi-non-expansive operator is said to be of \textbf{quasi-asymptotic regularity} property, if there exist a strictly increasing function $\psi: \mathbb{R}_0^+ \to \mathbb{R}_0^+$ with $\psi(0)=0$, such that
\begin{align} \label{s}
	d(Tx,p) \leq d(x,p)- \psi(d(x,Tx)), \forall x\in D, p \in \mathrm{Fix}(T).
\end{align}
Geometrically, the quasi-asymptotic regularity property requires that the action of $T$ not only does not increase the distance from \(x\) to \(p\), but can reduce this distance by at least \(\psi(d(x,Tx))\). 

It should be noted that the quasi-asymptotic regularity property is a mild and commonly satisifed condition for a Lipschitz operator in applications. For example, when $T: X \to D$ is a projection operator in a Hilbert space $X$, i.e., defined by 
\[Tx = \arg\min_{y \in D} \|x-y\|, \forall x \in X,\] it then satisfies that
\[ \|Tx-p\|^2 \leq \|x-p\|^2 -\|x-Tx\|^2, \forall x \in X, p \in D,  \] 
and hence $T$ is of quasi-asymptotic regularity property with $\psi(t) = t^2$. More generally, if $T\in \mathcal{L}(D)$ is any non-expansive operator, its average version $T_{\alpha} = \alpha I +(1-\alpha)T, 0<\alpha<1,$ satisfies 
\[
\|T_{\alpha}x - p\|^2 \leq \|x-p\|^2 - \frac{(1-\alpha)}{\alpha} \|x-T_{\alpha}x\|^2.
\]
Hence, it is of quasi-asymptotic regularity property with $\psi(t) =\frac{1-\alpha}{\alpha}t^2$. 

The residual connection type operator $T=I+A$ can be written as
\[T = I + A = (1-\alpha)I + \alpha\left(I+\frac{1}{\alpha}A\right), 0<\alpha<1. \]
If $A$ is a dissipative operator \cite{lumer1961dissipative}, then $I+\frac{1}{\alpha}A$ is a non-expansive operator. 
Consequently, $T$ is an average version of non-expansive operator $I+\frac{1}{\alpha}A$, and thus $T$ is a non-expansive operator and of quasi-asymptotic regularity property. The residual connection commonly eppears in the basic functional blocks of foundation models, e.g., see Figure \ref{fig1}. And in Sections \ref{eqmodel} and \ref{limitsa}, we will further demonstrate the crucial role of such property of the residual connection type operator in assurance of intelligence emergence of the models.
%We will assume such quasi-asymptotic regularity property throughout this study whenever a Lipschitz operator is discussed.

%Acturally, let $x = x_n = T^n x_0$, we have 
%\[d(T^{n+1}x_0,p)  \leq d(T^nx_0,p)- \psi(d(T^nx_0,T^{n+1}x_0)), \forall x_0\in D, p \in \mathrm{Fix}(T).\]
%Then one has $\sum_{i=0}^n \psi(d(T^ix_0,T^{i+1}x_0)) \leq \sum_{i=0}^n [d(T^ix_0,p)-d(T^{i+1}x_0,p)] \leq d(x_0,p)$, which shows that $\sum_{i=0}^\infty \psi(d(T^ix_0,T^{i+1}x_0))\leq \infty$. Thus $\psi(d(T^nx_0,T^{n+1}x_0)) \to 0$. Since $\psi$ is a strictly increasing function, we have $d(T^nx_0,T^{n+1}x_0) \to 0, \forall x_0 \in D$. This implies the asymptotic regularity property of $T$. Moreover, asymptotic regularity is a necessary condition for convergence to a fixed point: if \(x_n\to p\in\operatorname{Fix}(T)\), then \(d(x_n,Tx_n)\to0\).
%
%
%Generally, the projection operator $T=P$ with $P^2=P$ and averaged operator $T = (1-\alpha) I + \alpha S, 0<\alpha<1$ where \(S\) is nonexpansive, satisfy the quasi-asymptotic regularity property (Eq.(\ref{s})). The proof can be found in Appendix \ref{regu}.

\subsubsection{Characteristic Number $\mathrm{Lip}(T)$}
\begin{Definition} [Lip number and Dahlquist number of a nonlinear Lipschitz operator]
Suppose $T \in \mathscr{L}(D)$ and let
%$T:D\to D \in \mathcal{L}_2(\mathbb{R}^{m}, \mathbb{R}^{m})$ is a self-mapping operator, let
\begin{align*}
	\mathrm{Lip}(T) = \lim\limits_{n\to \infty}[L(T^n)]^{\frac{1}{n}}, \quad
	m(T) = \lim\limits_{h\to 0^+} \frac{1-L(I-hT)}{h}.
\end{align*}
We call $\mathrm{Lip}(T)$ and $m(T)$ the Lip number and Dahlquist number of $T$, respectively. 
\end{Definition}
The Lip number of $T$ was early introduced by Xu and Wang \cite{xu1996} in 1996, and Dahlquist number of $T$ was introduced by Germund Dahlquist \cite{dahlquist1958stability} for linear operators in 1958, and extended to the nonlinear Lipschitz operators in 1986 \cite{soderlind1986bounds,Wang1996,soderlind2006logarithmic}. We can verify several equivalent characterizations of $\mathrm{Lip}(T)$ as in Proposition \ref{prop1}. Based on this, it would become more direct to estimate values of $\mathrm{Lip}(T)$. 
%\begin{remark}
%Lip	constant $\mathrm{Lip}(T)$  was introduced by Xu and Wang \cite{xu1996} in 1996.
 %and $\mu(T)$ was introduced by Germund Dahlquist \cite{dahlquist1958stability} in 1958, which was further extended to nonlinear maps, please refered to \cite{soderlind1986bounds,Wang1996,soderlind2006logarithmic} for more details . 
 %In particular, when $T$ is a linear operator, $\mathrm{Lip}(T)$ is the spectral radius of $T$, i.e., $\mathrm{Lip}(T) = \rho(T)$.
 % and $\mu(T)$ is the smallest eigenvalue of $T$ (Please refer to Example \ref{ex2}). 
%\end{remark}

%The following proposition provides an equivalent characterization of $\mathrm{Lip}(T)$. 

\begin{proposition} [$\mathrm{Lip}(T)$ evaluation] \label{prop1}
Suppose $T \in \mathscr{L}(D)$ is continuously differentiable. Let $T'$ denote Fr\'{e}chet derivative of $T$, and $\rho(\cdot)$ is the spectral radius of a linear operator. Then 

(\romannumeral1)  If $\mathrm{Lip}(T)<1$, then $\mathrm{Lip}(T) = \rho(T'(x^*))$, where $x^*$ is the unique fixed point of $T$.

(\romannumeral2)  If $\mathrm{Lip}(T)\leq 1$, then \[\mathrm{Lip}(T) = \sup_{x^* \in \mathrm{Fix}(T)}\rho(T'(x^*)),\] where $\mathrm{Fix}(T) = \{x \in D| Tx = x \}$ is the fixed point set of $T$.

(\romannumeral3)  Generally, one has \[\mathrm{Lip}(T) = \sup_{x\in D}  \rho(T'(x)).
\]
\end{proposition}

The proof of Proposition \ref{prop1} is presented in Appendix \ref{secA2}. Propositions \ref{prop1} shows that the Lip number $\mathrm{Lip}(T)$ is a natural extension of the spectral radius of linear operators.
It can be seen that $\mathrm{Lip}(T)$ is an invariant constant with respect to any strong equivalence distance, even if it is defined by means of a specific norm. Thus $\mathrm{Lip}(T)$ is independent of the strong equivalence distance, particularly independent of the choice of norm.
We will apply $\mathrm{Lip}(T)$ to characterize existence of the limit architecture of foundation models. The examples presented in Section 
\ref{eqmodel} will illustrate that the number $\mathrm{Lip}(T)$ of various basic functional blocks $T_i$ of specific foundation models can be very well estimated. 
%\begin{proposition} [Estimation for $\mu(T)$] \label{prop2}
%	Suppose $T \in \mathcal{L}(D)$ is continuously differentiable, then 
%	\[\mu(T) = \inf_{x ,y\in D, x \neq y} \left\{ \frac{\langle Tx-Ty,x-y \rangle}{||x-y||^2}\right\}.\]
%\end{proposition}

%The proof could be found in Appendix . Propositions \ref{prop2} illustrates that the $\mu(T)$ is the extension of the smallest eigenvalue of linear operators in some sense. Besides, it could be considered as the minimal monotone constant of nonlinear operator \(T\). We will use $\mathrm{Lip}(T)$ and $\mu(T)$ to characterize the existence of the limit architecture of foundation models.

%\subsubsection{Quasi-Asymptotic Regular Operator}

\subsubsection{Lipschitz Dual Operator}

%Let $E$ be a Banach space, $E^*$ be its dual space, and $D$ a closed subset of $E$.
%It is straightforward to verify that \( L(\cdot) \) is a seminorm on \( \mathscr{L}(D) \).
It is straightforward to verify that \( L(\cdot) \) is a seminorm on \( \mathscr{L}(D,X) \). Specifically, if \( T \) is a bounded linear operator from \( X \) to \( X \), then the restriction of $T$ on $D$ satisfies
\( T_D \in \mathscr{L}(D, X) \) on \( D \), and \( L(T_D) \leq \|T\| \), where $\|\cdot\|$ denote any norm of $T$. 
 Lemma \ref{lemmas} below shows that \( L(\cdot) \) acturally is also a norm under certain constraint, the proof of of which is given in Appendix \ref{dual}.
%Specifically, if \( T \) is a bounded linear operator from \( E \) to \( E \), then the restriction of $T$ on $C$ satisfies that
%\( T_C \in \mathscr{L}(C, E) \) on \( C \), and \( L(T_C) \leq \|T\| \).

\begin{lemma} \label{lemmas}
	Let \( 0 \in D \), and \( \mathscr{L}_0(D) = \{ T \in\mathscr{L}(D) \mid T(0) = 0 \} \). Then \( L(\cdot) \) is a norm on \( \mathscr{L}_0(D) \), and \( (\mathscr{L}_0(D), L(\cdot)) \) is a Banach space.
\end{lemma}

In the following, we assume \( 0 \in D \) without loss of generality. With the Banach space $\mathscr{L}_0(D)$, the Lipschitz dual of a Lipschitz operator can be defined.
\begin{Definition}[Lipschitz dual space] \label{dess}
	Let \( \mathbb{K} \) be a number field, the Banach space \(\mathscr{L}_0(D, \mathbb{K}) \) is called the Lipschitz dual space of \( D \),  and denote it by \( D_L^* \).
\end{Definition}

\begin{Definition}[Lipschitz dual operator]\label{des}
	Let \( T \in \mathscr{L}_0(D) \), and define the mapping \( T_L^*: D_L^* \to D_L^* \) by:
	\[
	(T_L^* f)(x) = f(Tx), \quad \forall f \in D_L^*, \, x \in D,
	\]  
	then \( T_L^* \) is called the Lipschitz dual operator of \( T \).
\end{Definition}
\begin{lemma}
	For any \( x \in D \),  
	\(
	\|x\| = \sup_{f \in D_L^*, L(f) \leq 1} |f(x)|.
	\)
\end{lemma}
The notions of Lipschitz dual space and dual operator were early introduced by Peng and Xu \cite{ji1999novel,ji2002novel} in 1999 \& 2002. We can verify the following important consequence of Definitions \ref{dess} and \ref{des} as presented in Proposition \ref{props} below. The proof is also presented in Appendix \ref{dual}.

\begin{proposition} \label{props}
	Let \( T \in \mathscr{L}_0(D) \), then \( T_L^* \) is a bounded linear operator on \( D_L^* \), and \( \rho(T_L^*) = \mathrm{Lip}(T) \).
\end{proposition}

%Proposition \ref{props} shows that the $\mathrm{Lip}(T) $ could be computed by $\rho(T_L^*)$. We could analyze the bounded linear operator \( T_L^* \) to investigate nonlinear operator $T$. This provide a graceful tool that bridge relationship between \( T_L^* \) and \(T\).

Proposition \ref{props} shows that $\mathrm{Lip}(T)$ can be computed by $\rho(T_L^*)$. This enables us to study the nonlinear operator $T$ through the bounded linear operator \( T_L^* \). In this way, we obtain a powerful tool for establishing the relationship between \( T_L^* \) and \(T\).

\subsubsection{Spectral Decomposition Theorem of Compact Operators}

\begin{lemma} [Decomposition Theorem of Compact Operators, p. 421 in \cite{riesz2012functional}] \label{lemma5}
	Let \( T \) be a compact operator of Banach space \(E\), and let \( \sigma \) and \( \bar{\sigma} \) be two complementary isolated parts of its spectrum; we permit \( \sigma \) or \( \bar{\sigma} \) to be empty. We can then decompose the space \( E \) into the vector sum of two linearly independent subspaces, \( \mathcal{M}_\sigma \) and \( \mathcal{M}_{\bar{\sigma}} \), each of which is transformed by \( T \) into itself, and with the property that the transformation \( T \) restricted to \( \mathcal{M}_\sigma \) or \( \mathcal{M}_{\bar{\sigma}} \) has its spectrum equal to \( \sigma \) or \( \bar{\sigma} \), respectively.  	
	The parallel projection of \( E \) onto \( \mathcal{M}_\sigma \) in the direction of \( \mathcal{M}_{\bar{\sigma}} \) is equal to the integral  	
	\[
	P_\sigma = -\frac{1}{2\pi i} \int_{\partial D} R_z \, dz,
	\]	
	taken along the boundary of an arbitrary domain \( D \) which is admissible with respect to \( T \) and such that \( \sigma = \sigma(T) \cap D \), where $R_z$ is a holomorphic function of $z$. 
	%We have \( P_\sigma = I \), \( P_\sigma = 0 \) if and only if \( \sigma \) coincides with \( \sigma(T) \) and \( \bar{\sigma} \) is empty.  
\end{lemma}

%\begin{proposition}\label{propq}
%	Let \( T \in \mathscr{L}_0(D) \), then \( T_L^* \) is a completely continuous operator or compact operator.
%\end{proposition}
Suppose $ \sup\limits_{\lambda \in \sigma(T_L^*)} |\lambda| = \rho(T_L^*) = \mathrm{Lip}(T)=1$, where $\sigma(T_L^*)$ is the spectrum set of linear operator $T_L^*$. If $T$ is quasi-asymptotic regular, we can derive that $\sigma(T_L^*) \cap \{\lambda \in \mathbb{C}: |\lambda| = 1\}=\{1\}$ (see Appendix \ref{b31}). In this case, let $\sigma = {1}, \bar{\sigma} = \sigma(T_L^*) - \sigma$. Since \( T_L^* \) is a bounded linear operator on \( D_L^* \), we have the following decomposition based on Lemma \ref{lemma5}, 
\[
D_L^* = \mathcal{M}_\sigma + \mathcal{M}_{\bar{\sigma}},
\]
where $\mathcal{M}_\sigma = \{f \in D_L^* | P_{\sigma} f = f\}, \mathcal{M}_{\bar{\sigma}} = \{f \in D_L^* | P_{\sigma} f = 0\}$, and $P_{\sigma}$ is the orthogonal projection onto the proper subspace corresponding to the part of the spectrum  $\sigma$.
Meanwhile, the study of  \( T_L^* \) could be decomposed into 
\[
T_L^* = (T_L^*)_1 + (T_L^*)_2,  
\]
where
\[
(T_L^*)_1 = T_L^*|_{\mathcal{M}_\sigma},  (T_L^*)_2 = T_L^*|_{\mathcal{M}_{\bar{\sigma}}}.
\]
Thus, for every $f \in D_L^*$, we have 
\begin{align*}
	f &= f_1 + f_2 = P_{\sigma}f +(I-P_{\sigma}) f, \\
	T_L^*f &= T_L^*f_1 + T_L^*f_2 = (T_L^*)_1f+(T_L^*)_2f.
\end{align*}
Furthermore, one has 
\[
(T_L^*)^2f = T_L^*(T_L^*f) = T_L^*(T_L^*f_1 + T_L^*f_2) = (T_L^*)^2f_1 + (T_L^*)^2f_2 =  [(T_L^*)^2]_1f+[(T_L^*)^2]_2f.
\]
Similarly one has
\[
(T_L^*)^nf =  [(T_L^*)^n]_1f+[(T_L^*)^n]_2f.
\]
Since $\rho((T_L^*)_1) =1, \rho((T_L^*)_2) <1$, we have
\[
(T_L^*)^nf \to (T_L^*)_1f + 0 = P_{\sigma}f,  n \to \infty.
\]
This analysis shows that, under the critical condition $\rho(T_L^*) =1$, the iterations $(T_L^*)^nf$ converge to a projection of $f$ onto the invariant subspace $P_{\sigma}f$ associated with $\sigma = {1}$ when $n \to \infty$. Based on the definition of Lipschitz dual operator, we could further study the convergence behavior of $T^n$. 
The results are presented in Section \ref{limitsa}. Here we utilize the tools of Lipschitz dual operator and spectral decomposition of compact operators to help address setting of nonlinear operators.

\subsection{Examples: Lip Numbers of Basic Functional Blocks of Foundation Models} \label{eqmodel}

In this section, we demonstrate that the Lip numbers can be well evaluated by Proposition \ref{prop1} for various typical basic functional blocks of foundation models.

\begin{example}[\bf{Linear Model}] \label{ex2}
	Consider the linear model $T(X) = WX, W \in \mathbb{R}^{d\times d}, X \in \mathbb{R}^d, T: \mathbb{R}^d \to \mathbb{R}^d$. Then $\mathrm{Lip}(T)$ is given by \[\mathrm{Lip}(T) = \sup\limits_{x\in \mathbb{R}^d}  \rho(T'(x)) = \rho(W) = \sqrt{\lambda_{\max} (W^\top W)}.\] 
	This implies that the Lip numbers of linear model is determined by the largest eigenvalue of weight matrix $W$.
	%	and the Dahlquist constant \[\mu(T) =\inf_{x ,y\in D, x \neq y} \left\{ \frac{\langle Wx-Wy,x-y\rangle}{||x-y||^2}\right\}   =  \lambda_{\min}(W) .\]
\end{example}

\begin{example}[\bf{MLP Model}] \label{ex1}
	Consider the single hidden layer (or 2-layer) MLP model $T(X) =MLP(X) = W_2(\sigma(W_1X)), x \in \mathbb{R}^d, W_1 \in \mathbb{R}^{n_1 \times d}, W_2 \in \mathbb{R}^{n_2 \times n_1},
	T: \mathbb{R}^d \to \mathbb{R}^{n_2}$, $\sigma(\cdot)$ is the activation function with $1-$Lipschitz continuous, e.g., ReLU, Leaky ReLU, SoftPlus, Tanh, Sigmoid, ArcTan or Softsign, etc. Generally, suppose that the activation function is ReLU, which is commonly used in most deep neural networks and foundation models. Then the Lip number can be estimated by
	\begin{align*}
		\mathrm{Lip}(T)& = \sup\limits_{x\in \mathbb{R}^{d}}  \rho(T'(x)) = \sup\limits_{x\in \mathbb{R}^{d}} \rho(W_1^\top Diag(\sigma'(W_1x))W_2^\top) \\
		& \leq \rho(W_2W_1) \leq \rho(W_2)\rho(W_1) = \sqrt{\lambda_{\max}(W_2^\top W_2)\lambda_{\max}(W_1^\top W_1)}.
	\end{align*}
	%and the Dahlquist constant is computed by
	%\begin{align*}
	%	\mu(T) & =\inf_{x ,y\in D, x \neq y} \left\{ \frac{\langle W_2(\sigma(W_1x))-W_2(\sigma(W_1y)),x-y\rangle}{||x-y||^2}\right\} \\
	%& = \inf_{x ,y\in D, x \neq y}  \left\{ \frac{\langle W_2(\sigma(W_1x)-\sigma(W_1y)),x-y\rangle}{||x-y||^2}\right\} \\
	%& \approx \inf_{x ,y\in D, x \neq y}  \left\{ \frac{\langle W_2 Diag(\sigma'(W_1z)) (W_1(x-y)),x-y\rangle}{||x-y||^2}\right\} \\
	%& = \sqrt{\lambda_{\min}(W_1^\top W_2^\top W_2W_1)} .
	%\end{align*}
	Generally, for $L$-layer MLP model $T(x) = W_L(\sigma (W_{L-1}(\cdots W_2(\sigma(W_1x)))))$, the Lip number could be upper bounded by
	\[\mathrm{Lip}(T) \leq \prod_{k=1}^L\rho(W_k). \]
\end{example}

\begin{example}[\bf{Self-Attention Model}] 
	Consider the single head self-attention block \(T(X) = SA(X) = \mathrm{softmax}\left(\frac{X W^Q (X W^K)^{\top}}{\sqrt{d}}\right) X W^V := PX W^V, \quad X\in \mathbb{R}^{N \times d}$, $\quad W^K, W^Q \in \mathbb{R}^{d \times d_K}, W^V \in \mathbb{R}^{d \times d_V}, T: \mathbb{R}^{N \times d} \to \mathbb{R}^{N \times d_V}\). 
		Then the Lip number satisfies
		$$\mathrm{Lip}(T) \leq \rho(W^V) \left(1+  \frac{B^2}{\sqrt{d}} \rho(W^K) \rho(W^Q) \right)$$if $X$ is bounded, and particularly $\|X\|_F \leq B $. And $\mathrm{Lip}(T) = \infty$ if $X$ is unbounded.
%	\begin{align*}
%		\mathrm{Lip}(T)& = \sup\limits_{x\in \mathbb{R}^{N \times d}}  \rho(T'(x)) = \infty,
%	\end{align*}

\end{example}
The proof is presented in Appendix \ref{attention}. This example shows that the Lip number of standard self-attention model is bounded only if $X$ is bounded. Otherwise, $\mathrm{Lip}(T)$ is unbounded if $X$ is unbounded. This explains why attention models always require finite input lengths in practice. This result can be extended to the multi-head cases.

%To improve training stability of self-attention model, Kim et al. \cite{kim2021lipschitz} proposed L2-distance attention to repalace the standard dot-product attention, whose Lip number is bounded conditioning on $W^Q = W^K$. In \cite{qilipsformer}, Qi et al. proposed a scaled cosine similarity attention, and illustrated the novel self-attention block is Lipschitz continuous, by making the norm $W^Q,W^K,W^V$ be bounded. 
%In practice,

%The proof is presented in \ref{attention}. As we can see, the Lip number of standard self-attention model is unbounded, which may cause potential instability problems in training self-attention model. In addition, when $X$ is bounded, and $W^Q,W^K,W^V$ have bounded, the Lip number could be bounded. The conclusion can be easily applied to multi-head self-attention model.

\begin{example}[\bf{Multi-Head Self-Attention Model}] \label{eq5}
	Consider the multi-head self-attention model $T(X) =MSA(X) =$
	\begin{align*}
	 \left[\mathrm{softmax}\left(\frac{X W_1^Q (X W_1^K)^{\top}}{\sqrt{d}}\right) X W_1^V,\cdots, \mathrm{softmax}\left(\frac{X W_K^Q (X W_K^K)^{\top}}{\sqrt{d}}\right) X W_K^V\right]W^O, \end{align*}
	where \( X \in \mathbb{R}^{N \times d}, W_i^K, W_i^Q \in \mathbb{R}^{d \times d_K}, W_i^V \in \mathbb{R}^{d \times d_V}, W^O \in \mathbb{R}^{Kd_V \times d}, 
	T: \mathbb{R}^{N \times d} \to \mathbb{R}^{N \times d}\). Then the Lip number $\mathrm{Lip}(T)$ satisfies
		$$\mathrm{Lip}(T) \leq \rho(W^V) \left(1+  \frac{B^2}{\sqrt{d}} \rho(W^K) \rho(W^Q) \right),$$ when $X$ is bounded. Otherwise $	\mathrm{Lip}(T)=\infty$ when $X$ is unbounded. 
	
\end{example}

\begin{example}[\bf{LayerNorm}] \label{ex6}
	Let $X \in \mathbb{R}^d$, the LayerNorm operator \cite{lei2016layer} is defined by
	\begin{align*}
		T(X)  & = LN(X) =  \mathbf{\gamma} \odot N(X) + \mathbf{\beta},  
	\end{align*} 
	where $\mathbf{\gamma},\mathbf{\beta}\in \mathbb{R}^d$ are learnable affine parameters, $\odot$ is the element-wise product, and $N(X)   = \frac{X - \mu}{\sigma}, \mu  = \frac{1}{d} \sum_{i=1}^{d} X_i, \sigma = \sqrt{\frac{1}{d} \sum_{i=1}^{d} (X_i - \mu)^2 + \epsilon},  \epsilon>0$. Then Lip number of $T$ satisfies
	\begin{align*}
		\mathrm{Lip}(T)\leq  \|\gamma\|_{\infty}\epsilon^{-\frac{1}{2}}.
	\end{align*}
\end{example}
The proof is presented in Appendix \ref{layernorm}.
Based on Examples \ref{ex1}, \ref{eq5} and \ref{ex6}, we can further study the Lip number of the Transformer model.

%\begin{example}[Model with Residual Connection]
%	Consider the model with residual connection $F(X) = X + T(X)$, where $T(X)$ could be chosen from Example \ref{ex2}-\ref{ex6}. We can rewritten $F(X)$ as $F(X) = X - \alpha ((-\frac{1}{\alpha})T(x))$.	
%	Then if $\mathrm{Lip}(T) \leq B $ for some constant $B>0$, and these exist a constant $0 <\alpha \leq \frac{2m}{B^2}$, where $m = \inf_{Tx\neq Ty, x,y \in D(T)} \frac{\left\langle Tx-Ty, x-y \right\rangle}{\|Tx-Ty\|^2}$, then the Lip number $\mathrm{Lip}(F)$ satisfies $\mathrm{Lip}(F) \leq 1$. 
%	
%	%given by \[\mathrm{Lip}(T) = \sup\limits_{x\in  \mathbb{R}^{d}}  \rho(T'(x)) = \rho(I+W) = \sqrt{\lambda_{\max} ((I+W)^\top (I+W))}.\]
%	%	and the Dahlquist constant is computed by \[\mu(T) =\inf_{x ,y\in D, x \neq y} \left\{ \frac{\langle (I+W)(x-y),x-y\rangle}{||x-y||^2}\right\} = \lambda_{\min}(I+W).\]
%\end{example}

\begin{example}[\bf{Transformer with Post-LayerNorm}] \label{ex7}
	In this case, the basic functional block $T$ is defined by $T(X) = T_2 \circ (I+T_3) \circ T_2 \circ (I+ T_1)(X)$, where $T_1=MSA, T_2 = LN, T_3 = MLP$ are multi-head self-attention model, LayerNorm and MLP model defined as in Examples \ref{eq5}, \ref{ex6} and \ref{ex1}, respectively, and $I$ is the identity operator.
	Then \[\mathrm{Lip}(T)= \infty.\]
%	\begin{align*}
%		\mathrm{Lip}(T)& = \infty.
%		%\mathrm{Lip}(T_2)  \cdot \mathrm{Lip}(I+T_3) \cdot \mathrm{Lip}(T_2) \cdot \mathrm{Lip}(I+T_1)  \\
%		%& \leq  \epsilon^{-1} \max_i |\gamma_i|^2 \left(\frac{D^2-2}{D}\right)^2 \sqrt{\lambda_{\max}(I+W_2^\top W_2)\lambda_{\max}(I+W_1^\top W_1)} \cdot \infty = \infty
%	\end{align*}
\end{example}

\begin{example}[\bf{Transformer with Pre-LayerNorm}] \label{ex8}
	In this case, the basic functional block $T$ is defined by $T(X) = T_2 \circ T_1 (X)$ with $T_1 = I + F_1, T_2 = I + F_2$, and $F_1= MLP\circ LN, F_2 = MSA\circ LN$, $I=$ identity operator, where $MLP, MSA, LN$ are MLP model, Multi-Head Self-Attention Model and LayerNorm defined in Example \ref{ex1}, \ref{eq5} and \ref{ex6}, respectively.
%	$(I+T_3 \circ T_2) \circ  (I+ T_1\circ T_2 )(X)$, which is composed by the basic functional blocks
%	$T_1(X)$= Multi-Head Self-Attention Model  defined in Example \ref{eq5},
%	$T_2(X) = \mathrm{LayerNorm}(X)$ defined in Example \ref{ex6},
%	$T_3(X) = \mathrm{MLP}(X)$ defined in Example \ref{ex1}, 
%	and the identity operator $I$.  
If $F_1, F_2$ are dissipative operators, i.e., $m(F_1)<0, m(F_2)<0$, then $T_i (i=1,2)$ satisfies $\mathrm{Lip}(T_i) \leq 1$ with the quasi-asymptotic regularity property. Consequently,

(\romannumeral1) $\mathrm{Lip}(T) \leq 1$;

(\romannumeral2) $T$ is of quasi-asymptotic regularity property.
\end{example}

The proofs of Examples \ref{ex7} and \ref{ex8} are presented in Appendices \ref{postlayernorm} and \ref{prelayernorm}, respectively. It should be observed that the residual connections in Transformers takes a crucial role in ensuring that the Lip number \(\mathrm{Lip}(T)\leq 1\) and $T$ is of quasi-asymptotic regularity property. In fact, we can see that $F_1= MLP\circ LN, F_2 = MSA\circ LN$ are Lipschitz operators based on Examples \ref{ex1}, \ref{eq5} and \ref{ex6}. If $F_1,F_2$ are dissipative operators, residual connection type operators $T_1=I+F_1,T_2=I+F_2$ satisfy that $\mathrm{Lip}(T_1) \leq 1, \mathrm{Lip}(T_2)  \leq 1$ and $T_1,T_2$ are of 
quasi-asymptotic regularity property

%the Lip number of residual connection type operators $T_1=I+F_1,T_2=I+F_2$ satisfy $\mathrm{Lip}(T_1) \leq 1, \mathrm{Lip}(T_2)  \leq 1$ if $F_1,F_2$ are dissipative operators, and 

All these examples show that Lip number of the basic functional blocks used frequently in foundation models can be computationally verifiable. This supports that the Lip number tool introduced in Section \ref{lips} could be practically powerful in performance assessment of emergent intergence for foundation models in practice. To be specific,
we estimate the Lip numbers of GPT-1 and GPT-2 models in subsequent Examples \ref{ex9} and \ref{ex10}.
\begin{example}[\bf{GPT-1 model}] \label{ex9}
	GPT-1 model is constructed by $T(X) = T_2 \circ \prod_{i=1}^{12} T_{1i} \circ f_0$, where $f_0$ is the input embedding, $T_{1i}, i=1,2,\cdots, 12,$ are the Transformer with post-LayerNorm defined in Example \ref{ex7}, and $T_2$ is the prediction layer defined in Example \ref{ex2}. 
	Then \[\mathrm{Lip}(T)= \infty.\]
\end{example}

\begin{example}[\bf{GPT-2 model}] \label{ex10}
	GPT-2 model is constructed by $T(X) = T_2 \circ \prod_{i=1}^{48} T_{1i} \circ f_0$, where $f_0$ is the input embedding, $T_{1i}, i=1,2,\cdots, 48,$ are the Transformer with pre-LayerNorm defined in Example \ref{ex8}, and $T_2$ is the prediction layer defined in Example \ref{ex2}. 
	Then the basic blocks $\mathrm{Lip}(T_{1i}) (i=1,\cdots,48)$ in GPT-2 satisfies:
	
	(\romannumeral1) $\mathrm{Lip}(T_{1i}) \leq 1, i=1,\cdots,48$;
	
	(\romannumeral2) $T_{1i} (i=1,\cdots,48)$ are of quasi-asymptotic regularity property.
	
	%the Lip numbers  of  $\mathrm{Lip}(T_{1i}) \leq 1$ and $T_{1i} (i=1,\cdots,48)$ are of quasi-asymptotic regularity property.
\end{example}

As shown in Figure \ref{fig1}, GPT-1 \cite{radford2018improving} is built on a Transformer model with post-LayerNorm layer, whereas GPT-2 \cite{radford2019language} adopts a pre-LayerNorm layer design. This difference leads to a striking discrepancy in the estimated Lip numbers. Paticularly, the basic functional blocks $T_{1i}$ of GPT-2 satisfy $\mathrm{Lip}(T_{1i})\leq 1, i=1,2,\cdots, 48$, which is crucial for yielding a stable architecture. We will further discuss this in the next section.

\section{The Necessary and Sufficient Conditions for Existence of Limit Architecture} \label{limitsa}

%Given a foundation model $\mathcal{M} = \{f_W, \mathcal{D}, \mathcal{A}\}$, 
Observing that $f_{W} = (\prod_{i=1}^K T_i) f_0, T_i: \mathbb{R}^{n_i} \to \mathbb{R}^{n_{i+1}} $, we could add virtual neurons and constrain all connection weights with the virtual neurons to be zero, so as to ensure $T_i$ to be a self-mapping.  We thus can assume $T_i :=T_i: \mathbb{R}^{n}  \to \mathbb{R}^{n}$, where $n =\max\{n_i, i = 1,2,\cdots\}$. Besides, we assume that $T_i$ satisfies the quasi-asymptotic regularity property.
%Without loss of generality, we assume that each \( T_i \) is a self-mapping from \( D \) to \( D \).
The following Theorems \ref{th1a}-\ref{th1c} present the necessary and sufficient conditions for existence of the limit architecture when the basic functional blocks $\{T_i\}, i=1,2,\cdots$ are chosen from one fixed Lipschitz operator, finite and countable collections of Lipschitz operators, respectively.  The proofs are all given in Appendix \ref{b3}.

\begin{theorem} \label{th1a}
Suppose that the basic functional blocks of foundation models share a common Lipschitz operator $T$, i.e., $T_i = T$ for all $i = 1,2,\ldots$. If $T$ satisfies the self-mapping and the quasi-asymptotic regularity conditions, then the necessary and sufficient condition for existence of the limit architecture $f^*$ in Eq.~(\ref{eqlimit}) is $\mathrm{Lip}(T) \leq 1$.
\end{theorem}

%Besides, the convergent fixed points relies on the initial architecture $f_0$.

By examining the Lip numbers in Examples~\ref{ex7} and \ref{ex8}, we can immediately see that GPT-2-like models admit the existence of the limit architecture, while GPT-1-like models do not. In consequence, GPT-2-like models are theoretically more favorable for training deep Transformer models based on Theorem \ref{th1a}. The empirical results in Section~\ref{es} will confirm such instability of GPT-1-like models and demonstrate the superior training efficiency of GPT-2-like models.

% By analyzing the Lip number in Example \ref{ex7} and \ref{ex8}, we could see that GPT-2 model is potential effective in training deep Transformer model from this theorem. The empirical result in Section \ref{es} further verifies that the instability issues in training GPT-1 type models, and further demonstrates the superiority of training efficacy for GPT-2 type models.
%	
%	As shown in Figure \ref{fig1}(a), GPT-1 model \cite{radford2018improving} is constructed by 12 layer transformer model with post-LayerNorm. 
%	The empirical result in Section \ref{es} further verifies that the instability issues in training GPT-1 type models.
%	which may cause potential instability problems in training deep Transformer model. 
%	The proof is presented in \ref{prelayernorm}. As we can see, the combination of pre-LayerNorm and Attention model could guarteen the boundedness of the Lip number. Note that the assembly manner of LayerNorm and Attention is the main difference between GPT-1 \cite{radford2018improving} and GPT-2 model \cite{radford2019language} (as shown in Figure \ref{fig1}(b)).
%	The empirical result in Section \ref{es} further demonstrates the superiority of training efficacy for GPT-2 type models.

\begin{theorem} \label{th1b}
	
	Suppose that the basic functional blocks $T_i, i = 1,2,\ldots$ of a foundation model are selected from a finite collection of operators $\{S_i\}_{i=1}^r$, and \(
	\bigcap_{i=1}^r \mathrm{Fix}(S_i) \neq \varnothing
	\). Assume that each $S_i$ satisfies the self-mapping and the quasi-asymptotic regularity conditions.
	Then the necessary and sufficient condition for existence of the limit architecture in Eq.~(\ref{eqlimit}) is $\mathrm{Lip}(S_i) \leq 1$ for all $i \in [r]$.
	
%	Suppose that basic functional blocks $T_i, i=1,2,\cdots$ of the foundation models are chosen from $S_i, i \in [r]$ satisfying self-mapping condition, where $r$ is finite constant, and 	\(
%	\bigcap_{i=1}^r \mathrm{Fix}(S_i) \neq \emptyset.
%	\)
%	Then the necessary and sufficient conditions for the existence of limit architecture $f^*$ in Eq.(\ref{eqlimit}) is $\mathrm{Lip}(S_i) \leq 1, i \in [r]$.
\end{theorem}

	In practice, foundation models are constructed from existing basic functional blocks $S_i, i \in [r]$, like Attention, MLP, etc. Theorem \ref{th1b} implies that if the Lip number of each basic block satisfies $\mathrm{Lip}(S_i)\leq 1$, then the limit architecture exists.
	
	In what follows, we say $T$ is a \textbf{generalized projection operator}, if it satisfies the following properties: (\romannumeral1) $T|_{\mathrm{Fix}(T)} = I$; (\romannumeral2) $T$ is quasi-asymptotic regular; and (\romannumeral3) $\mathrm{Lip}(T) \leq 1$. %Theorem \ref{th1c} shows that $T_i$ convergences to a generalized projection operator if the limit architecture exists.

\begin{theorem} \label{th1c}
	Suppose that the basic functional blocks $T_i$ ($i = 1,2,\ldots$) of a foundation model satisfy the self-mapping and the quasi-asymptotic regularity conditions. Then the necessary and sufficient condition for the existence of the limit architecture $f^*$ in Eq.~(\ref{eqlimit}) are that:
	
	(\romannumeral1) $\mathrm{Lip}(T_i) \leq 1$ for all $i \geq K_0$, for some constant $K_0$;
	
	(\romannumeral2) there exists a generalized projection operator $T$ such that $\|T_i - T\|  \leq \epsilon_i$, and $\sum_{i=1}^\infty \epsilon_i < \infty$.

%	Suppose that basic functional blocks $T_i, i=1,2,\cdots$ of the foundation models satisfies self-mapping condition, then the necessary and sufficient conditions for the existence of limit architecture $f^*$ in Eq.(\ref{eqlimit}) are (1) $\mathrm{Lip}(T_i)\leq 1, i >K $ for some constant $K$; (2) there exist a fixed operator $T$ with $\mathrm{Lip}(T)\leq 1$, such that $\|T_i-T\|= \epsilon_i$, where $\sum_{i=1}^\infty \epsilon_i < \infty$. 
\end{theorem}

Empirically speaking, the model weights of different basic functional blocks are typically different, meaning that the same network architecture may exhibit different functional properties across layers. Theorem~\ref{th1c} addresses this in a more general setting and shows that, once a limit architecture exists, the sequence of basic functional blocks exhibits a \textbf{condensing property}, that is, $T_i \to T$ as $i \to \infty$, and the convergence speed is sufficiently fast. 
In other words, the basic functional blocks of foundation models tend to a common generalized projection operator $T$ after sufficiently large number of layers.

Theorem \ref{th1c} together Examples \ref{ex8} and \ref{ex10} implies immediately the following conclusion on GPT-2 like foudation models.

\begin{corollary}
If $T_{1i} = T_{1i}^{(2)} \circ T_{1i}^{(1)},i = 1,2,\ldots$ are the basic functional block operators specified by GPT-2 like architecture, as defined in Examples \ref{ex8} and \ref{ex10}. Then the limit architecture $T(X) = T_2 \circ \prod_{i=1}^{\infty} T_{1i} \circ f_0$ exists. 
\end{corollary}
%Note that $T_{1i} = T, i =1,2,\cdots$, where $T$ denotes the Transformer with pre-layerNorm defined in Example \ref{ex8}, the desired conclusion follows directly from Theorem~\ref{th1a} since $\mathrm{Lip}(T) \leq 1$.

\begin{remark} 
%\textbf{Remarks:}
\normalfont
%\\

\begin{enumerate}
	\renewcommand{\labelenumi}{(\roman{enumi})}
	\item From Appendix \ref{appendixb}, we can see that it is the nonlinear Lipschitz operator tools introduced in Section \ref{lips} that play a crucial part in the proof of Theorems \ref{th1a}-\ref{th1c}. The significance of Theorems \ref{th1a}-\ref{th1c} is not only the characterization of existence of the limit architecture of foundation models, but also the generalization of several classifical nonlinear analysis results. In fact, the general conditions for existence and convergence of fixed points for a family of operators 	are $L(T) \leq 1$, which holds under a specific chosen metric $d(\cdot,\cdot)$. In contrast, the results presented in Theorems \ref{th1a}-\ref{th1c} show that this dependence on metric can be eliminated: a fixed point exists and the self-iteration orbit $\{T^nf_0\}$ converges whenever $\mathrm{Lip}(T) \leq 1$, regardless of the particular choice of metric.

	\item If $\mathrm{Lip}(T) < 1$, the convergence of $T^nf_0$ is immediate, whereas if $\mathrm{Lip}(T) > 1$, the iterates $T^nf_0$ may diverge. $\mathrm{Lip}(T)=1$ thus is a critical case at which a very complex dynamics might occur. Theorems \ref{th1a}-\ref{th1c} show however that the critical case $\mathrm{Lip}(T) = 1$ does play a fundamental role in enabling emergent abilities of foundation models. Indeed, when $\mathrm{Lip}(T) <1$, the operator admits a unique fixed point, which leads to a unique architecture for any initial $f_0$, thus seriously restricting the richness of the induced limit architecture. While when $\mathrm{Lip}(T) =1$, $T$ may have many different fixed points, and the convergent fixed points of $T^nf_0$ then depend on the initial architecture, thus may exhibit diverse properties of different limit architectures. This supports that the emergent intergence of foundation models most likely comes from the critical setting $\mathrm{Lip}(T)=1$. This claim is acturally consistent with the recent empirical findings of Simon Vock and Christian Meisel \cite{vock2025critical} and M{\"u}ller et al. \cite{cai2025learning}, and thus, Theorems \ref{th1a}-\ref{th1c} provide a first theoretical evidence of these findings.

	\item Recently, Xu and Zhang \cite{xu2024uniform,xu2022convergence} studied the convergence of deep neural networks as the depth of the
	networks tends to infinity. They employed a convergence condition of infinite matrix products. The key result is presented in the following:
	
	Let $\|\cdot\|$ be the matrix norm with property $\|AB\| \leq \|A\|\|B\|$ for all matrices $A,B$. Suppose that the weight matrices $\mathbf{W}_1 \in \mathbb{R}^{n\times d}$, $\mathbf{W}_l \in \mathbb{R}^{n\times n}$, $k \geq 2$ satisfy 
	\begin{equation}
		\mathbf{W}_l = \mathbf{I} + \mathbf{P}_l, l \geq 2, \text{ and }  \quad \sum_{l=2}^{\infty} \|\mathbf{P}_l\| < +\infty, 
	\end{equation}
	then the limit $\lim\limits_{L\to \infty} \prod_{l=1}^{L} W_l$ exists.  
	
	It can be seen that Theorem \ref{th1c} generalizes the above result from the linear case to the nonlinear case. Indeed, let the generalized projection operator $T$ be $\mathbf{I}$, $T_i$ be weight matrices $\mathbf{W}_i$($i \geq 2$), and $\epsilon_i=\|\mathbf{W}_i-\mathbf{I}\|$ in Theorem \ref{th1c}, then it is clear that $\sum_{l=2}^{\infty} \|\mathbf{W}_i-\mathbf{I}\|=\sum_{l=2}^{\infty} \|\mathbf{P}_l\| < \infty$, namely, all the conditions of Theorem \ref{th1c} are satisfied, thus the limit $\lim\limits_{L\to \infty} \prod_{l=1}^{L} W_l$ exists. %Therefore, our theoretical results extend the scope of applicability and provide a more precise characterization of existence of the limit architectures for foundation models.

\end{enumerate}

\end{remark}

\section{Existence of Emergent Intelligence}  \label{sec4}
Recall that a function $\ell$ is $L_0$-Lipschitz if \(
\|\ell(x) -  \ell(y)\| \le L_0 \|x - y\|, 
\)
and a function $\ell$ is $G$-smooth if $\nabla \ell$ is $G$-Lipschitz, i.e.,\(
\|\nabla \ell(x) - \nabla \ell(y)\| \le G \|x - y\|.
\)
It is $\mu$-strongly convex if
\(
\ell(y) \ge \ell(x) + \langle \nabla \ell(x), y - x \rangle + \frac{\mu}{2}\|y - x\|^2.
\) When $\mu=0$, it is a convex function. We write $A \lesssim B$ if there exists a constant $C > 0$ such that \(A \le C B\).

We demostrate in this section that the limit of $\mathcal{E}(N,P,K)$ could help identify the existence of emergent intelligence in a foundation model. Wnenever it exists, the limiting behavior $\mathcal{E}(\infty,\infty,\infty)$ highlights the emergent intelligence. With this understanding, we present and verify the following result.

\begin{theorem} \label{the2}
	
	Suppose the following conditions hold:
	
	(\romannumeral1) The loss function $\ell$ is convex and $G$-smooth, and the learning rate $\gamma_k \leq 1/L$.
	
	(\romannumeral2) The basic functional blocks $\{T_i, i \in [L]\}$ satisfy the self-mapping condition and $\mathrm{Lip}(T_i) \leq 1$ for all $i > K_0$ for some integer $K_0$, and there exists a generalized projection operator $T$ such that $\|T_i - T\| \leq \epsilon_i$, where $\sum_{i=1}^\infty \epsilon_i < \infty$.

	(\romannumeral3) The training data $\{(x_i,y_i), \in [N]\}$ are independently and identically distributed (i.i.d.) samples from the distribution \( \mathbb{P}(x,y) \).
	
	Then the foundation models $\mathcal{M} = \{f_W, \mathcal{D}, \mathcal{A}\}$ display emergent intelligence, i.e.,
	\(\lim\limits_{N,P,K \to \infty} \mathcal{E}(N,P,K)\) exists.
\end{theorem}

The proof of this theorem is presented in Appendix \ref{existence}. Note that the conditions 	(\romannumeral1), (\romannumeral2),	(\romannumeral3) in Theorem \ref{the2} respectively ensure the existence of $\mathcal{E}(N,P,\infty)$, $\mathcal{E}(N,\infty,\infty)$ and  $\mathcal{E}(\infty,\infty,\infty)$. It says that when a foundation model is well-trained (condition (\romannumeral1)), the corresponding limit architecture exists (condition (\romannumeral2)), and the training data are sufficiently large and of high quality (condition (\romannumeral3)), the corresponding foundation model then inclines to display emergent intelligence.  

Theorem \ref{the2} reveals that the emergent intelligence of foundation models occurs largely depending on the properties of the basic functional blocks. Of course, it also depends precisely on the training data and training steps. Thus, it underlies the design of foundation models. Specifically, a feasible architecture (may display emergent intelligence) can be obtained by stacking multiple basic functional blocks, provided that they satisfy condition (\romannumeral2) in Theorem \ref{the2}. 
Corollary \ref{coro} below shows that the foundation models with GPT-2-like architectures do display emergent intelligence.

\begin{corollary}\label{coro}
		Suppose the following conditions hold:
	
	(\romannumeral1) The loss function $\ell$ is convex and $G$-smooth, and the learning rate $\gamma_k \leq 1/L$.
	
	(\romannumeral2) The model $f_W$ is contructed by a GPT-2-like architecture, that is, $f_W = T_2 \circ \prod_{i=1}^{L} T_{1i} \circ f_0 $, as defined in Example \ref{ex10}. The basic functional blocks $T_{1i}$ ($i = 1,2,\cdots,L$) satisfy the self-mapping and quasi-asymptotic regularity conditions.
	
	(\romannumeral3) The training data $\{(x_i,y_i), \in [N]\}$ are independently and identically distributed (i.i.d.) samples from the distribution \( \mathbb{P}(x,y) \).
	
	Then the GPT-2-like foundation models $\mathcal{M} = \{f_W, \mathcal{D}, \mathcal{A}\}$ display emergent intelligence, i.e.,
	\(\lim\limits_{N,P,K \to \infty} \mathcal{E}(N,P,K)\) exists.
\end{corollary}

\begin{remark} 
	%\textbf{Remarks:}
	\normalfont
Observe that \( \prod_{i={1}}^{\infty}T_i =  \prod_{i={1}}^{N}T_i + \left(\prod_{i={1}}^{\infty}T_i - \prod_{i={1}}^{N}T_i\right)\). We then can conclude that though the limit architecture is an infinite-dimensional learning system, it could be effectively approximated by a finite product $\prod_{i=1}^{N} T_i $, that is, by a finite-dimensional learning system. The approximation error is bounded by the tail error $\prod_{i={1}}^{\infty}T_i - \prod_{i={1}}^{N}T_i$, which defines a rapidly decaying infinite-dimensional system. Consequently, the limit architecture is in fact a special infinite-dimensional system that consists of a finite system plus a fast decaying infinite system. This provides a theoretical understanding of why foundation models may work well.
Very recently, Kaushik et al., \cite{kaushik2025universal}  empirically observed that the deep neural networks, when applicable, may always converge to shared low-dimensional parametric subspaces regardless of initialization, task, and domain. Considering the general setting independent of initialization, task, or domain in Theorem \ref{the2}, Theorem \ref{the2} thus provides a precise theoretical justification of the empirical findings of Kaushik et al. \cite{kaushik2025universal}.
The experimental results in Section \ref{condensing} will further support our theorical justification.

For the linear regression or random features regression, it is possible to determine the precise form of $\mathcal{E}(\infty,\infty,\infty)$ further. For example, previous studies have computed the precise asymptotics of the test error, e.g., \cite{hastie2022surprises,mei2022generalization,ba2022high,liang2022precise}. We leave the precise asymptotics estimation for the general model defined in Eq.(\ref{eqlimit}) for future study.
	
	\end{remark}

\section{Scaling Laws of Foundation Models} \label{sec5}

When the limit of $\mathcal{E}(N,P,K)$ exists, we can estimate the discrepancy between $\mathcal{E}(N,P,K)$ and $\mathcal{E}(\infty,\infty,\infty)$ using the standard error decomposition in Eq.(\ref{eqdecom}). 

\subsection{Scaling Law of Training Steps}
\begin{theorem}  \label{th7}
		Assume that $\ell$ is $G$-smooth and $\mu$-strongly convex. Choose $\gamma_k \leq 1/G$, then iterates $\{W_k\}_{k \ge 0}$ of GD algorithm defined by Eq.(\ref{eqsss}) on $\ell$ satisfy
		\begin{equation}
	|	\mathcal{E}(N, P, K) - \mathcal{E}(N, P, \infty)| \lesssim  \beta^K %=   e^{-T\ln\left(\frac{1}{\beta}\right)}, 
		\end{equation}
where $\beta=1-1/\kappa$, and $\kappa = G/\mu$ is the condition number of $\ell$.
\end{theorem}

Theorem \ref{th7} is essentially proven in \cite{bach2024learning,karimi2016linear}, showing that the scaling law of training steps follows an exponential law in general. When $\ell$ is a general convex function, the scaling law would follow an power law.
Existing theories \cite{sra2011optimization,jain2017non} often deduce that the derived coefficient $\beta = 1 - \mu/G$ is independent of $N,P$, i.e., data size and model size. Nevertheless‌, such kind of correlation must be taken into account for the analysis of foundation models to make it better comply with training practice of foundation models. Recently, Du et al. \cite{du2019gradient} derived $\beta_{N,P} = 1- \mathcal{O}\left(\frac{1}{N^2}\cdot \frac{1}{2^{\mathcal{O}(P)}}\right)$ for the deep fully-connected neural networks. Notice that here $\beta_{N,P} \to 1$, when $N, P \to \infty$. Such a bound implies that when \( N \) and \( P \) are sufficiently large, the convergence speed of model weight learning becomes extremely slow.  In other word, 
the estimation of $\beta_{N,P}$ highlights the extreme complexity of training large foundation models, and this is also closely related to the quality of training data and the choice of loss function. More refined estimations are clearly needed for this issue, and we leave this in our future study.

\subsection{Scaling Law of Model Size} \label{modelscaling}

\begin{theorem}\label{th8}
%	Assume the functional blocks $\{T_i\}_{i=1}^K$ of foundation models belong to the class of nonlinear Lipschitz operators, and $\ell$ is $L$-Lipschitz in its first argument. 
%	If the foundation model is constructed by \( f_{k+1} = T_i(f_k) \), and the Lip number satisfies \( \text{Lip}(T_i) < 1 \), then the limit architecture satisfies the following estimation:  
%	\[
%	|\mathcal{E}(N, P, \infty) - \mathcal{E}(N, \infty, \infty)| \lesssim  (\mathrm{Lip}(T))^P %=   e^{-P |\ln(\mathrm{Lip}(T))|}.
%	%\leq \frac{(\mathrm{Lip}(T))^K}{1 - \mathrm{Lip}(T)} \|f_0(x) - f^*(x)\| = \mathcal{O}\left(e^{-P |\ln(\mathrm{Lip}(T))|}\right).
%	\]
%	(2) If the network is constructed in the ``A-mode'', i.e., \( f_{k+1} = f_k - \alpha_k A(f_k) \), where \( A = I - T \), \( f_0 = I \), and \( \alpha_k \in (0, 1) \), then when \( m(A) > 0 \), the limit architecture satisfies the following estimation:  
%	\[
%	\mathcal{E}(N, P, 0) - \mathcal{E}(N, \infty, 0) \leq e^{-m(A)} \sum_{k=0}^P \alpha_k \|f_0(x) - f^*(x)\| = \mathcal{O}\left(e^{-m(A) \log P} \vee e^{-m(A) P^{1-s}}\right).
%	\]
Assume that the basic functional blocks $\{T_k\}_{k=1}^K$ of a foundation model belong to the class of nonlinear Lipschitz operators, and the model is construcred through $f_{k+1}=T_k(f_k)$. Assume also that

(\romannumeral2) $\mathrm{Lip}(T_i) \leq 1$ for all $i > K_0$ for some integer $K_0$, and there exists a generalized projection operator $T$ such that $\|T_i - T\| \leq \epsilon_i$, where $\sum_{i=1}^\infty \epsilon_i < \infty$;

(\romannumeral3) $T_k$ satisfies $\gamma$-strongly quasi-nonexpansive propety, that is, \(\|T_kf - f^* \|^2 \leq \|f-f^*\|^2 -\gamma \|T_kf-f\|^2\), for any $f^*\in \cap_{i=1}^\infty\mathrm{Fix}(T_i), f\in \cap_{i=1}^\infty \mathbb{D}(T_i)$, and for some constant $\gamma>0$, and there exists a constant $C>0$, such that $\|f-f^*\| \leq C \|f-T_kf\|$ for any integer $k$.

Then the limit architecture satisfies the estimation:
\[
\big|\mathcal{E}(N, P, \infty) - \mathcal{E}(N, \infty, \infty)\big|
\;\lesssim\; P^{-\frac{1}{2}}.
\]
\end{theorem}

The proof of Theorem \ref{th8} is given in Appendix \ref{modelsize}. As shown in the theorem, the scaling law of the model size will follow a power law when $\mathrm{Lip}(T_i)\leq 1$. It is noted, however, that if $\mathrm{Lip}(T_i)< 1$, the scaling law of the model size may be fast, following an exponential law with exponent $\prod_{i=1}^{P}\mathrm{Lip}(T_i)$. We notice that $\mathrm{Lip}(T_i)=1$ in general cannot imply an appropriate scaling law even it is known to be convergent. 
Here we have imposed a slightly stronger quasi-asymptotic regularity condition that $\{T_i\}_{i=1}^K$ are $\gamma$-strongly quasi-nonexpansive operators, to deduce the power law. This condition, as noticed in Section \ref{lips}, as well as in Example \ref{ex8}, any averaged version of non-expansive operators in general and any Transformer with pre-LayerNorm can naturally be satisfied. Nevertheless, it is still open how to determine a rate of convergence for a generic nonlinear Lipschitz operator with $\text{Lip}(T) = 1$.

%Generally, it is known \cite{halperin1962product,cegielski2012iterative} that an $\alpha$-averaged nonexpansive operator $T=(1-\alpha)I+\alpha F, 0<\alpha <1$ is a $\gamma$-strongly quasi-nonexpansive operator with $\gamma = \frac{\alpha}{1-\alpha}$, where $F$ is a non-expansive operator. 

Theorem \ref{th8}, combined with Example \ref{ex8} and \ref{ex10}, we can immediately derive the scaling law of model size for GPT-2-like models.

\begin{corollary}\label{corco}
	Assume the GPT-2-like architecture is contructed by $f_W = T_2 \circ \prod_{i=1}^{L} T_{1i} \circ f_0 $, as defined in Example \ref{ex10}. We further assume that there exists a constant $C>0$, such that $\|f-f^*\| \leq C \|f-T_{1i}f\|$ for any integer $k$. Then the limit architecture satisfies the following estimation:
	\[
	\big|\mathcal{E}(N, P, \infty) - \mathcal{E}(N, \infty, \infty)\big|
	\;\lesssim\; P^{-\frac{1}{2}}.
	\]
\end{corollary}

This shows that GPT-2-like foundation models may follow a power law.

\subsection{Scaling Law of Data Size} \label{datascaling}

Let us first recall the definition on covering numbers of a family of functions. For more illustrations, please refer to \cite{zhang2002covering,bartlett2002rademacher}.
\begin{Definition}[Covering number]\label{def-CoveringNumber}
	For a given class of vector-valued functions $\mathcal{F}$, the covering number 
	\[
	\mathcal{N}_{\infty}(\mathcal{F};\epsilon;\{x_i\}_{i=1}^{N};\|\cdot\|)
	\]
	is the smallest size of a collection (a cover) $\mathcal{C} \subset \mathcal{F}$ such that for all $f\in \mathcal{F}$, there exists a $\hat{f}\in\mathcal{C}$ satisfying
	\[
	\max_{i} \|f(x_i) - \hat{f}(x_i)\| \leq \epsilon.
	\]
	Further, we define
	\[ 
	\mathcal{N}_{\infty}(\mathcal{F};\epsilon;N;\|\cdot\|) = \sup_{\{x_i\}_{i=1}^{N}} \mathcal{N}_{\infty}(\mathcal{F};\epsilon;\{x_i\}_{i=1}^{N};\|\cdot\|).
	\]
	If $\mathcal{F}$ is real-valued (instead of vector-valued), we drop the norm from the notation. And the metric entropy of $\mathcal{F}$ is defined by $\log\mathcal{N}_{\infty}(\mathcal{F};\epsilon;\{x_i\}_{i=1}^{N};\|\cdot\|)$.
\end{Definition}

\begin{theorem}\label{th9}
	Let $\mathcal{D} = \{x_i,y_i\}_{i=1}^N$ be the training dataset as in Theorem \ref{the2}, and let $\ell$ be a $G$-bounded loss function that is $L_0$-Lipschitz.  
	Consider a function class $\mathcal{F}$ such that $|f| \leq A < \infty$ for all $f\in\mathcal{F}$ and 
	$\log\mathcal{N}_{\infty}(\mathcal{F};\epsilon;\{x_i\}_{i=1}^N) \leq C_{\mathcal{F}}/\epsilon^2$ for all $\{x_i\}_{i=1}^N\in \mathcal{X}^N$. Then for any $\delta > 0$, with probability at least $1-\delta$, it holds that
	\begin{align*}\label{estimateSLDS2}
		\begin{split}
			& |\mathcal{E}(N, \infty, \infty) -  \mathcal{E}(\infty, \infty, \infty)|  \\
			\leq& C \sqrt{\frac{C_{\mathcal{F}}}{N}}
			\left( 1 + \log\!\left( A\sqrt{\frac{N}{C_{\mathcal{F}}}} \right) + 2G\sqrt{\frac{\log(1/\delta)}{2N}} \right) \lesssim  N^{-\frac{1}{2}}. %= e^{-1/2 \ln(N)}.
		\end{split}
	\end{align*}
	Here, $C$, $C_{\mathcal{F}}$ are some constants.
	%the terms \(\inf_{f_W} R(f_W)\) and \(\ell(f,\mathcal{D})\) can hardly be removed, as they represent intrinsic errors, which are typically assumed to be small.

\end{theorem}

The proof of Theorem \ref{th9} is presented in Appendix \ref{dara}. Theorem \ref{th9} shows that the scaling law of data size follows a power law, which is aligned with typical results of learning theory \cite{bach2024learning}. In Appendices \ref{mlp} and \ref{trans}, we will provide a detailed estimations on covering numbers of MLP and Transformer models even with infinite depth under some mild conditions.

\subsection{The Overall Scaling Law }
\begin{theorem} \label{ee}
	Under the conditions of Theorems \ref{th7}, \ref{th8}, \ref{th9}, for any $\delta > 0$, with probability at least $1-\delta$, the following estimation holds:
	\begin{align*}
			|\mathcal{E}(N, P, K) - \mathcal{E}(\infty, \infty, \infty) | \lesssim  N^{-\frac{1}{2}} + P^{-\frac{1}{2}} + \beta^K,
	\end{align*} 
	where $\beta=1-1/\kappa$, and $\kappa = G/\mu$ is the condition number of $\ell$.
	\end{theorem}
The above theorem presents a joint estimation on scaling law of foundation model $\mathcal{M} = \{f_W, \mathcal{D}, \mathcal{A}\}$. It shows that the overall scaling law exhibits a power law with exponent \(1/2\) on both the model size and the data size, together with an exponential law that is dependent of the condition number of the loss function. This scaling law is consistent with those experimentally observed by Kaplan et al.~\cite{kaplan2020scaling}, Hoffmann et al.~\cite{hoffmann2022training}, and others.
Generally, such a scaling law provides a theoretical foundation for determining the data size, model size, and the training costs, particularly the ratio between data and model size for training a foundation model in practice, e.g., \cite{hoffmann2022training,kaplan2020scaling}. 
In a very recent study \cite{simon2026there}, Simon et al. asked that ``can we develop a theory of scaling laws that both explains why power laws arise and predicts their exponents a priori''. Theorem \ref{ee} provides, to some extent, a answer to the question. In particular, Theorem \ref{ee} reveals that the scaling law arises from the convergence rates of \(\mathcal{E}(N, P, K)\) to its limit, and the exponents are respectively determined by the condition number of loss function, the characteristic number $\mathrm{Lip}(T)$ of basic functional blocks and the metric entropy of the hypothesis class.

%\section{Experimental Results}

\section{Empirical Validation (\uppercase\expandafter{\romannumeral1}): Roles of $\mathrm{Lip}(T)$ }  \label{es}
In this section, we empirically show that the Lip number $\mathrm{Lip}(T)$ does play a crucial role in assess‌ing the usability of different network architectures. In section \ref{sec:experiments}, we make a comparative analysis between GPT-1 and GPT-2 models trained on Pile Dataset. In section \ref{norm}, we analyze the Lip numbers of GPT-1 and GPT-2 models, demonstrating that the Lip number does indicate the performance and stabilility of foundation models.

%\subsection{Comparative Analysis between GPT-1 and GPT-2 Architectures}
%

%\textbf{Model Architectures:} 
%
%
%
%While both models utilize a Transformer decoder-only framework, the GPT-2 architecture introduces several critical refinements over GPT-1. These include the relocation of Layer Normalization (Pre-norm configuration), modified weight initialization schemes, and an expanded vocabulary. We investigate whether these architectural improvements translate into better sample efficiency and reasoning capabilities under identical data constraints.
\subsection{Comparative Analysis between GPT-1-like and GPT-2-like Architectures on Pile Dataset}\label{sec:experiments}

In this subsection, we present an experimental comparison between the GPT-1~\cite{radford2018improving} and GPT-2~\cite{radford2019language} model architectures (see Figure \ref{fig1}(a)(b)). The main difference of GPT-1-like and GPT-2-like architectures is in the location of Layer Normalization. The theoretical analysis in Section \ref{eqmodel} has demonstrated that the Transformer model with pre-LayerNorm (GPT-2 like architectures) possesses better propety than that with post-LayerNorm (GPT-1 like architectures) in terms of $\mathrm{Lip}(T)$ measure. We will empirically verify this theoretical results to illuminate the role of $\mathrm{Lip}(T)$.
To ensure a fair evaluation of the architectural impact on performance, both models were standardized to a 1-billion (1B) parameter configuration and trained on a fixed budget of 20-billion (20B) tokens following~\cite{liu2024regmix}. 

%\textbf{Training Data:} 
We utilized the 17 publicly available domains from \textit{The Pile} dataset~\cite{gao2020pile} as our evaluation task. The task information are shown in Table~\ref{tab:pile_stats_full}.

\begin{table}[ht]
	\centering
	\caption{An overview of the 17 available domains in the Pile dataset used in our experiments.}
	\label{tab:pile_stats_full}
	\tiny
	\newcolumntype{L}{>{\RaggedRight\arraybackslash}X} 
	\begin{tabularx}{\textwidth}{l r L} % l:左对齐, r:右对齐, L:自动换行左对齐
		\toprule
		\textbf{Component} & \textbf{Effective Size} & \textbf{Brief Description} \\
		\midrule
		Pile-CC            & 227.12 GiB & Web crawl data processed via the Pile's extraction pipeline. \\ \addlinespace[2pt]
		PubMed Central     & 180.55 GiB & Full-text biomedical research articles from the NIH. \\ \addlinespace[2pt]
		ArXiv              & 112.42 GiB & Scientific preprints in Physics, Mathematics, and Computer Science. \\ \addlinespace[2pt]
		GitHub             & 95.16 GiB  & Public source code repositories from various programming languages. \\ \addlinespace[2pt]
		FreeLaw            & 76.73 GiB  & Legal court opinions from US federal and state courts. \\ \addlinespace[2pt]
		Stack Exchange     & 64.39 GiB  & Q\&A data from the diverse Stack Exchange network. \\ \addlinespace[2pt]
		USPTO Backgrounds  & 45.81 GiB  & Technical background sections of US patents. \\ \addlinespace[2pt]
		PubMed Abstracts   & 38.53 GiB  & Summaries of biomedical research papers. \\ \addlinespace[2pt]
		Gutenberg (PG-19)  & 27.19 GiB  & Long-form literary works from the Project Gutenberg library. \\ \addlinespace[2pt]
		Wikipedia (en)     & 19.13 GiB  & High-quality encyclopedic content from English Wikipedia. \\ \addlinespace[2pt]
		DM Mathematics     & 15.49 GiB  & Generated mathematics problems covering various topics. \\ \addlinespace[2pt]
		Ubuntu IRC         & 11.03 GiB  & Chat logs from the Ubuntu technical support channels. \\ \addlinespace[2pt]
		EuroParl           & 9.17 GiB   & Multilingual proceedings of the European Parliament. \\ \addlinespace[2pt]
		HackerNews         & 7.80 GiB   & Discussion threads and comments from the technology forum. \\ \addlinespace[2pt]
		PhilPapers         & 4.76 GiB   & Academic research papers and books in Philosophy. \\ \addlinespace[2pt]
		NIH ExPorter       & 3.79 GiB   & Descriptions of NIH-funded research projects and grants. \\ \addlinespace[2pt]
		Enron Emails       & 1.76 GiB   & Historical email communication corpus from Enron Corporation. \\
		\bottomrule
	\end{tabularx}
\end{table}

\textbf{Evaluation Benchmarks:} We evaluated the models across 13 downstream benchmarks covering three primary cognitive dimensions:
\begin{itemize}
	\item \textbf{Common Sense Reasoning:} HellaSwag~\cite{zellers2019hellaswag}, PIQA~\cite{bisk2020piqa}, WinoGrande~\cite{sakaguchi2021winogrande}, Social IQA~\cite{sap2019social}, and COPA~\cite{sarlin2020superglue}.
	\item \textbf{Scientific Knowledge:} SciQ~\cite{welbl2017crowdsourcing}, ARC-Easy~\cite{clark2018think}, and OpenBookQA~\cite{mihaylov2018can}.
	\item \textbf{Logic and Comprehension:} LogiQA~\cite{liu2020logiqa}, Lambada~\cite{paperno2016lambada}, RACE~\cite{lai2017race}, MultiRC~\cite{khashabi2018multirc}, and QQP~\cite{wang2018glue}.
\end{itemize}

\textbf{Results and Discussion:}
The final performance comparison, averaged across multiple evaluation shots, is summarized in Table~\ref{tab:final_results}. The empirical results demonstrate that the GPT-2-like architecture significantly outperforms GPT-1-like one in nearly all task categories. The transition from GPT-1-like to GPT-2-like architecture provides a 36.3\% relative improvement in overall model performance, 
validating the superiority of pre-LayerNorm over post-LayerNorm design. Note that by Examples \ref{ex7} and \ref{ex8}, in post-LayerNorm case, $\mathrm{Lip}(T)=\infty$ but $\mathrm{Lip}(T)\leq 1$ in the pre-LayerNorm case. This hightlights the crucial role of boundedness of $\mathrm{Lip}(T)\leq 1$.

\begin{table}[ht]
\centering
\caption{Performance Comparison between GPT-1-like and GPT-2-like models (Values in \%)}
\label{tab:final_results}
\small
\renewcommand{\arraystretch}{1.1} % 增加行高，让表格更有呼吸感
\begin{tabular}{l S[table-format=2.2] S[table-format=2.2] S[table-format=+2.2]}
	\toprule
	\textbf{Task} & {\textbf{GPT-1 1B (\%)}} & {\textbf{GPT-2 1B (\%)}} & {\textbf{Improvement ($\Delta$)}} \\
	\midrule
	Arc-Easy         & 27.07 & 49.25 & +22.18 \\
	COPA             & 57.00 & 68.17 & +11.17 \\
	HellaSwag        & 25.84 & 41.14 & +15.30 \\
	Lambada          & 0.00  & 29.81 & +29.81 \\
	LogiQA           & 24.58 & 26.96 & +2.38  \\
	MultiRC          & 42.80 & 54.41 & +11.61 \\
	OpenBookQA       & 27.00 & 29.77 & +2.77  \\
	PIQA             & 51.39 & 67.68 & +16.29 \\
	QQP              & 63.18 & 43.07 & -20.11 \\
	RACE             & 20.96 & 30.21 & +9.25  \\
	SciQ             & 22.10 & 78.58 & +56.48 \\
	Social IQA       & 34.80 & 38.80 & +4.00  \\
	WinoGrande       & 49.91 & 51.91 & +2.00  \\
	\midrule
	\textbf{Overall Average} & \textbf{34.89} & \textbf{47.57} & \textbf{+12.68} \\
	\bottomrule
\end{tabular}
\end{table}

\begin{figure}
	\centering
	\includegraphics[width=0.9\linewidth]{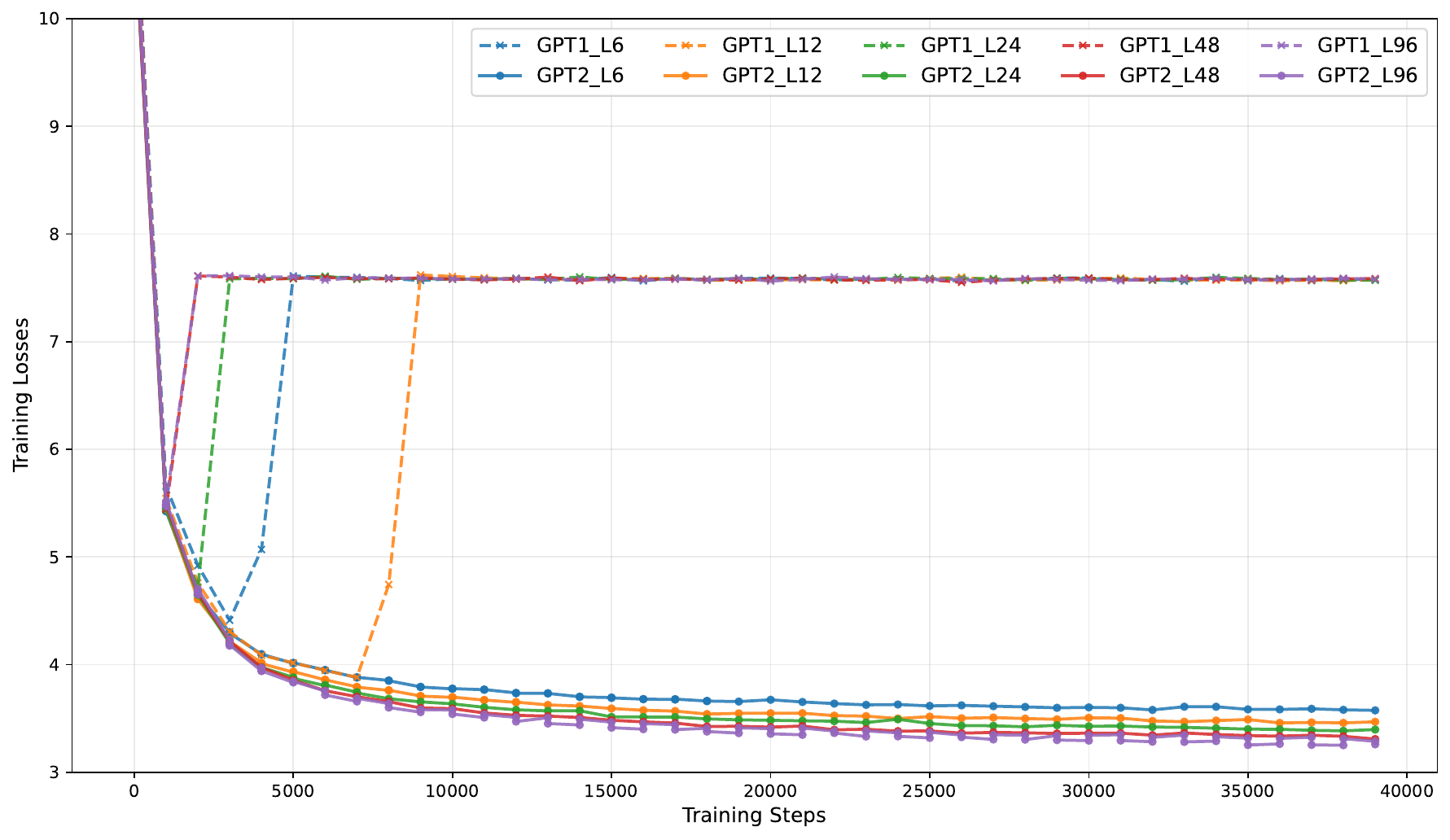}
	\caption{Comparison of the training loss evolution of GPT-1-like and GPT-2-like models on the OpenWebText dataset under the nanoGPT benchmark, with depths of 6, 12, 24, 48, and 96 layers, respectively. GPT-1-like models exhibit invalid NaN/Inf values, which is marked by large values in the figure. }
	\label{fig:openwebtext_loss}
\end{figure}

\begin{figure}
	\centering
	\includegraphics[width=0.9\linewidth]{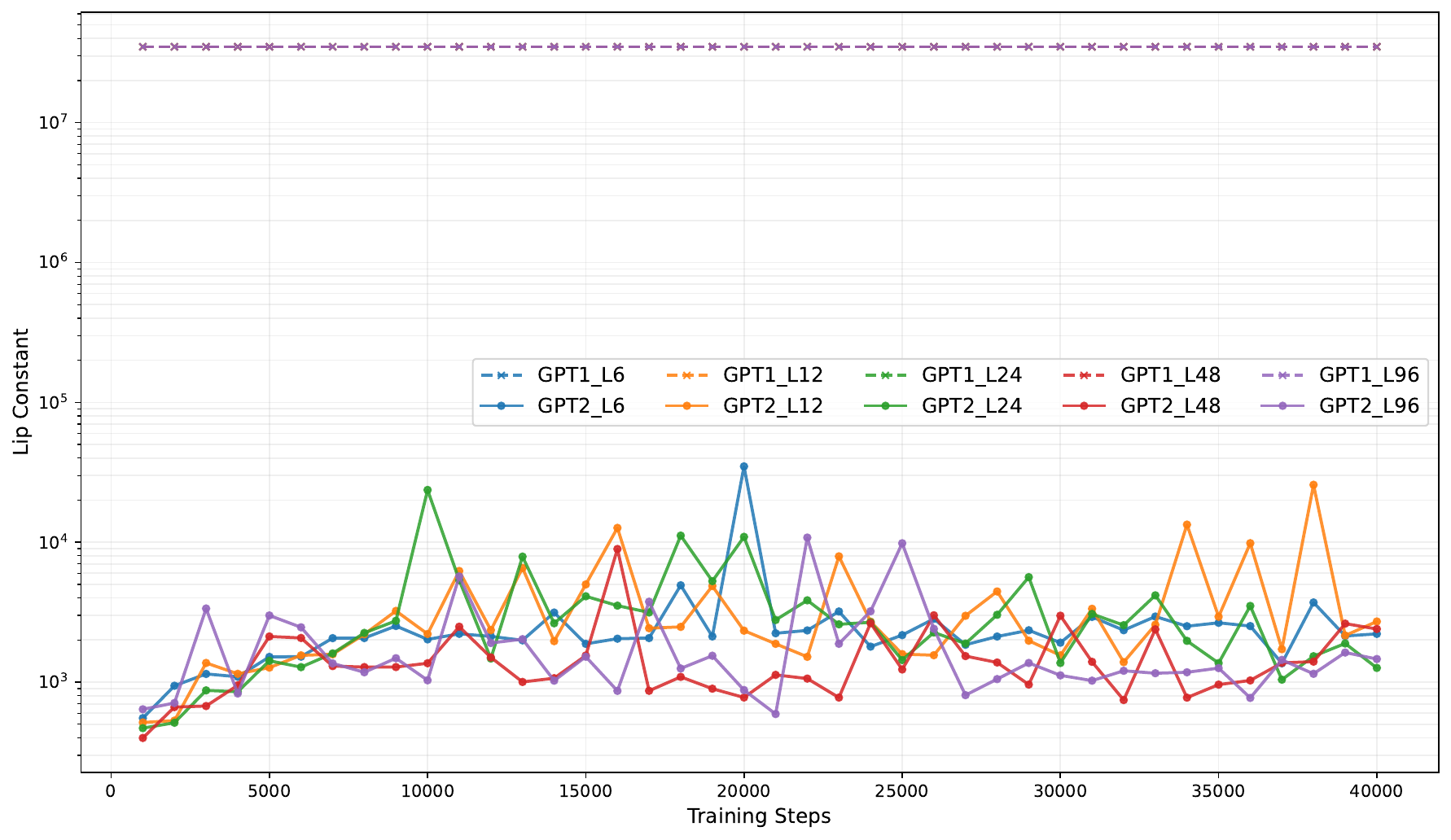}
	\caption{Comparison of the Lip number evolution of GPT-1-like and GPT-2-like models on the OpenWebText dataset under the nanoGPT benchmark, with depths of 6, 12, 24, 48, and 96 layers, respectively. As shown in Examples \ref{ex9} and \ref{ex10}, the Lip number of GPT-2-like model is upper bounded by $\sqrt{\lambda_{max}(W^\top W)}$, in which $W$ is model weights of the last prediction layer, while the Lip number of GPT-1-like models is infinite. The computed Lip numbers corroborates the theoretical results. }
	\label{fig:openwebtext_lower}
\end{figure}

	\subsection{Analysis of Lip Number} \label{norm}

We visualized $\mathrm{Lip}(T)$ under the original GPT-1 and GPT-2 reproduction setting on OpenWebText~\cite{gokaslan2019openwebtext} using the nanoGPT benchmark, with implementation following \cite{karpathy2025nanogpt}. We trained GPT-1-like and GPT-2-like models with depths of 6, 12, 24, 48, and 96 layers.
As shown in Fig.~\ref{fig:openwebtext_loss}, GPT-2-like models consistently achieve lower and more stable training losses across all tested depths as compared to GPT-1-like models. In contrast, after an initial decrease, the training losses of GPT-1-like models abruptly increase and collapse to a high plateau, leading to a degraded performance. This behavior indicates that the GPT-1-like architecture is relatively unstable and difficult to obtain a well-performing model. Moreover, this instability becomes more and more pronounced as the depth increases, suggesting that GPT-1-like architecture is significantly more fragile than GPT-2-like architecture.
The evolution of the Lip numbers for GPT-1-like and GPT-2-like models are illustrated in Fig.~\ref{fig:openwebtext_lower}. It is seen that GPT-2-like models maintain relatively stable Lip numbers with a upper bound across different depths, whereas GPT-1-like models become numerically unstable and even produces invalid NaN/Inf values, consistent with our theoretical analysis in Section~\ref{eqmodel}. Consequently, the training instability of GPT-1-like architecture is not only reflected in the loss curves, but also manifested in the corresponding Lip number dynamics, say, $\mathrm{Lip}(T) = \infty$ computed in Example \ref{ex7}.

These observations suggest that the $\mathrm{Lip}(T)$ does can serve as a very useful indicator of architectural stability and performance. In particular, deeper GPT-2-like models achieve lower training losses, while GPT-1-like models tend to suffer from performance collapse. This further supports that the $\mathrm{Lip}(T)$ of the basic functional blocks provide a principle criterion for assessing architectural stability.
Overall, these empirical findings corroborate our theoretical claims in Section~\ref{limitsa}, showing that the existence of a limit architecture critically depends on the $\mathrm{Lip}(T)$ of the basic functional blocks.

\section{Empirical Validation (\uppercase\expandafter{\romannumeral2}): Condensing Property } \label{condensing}

%\subsection{Models}
To comprehensively investigate the condensing property of foundation models, we selected a diverse suite of state-of-the-art open-source foudation models. Specifically, our experiments encompass dense models of varying parameter scales, including Llama-3.1-8B and Llama-3.1-70B \cite{grattafiori2024llama3}, as well as Qwen-2-7B and Qwen-2-72B \cite{yang2024qwen2}. Furthermore, to verify the generalizability of our findings within sparse architectures, we included the Mixture-of-Experts (MoE) model Deepseek-MoE-16B \cite{dai2024deepseekmoe}. Detailed parameters of these models are demonstrated in Table \ref{tab:models}.

\begin{table}[htbp]
	\centering
	\caption{Basic Parameters and Architectures of the Evaluated Models}
	\label{tab:models}	
		\begin{tabular}{lcccc}
			\toprule
			\textbf{Model} & \textbf{Total Parameters} & \textbf{Layers} & \textbf{Architecture} & \textbf{Context Length} \\
			\midrule
			Llama-3.1-8B        & 8B                  & 32 & Dense & 128K \\
			Llama-3.1-70B       & 70B                 & 80 & Dense & 128K \\
			Qwen-2-7B           & 7B                  & 28 & Dense & 128K \\
			Qwen-2-72B          & 72B                 & 80 & Dense & 128K \\
			Deepseek-MoE-16B    & 16.4B (2.8B Active) & 28 & MoE   & 4K \\
			\bottomrule
		\end{tabular}
\end{table}

\textbf{Datasets:} To ensure that our empirical conclusions are not biased toward a specific domain, we constructed an evaluation benchmark comprising four fundamental Natural Language Processing tasks:
\begin{itemize}
	\item \textbf{Question Answering (QA):} To evaluate reading comprehension and information extraction capabilities, we utilize the SQuAD dataset \cite{rajpurkar2016squad}. This benchmark consists of questions posed by crowdworkers on Wikipedia articles, where the answer to each question is a short text segment (or span) from the corresponding reading passage.
	
	\item \textbf{Mathematical Reasoning (Math):} We select the \textbf{GSM8K} dataset \cite{cobbe2021training} to test the models' multi-step logical and computational proficiency. It comprises high-quality, linguistically diverse grade school math word problems specifically designed to challenge sequential mathematical reasoning.
	
	\item \textbf{Dialogue Summarization (Summarization):} To assess the ability to distill contextual information, we employ \textbf{DialogSum} \cite{chen2021dialogsum}, a large-scale dataset containing real-life scenario conversations alongside corresponding human-annotated summaries.
	
	\item \textbf{Machine Translation (Translation):} We utilize the \textbf{FairTranslate} dataset \cite{jourdan2025fairtranslate} to evaluate cross-lingual alignment and translation accuracy. This English-French benchmark is uniquely designed to simultaneously assess inclusive language generation and the mitigation of non-binary gender biases.
\end{itemize}

\textbf{Experimental procedure:} This experiment aims to examine the information propagation and evolution of intermediate hidden layers across the entire foundation models. Let $K$ denote the total number of layers in a given foundation model. First, we select the $k$-th layer at specific relative depths, where $k$ represents 5\%, 20\%, 40\%, 60\% and 80\% of the total network depth (i.e., $k \in \{0.05K, 0.2K, 0.4K, 0.6K, 0.8K\}$). During a standard forward pass, we extract the hidden state input to the $k$-th layer, denoting it as $a^k$.

Subsequently, we bypass the model's sequential computational graph by treating the extracted intermediate representation $a^k$ as an independent input. Specifically, we feed $a^k$ separately into each individual layer of the model (from layer $1$ to $K$), which produces a set of corresponding single-layer outputs denoted by $\{o_i^k\}_{i=1}^{K}$.
To quantify the functional transformation applied by each layer, we compute both the Euclidean distance and the cosine similarity between the injected input $a^k$ and the corresponding output $o_i^k$ of the $i$-th layer\footnote{It is worth clarifying that although representations from different layers typically reside in distinct semantic spaces, this computation is mathematically valid because the residual connections in Transformers ensure that all hidden states share the same dimensional space. The purpose of this measurement is precisely to detect whether later layers degenerate into approximate identity mappings; a small distance indicates that the layer fails to alter the representation substantively, thereby revealing the "condensing property" of the model.}. Formally, the Euclidean distance is defined as
\[
d_i^k = \| a^k - o_i^k \|_2 , i =1, 2, \cdots, K,
\]
and the cosine similarity is given by
\[
s_i^k = \frac{\langle a^k, o_i^k \rangle}{\|a^k\|_2 \, \|o_i^k\|_2}, i =1, 2, \cdots, K.
\]
We compute the mean and variance of $d_i^k$ and $s_i^k$ across all samples in the dataset. The smaller the distance (i.e., the higher the similarity), the more the layer tends to behave like a projection operator.

\textbf{Experimental results:} Figures \ref{fig:llama31-8b-eu}-\ref{fig:deepseek-moe-16b-cos} show the computed metrics (Euclidean/cosine distance) of different layers for different foundation models. In each figure, rows correspond to different task groups (QA=SQuAD, Math=GSM8K, Summ=DialogSum, and MT=FairTranslate) and columns correspond to selected layer indices $k$.
As evidently shown in these figures, the open-source foundation models exhibit condensing property. Specifically, when the selected intermediate representation $a^k$ is extracted from deeper layers (e.g., $k=0.5K$ or $0.9K$), the distance between $a^k$ and the outputs of the subsequent layers $\{o^k_i\}$ tends to approach zero (indicating extremely high similarity). 
This observation indicates that, beyond a certain depth, the representations across subsequent network layers undergo almost no substantive change. In other words, the first $k$ layers can be viewed as forming an approximate ``fixed point'' for the remaining layers, exhibiting similar functional behavior as the potential limit architecture. The remaining layers can be viewed as approximating the projection operator. This further confirms the condensing property among existing well-performing foundation models, which suggests the condensing property as a core condition for the existence of emergent intelligence, as demonstrated in Section \ref{sec4}.
Notably, this property persists throughout most of the network, with noticeable distributional shifts occurring only in the final few layers, which are likely responsible for task-specific adaptation.

%Based on the plotted distance distributions, we observe a significant representation stagnation phenomenon:  This suggests that deeper hidden states undergo minimal substantive modification when passed through the vast majority of network layers. Effectively, $in^k$ acts as a ``fixed point'' for the computational functions of these layers. This phenomenon is pervasive across the network, with significant distributional shifts only occurring in the final few layers, which are presumably responsible for vocabulary space alignment and logits projection.

\section{Conclusions}
In this work, we have proposed a rigorous mathematical framework for ‌decipher‌ing emergent intelligence and scaling laws of foundation models from a limit perspective. 
We formulate a foudation model as the composition of stacked basic functional blocks $\{T_i\}$, and then introduce the limit architecture $T^* = \prod_{i=1}^\infty T_if_0$, an infinite-dimensional learning system, to embody the infinite-parameters limit of the foundation model. Our theories show that the emergent intelligence originates from the existence of the limit architecture, and the emergent abilities correspond to the learning behavior of this limit architecture. 
By employing the tools from nonlinear Lipschitz operator theory, particularly the characteristic number $\mathrm{Lip}(T)$ of Lipschitz operator, the Lipschitz dual operator theory and the spectral property of compact operators, we have proved that the limit architecture exists if and only if there exists an integer $K_0>0$ such that $\mathrm{Lip}(T_i) \leq 1, \forall i \geq K_0$, and the sequence $\{T_i\}$ possesses the condensing property. We have derived the scaling laws of foundation models, revealing that the performance of foundation models depends on data size, model size and training steps, totally following a scaling law of $N^{-\frac{1}{2}} + P^{-\frac{1}{2}}+\beta^K$. Our theory sheds some lights on the question raised by Simon et al.~\cite{simon2026there}: 
``can we develop a theory of scaling laws that both explains why power laws arise and predicts their exponents a priori''.
Our theoretical analysis shows also that the emergent intelligence is governed by training steps, data size, and model architecture, all of which are however crucially determined by properties of the constituting functional blocks $\{T_i\}$. We show also that the emergent intergence most likely occurs at the critical case when $\mathrm{Lip}(T)=1$. This aligns with the prior findings of Simon Vock and Christian Meisel \cite{vock2025critical} and M{\"u}ller et al. \cite{cai2025learning}, thus providing a theoretical support for these empirical findings. Moreover, we have showed that the emergent intelligence originates from an infinite-dimensional system, but it can be effectively approximated and realized via finite-dimensional architectures. This provides a theoretical explanation on why foudation models may work well with finite architecture. Our theorical analysis further confirms that the limit architecture may evatully evolve to such an infinite system that is constituted of a finite system plus a very fast decaying infinite system, supporting the universal weight subspace hypothesis observated by Kaushik et al. \cite{kaushik2025universal}, and the discretization hypothesis raised by Simon et al. \cite{simon2026there}. Furthermore, we apply the established theories to GPT-2 like architectures, showing that GPT-2-like foundation models can exhibit intelligent emergence, whereas GPT-1-like models do not, and for GPT-2-like models, their scaling law with respect to model size follows a power-law.
Finally, we have provided a series of experiments to support the rationality and correctness of the established theories.
%Empirically, we have demonstrated that the Lip number takes a crucial role in evaluating the usability of foundation models, and could help assess architectural stability.
%We further substantiated the condensing property in state-of-the-art open-source foundation models. 

Our analyses have assumed that the foundation models are constructed by stacking basic functional blocks. An interesting extension is to consider more general composition mechanisms, such as residual connections and mixture-of-experts architectures, and in these settings to derive corresponding conditions for existence of the limit architecture.
Another promising direction is to characterize the precise form of the limit architecture and leverage its properties to provide theoretical guidance for the architectural design of foundation models.
Finally, our limit theory suggests that the foundation models may achieve benign overfitting by enforcing the existence of a limit architecture. This offers a novel perspective on mitigating overfitting, contrasting with the traditional statistical learning paradigm that relies on controlling the complexity of the hypothesis space.

%Task rows map to datasets as QA=SQuAD, Math=GSM8K, Summ=DialogSum, and MT=fair\_translate\_fr. Columns correspond to selected layer indices $k \in \{1, 3, 16, 28\}$. 

%Subsequently, we bypass the model's sequential computational graph. We treat the extracted intermediate representation $a^k$ as an independent input and feed it separately into \textit{every} individual layer of the model (from layer $1$ to $K$). This yields a corresponding set of single-layer outputs, denoted as $\{o^k_i\}_{i=1}^{K}$. We calculate both the Euclidean distance and the cosine similarity between the injected input $a^k$ and output $o^k_i$ of each layers. 
%We compute these metrics across all layers of all samples in the dataset, and visualize the mean and variance of these metrics with respect to samples.

%\subsection{Preliminary Results and Findings}

%\clearpage
% Auto-generated figure gallery. Do not edit manually.
%\section{Comprehensive Visualization}

% Auto-generated figure gallery. Do not edit manually.
%\section{Comprehensive Visualization}
%For each model architecture, we provide two dedicated figures: one for Euclidean distance and one for Cosine similarity. In each figure, rows correspond to task groups (QA, Math, Summ, MT) and columns correspond to the selected layer indices $k$ available in the current result set.

\newcommand{\figurematrixsetup}{%
	\setlength{\tabcolsep}{2pt}%
	\renewcommand{\arraystretch}{1.08}%
	\small%
}

\newcommand{\resultimg}[5]{%
	\begingroup
	\edef\resolvedimagepath{#1/#2#3_#4_input_#5.pdf}%
	\edef\resolvedimage{%
		\noexpand\endgroup
		\noexpand\includegraphics[width=\noexpand\linewidth]{\resolvedimagepath}%
	}%
	\resolvedimage
}

\newcommand{\layerheaderfive}[5]{%
	\textbf{Task} & \textbf{$k=#1$} & \textbf{$k=#2$} & \textbf{$k=#3$} & \textbf{$k=#4$} & \textbf{$k=#5$} \\
}

\newcommand{\taskrowfive}[7]{%
	\textbf{#1} &
	\resultimg{\modeldir}{\modelprefix}{#2}{#3}{\metrictag} &
	\resultimg{\modeldir}{\modelprefix}{#2}{#4}{\metrictag} &
	\resultimg{\modeldir}{\modelprefix}{#2}{#5}{\metrictag} &
	\resultimg{\modeldir}{\modelprefix}{#2}{#6}{\metrictag} &
	\resultimg{\modeldir}{\modelprefix}{#2}{#7}{\metrictag} \\
}

\begin{figure}[p]
	\centering
	\def\modeldir{meta-llama_Llama-3.1-8B}
	\def\modelprefix{datasets_}
	\def\metrictag{eu}
	\figurematrixsetup
	\resizebox{\textwidth}{!}{%
		\begin{tabular}{>{\centering\arraybackslash}m{0.09\textwidth}*{5}{>{\centering\arraybackslash}m{0.172\textwidth}}}
			\toprule
			\layerheaderfive{1}{6}{12}{19}{25}
			\midrule
			\taskrowfive{QA}{rajpurkar_squad}{1}{6}{12}{19}{25}
			\addlinespace[2pt]
			\taskrowfive{Math}{openai_gsm8k}{1}{6}{12}{19}{25}
			\addlinespace[2pt]
			\taskrowfive{Summ}{knkarthick_dialogsum}{1}{6}{12}{19}{25}
			\addlinespace[2pt]
			\taskrowfive{MT}{Fannyjrd_fair_translate_fr}{1}{6}{12}{19}{25}
			\bottomrule
		\end{tabular}%
	}
	\caption{Visualization of Euclidean distance across each layers for Llama-3.1-8.}
	\label{fig:llama31-8b-eu}
\end{figure}

\begin{figure}[p]
	\centering
	\def\modeldir{meta-llama_Llama-3.1-8B}
	\def\modelprefix{datasets_}
	\def\metrictag{cos}
	\figurematrixsetup
	\resizebox{\textwidth}{!}{%
		\begin{tabular}{>{\centering\arraybackslash}m{0.09\textwidth}*{5}{>{\centering\arraybackslash}m{0.172\textwidth}}}
			\toprule
			\layerheaderfive{1}{6}{12}{19}{25}
			\midrule
			\taskrowfive{QA}{rajpurkar_squad}{1}{6}{12}{19}{25}
			\addlinespace[2pt]
			\taskrowfive{Math}{openai_gsm8k}{1}{6}{12}{19}{25}
			\addlinespace[2pt]
			\taskrowfive{Summ}{knkarthick_dialogsum}{1}{6}{12}{19}{25}
			\addlinespace[2pt]
			\taskrowfive{MT}{Fannyjrd_fair_translate_fr}{1}{6}{12}{19}{25}
			\bottomrule
		\end{tabular}%
	}
	\caption{Visualization of cosine similarity across each layers for Llama-3.1-8B.}
	\label{fig:llama31-8b-cos}
\end{figure}
\clearpage

\begin{figure}[p]
	\centering
	\def\modeldir{meta-llama_Llama-3.1-70B}
	\def\modelprefix{datasets_}
	\def\metrictag{eu}
	\figurematrixsetup
	\resizebox{\textwidth}{!}{%
		\begin{tabular}{>{\centering\arraybackslash}m{0.09\textwidth}*{5}{>{\centering\arraybackslash}m{0.172\textwidth}}}
			\toprule
			\layerheaderfive{4}{16}{32}{48}{64}
			\midrule
			\taskrowfive{QA}{rajpurkar_squad}{4}{16}{32}{48}{64}
			\addlinespace[2pt]
			\taskrowfive{Math}{openai_gsm8k}{4}{16}{32}{48}{64}
			\addlinespace[2pt]
			\taskrowfive{Summ}{knkarthick_dialogsum}{4}{16}{32}{48}{64}
			\addlinespace[2pt]
			\taskrowfive{MT}{Fannyjrd_fair_translate_fr}{4}{16}{32}{48}{64}
			\bottomrule
		\end{tabular}%
	}
	\caption{Visualization of Euclidean distance across each layers for Llama-3.1-70B.}
	%	Visualization matrix for Llama-3.1-70B (Dense) using Euclidean Distance. Task rows map to datasets as QA=SQuAD, Math=GSM8K, Summ=DialogSum, and MT=fair\_translate\_fr. Columns correspond to the selected layer indices $k \in \{4, 16, 32, 48, 64\}$. This figure emphasizes magnitude changes in representation space.}
	\label{fig:llama31-70b-eu}
\end{figure}

\begin{figure}[p]
	\centering
	\def\modeldir{meta-llama_Llama-3.1-70B}
	\def\modelprefix{datasets_}
	\def\metrictag{cos}
	\figurematrixsetup
	\resizebox{\textwidth}{!}{%
		\begin{tabular}{>{\centering\arraybackslash}m{0.09\textwidth}*{5}{>{\centering\arraybackslash}m{0.172\textwidth}}}
			\toprule
			\layerheaderfive{4}{16}{32}{48}{64}
			\midrule
			\taskrowfive{QA}{rajpurkar_squad}{4}{16}{32}{48}{64}
			\addlinespace[2pt]
			\taskrowfive{Math}{openai_gsm8k}{4}{16}{32}{48}{64}
			\addlinespace[2pt]
			\taskrowfive{Summ}{knkarthick_dialogsum}{4}{16}{32}{48}{64}
			\addlinespace[2pt]
			\taskrowfive{MT}{Fannyjrd_fair_translate_fr}{4}{16}{32}{48}{64}
			\bottomrule
		\end{tabular}%
	}
	\caption{Visualization of Cosine similarity across each layers for Llama-3.1-70B.}
		%Visualization matrix for Llama-3.1-70B (Dense) using Cosine Similarity. Task rows map to datasets as QA=SQuAD, Math=GSM8K, Summ=DialogSum, and MT=fair\_translate\_fr. Columns correspond to the selected layer indices $k \in \{4, 16, 32, 48, 64\}$. This figure emphasizes directional similarity in representation space.}
	\label{fig:llama31-70b-cos}
\end{figure}
\clearpage

\begin{figure}[p]
	\centering
	\def\modeldir{Qwen_Qwen2-7B}
	\def\modelprefix{datasets_}
	\def\metrictag{eu}
	\figurematrixsetup
	\resizebox{\textwidth}{!}{%
		\begin{tabular}{>{\centering\arraybackslash}m{0.09\textwidth}*{5}{>{\centering\arraybackslash}m{0.172\textwidth}}}
			\toprule
			\layerheaderfive{1}{5}{11}{16}{22}
			\midrule
			\taskrowfive{QA}{rajpurkar_squad}{1}{5}{11}{16}{22}
			\addlinespace[2pt]
			\taskrowfive{Math}{openai_gsm8k}{1}{5}{11}{16}{22}
			\addlinespace[2pt]
			\taskrowfive{Summ}{knkarthick_dialogsum}{1}{5}{11}{16}{22}
			\addlinespace[2pt]
			\taskrowfive{MT}{Fannyjrd_fair_translate_fr}{1}{5}{11}{16}{22}
			\bottomrule
		\end{tabular}%
	}
	\caption{Visualization of Euclidean distance across each layers for Qwen-2-7B.}
		%Visualization matrix for Qwen-2-7B (Dense) using Euclidean Distance. Task rows map to datasets as QA=SQuAD, Math=GSM8K, Summ=DialogSum, and MT=fair\_translate\_fr. Columns correspond to the selected layer indices $k \in \{1, 5, 11, 16, 22\}$. This figure emphasizes magnitude changes in representation space.}
	\label{fig:qwen2-7b-eu}
\end{figure}

\begin{figure}[p]
	\centering
	\def\modeldir{Qwen_Qwen2-7B}
	\def\modelprefix{datasets_}
	\def\metrictag{cos}
	\figurematrixsetup
	\resizebox{\textwidth}{!}{%
		\begin{tabular}{>{\centering\arraybackslash}m{0.09\textwidth}*{5}{>{\centering\arraybackslash}m{0.172\textwidth}}}
			\toprule
			\layerheaderfive{1}{5}{11}{16}{22}
			\midrule
			\taskrowfive{QA}{rajpurkar_squad}{1}{5}{11}{16}{22}
			\addlinespace[2pt]
			\taskrowfive{Math}{openai_gsm8k}{1}{5}{11}{16}{22}
			\addlinespace[2pt]
			\taskrowfive{Summ}{knkarthick_dialogsum}{1}{5}{11}{16}{22}
			\addlinespace[2pt]
			\taskrowfive{MT}{Fannyjrd_fair_translate_fr}{1}{5}{11}{16}{22}
			\bottomrule
		\end{tabular}%
	}
	\caption{Visualization of cosine similarity across each layers for Qwen-2-7B.}
		%Visualization matrix for Qwen-2-7B (Dense) using Cosine Similarity. Task rows map to datasets as QA=SQuAD, Math=GSM8K, Summ=DialogSum, and MT=fair\_translate\_fr. Columns correspond to the selected layer indices $k \in \{1, 5, 11, 16, 22\}$. This figure emphasizes directional similarity in representation space.}
	\label{fig:qwen2-7b-cos}
\end{figure}
\clearpage

\begin{figure}[p]
	\centering
	\def\modeldir{Qwen_Qwen2-72B}
	\def\modelprefix{datasets_}
	\def\metrictag{eu}
	\figurematrixsetup
	\resizebox{\textwidth}{!}{%
		\begin{tabular}{>{\centering\arraybackslash}m{0.09\textwidth}*{5}{>{\centering\arraybackslash}m{0.172\textwidth}}}
			\toprule
			\layerheaderfive{4}{16}{32}{48}{64}
			\midrule
			\taskrowfive{QA}{rajpurkar_squad}{4}{16}{32}{48}{64}
			\addlinespace[2pt]
			\taskrowfive{Math}{openai_gsm8k}{4}{16}{32}{48}{64}
			\addlinespace[2pt]
			\taskrowfive{Summ}{knkarthick_dialogsum}{4}{16}{32}{48}{64}
			\addlinespace[2pt]
			\taskrowfive{MT}{Fannyjrd_fair_translate_fr}{4}{16}{32}{48}{64}
			\bottomrule
		\end{tabular}%
	}
	\caption{Visualization of Euclidean distance across each layers for Qwen-2-72B.}
		%isualization matrix for Qwen-2-72B (Dense) using Euclidean Distance. Task rows map to datasets as QA=SQuAD, Math=GSM8K, Summ=DialogSum, and MT=fair\_translate\_fr. Columns correspond to the selected layer indices $k \in \{4, 16, 32, 48, 64\}$. This figure emphasizes magnitude changes in representation space.}
	\label{fig:qwen2-72b-eu}
\end{figure}

\begin{figure}[p]
	\centering
	\def\modeldir{Qwen_Qwen2-72B}
	\def\modelprefix{datasets_}
	\def\metrictag{cos}
	\figurematrixsetup
	\resizebox{\textwidth}{!}{%
		\begin{tabular}{>{\centering\arraybackslash}m{0.09\textwidth}*{5}{>{\centering\arraybackslash}m{0.172\textwidth}}}
			\toprule
			\layerheaderfive{4}{16}{32}{48}{64}
			\midrule
			\taskrowfive{QA}{rajpurkar_squad}{4}{16}{32}{48}{64}
			\addlinespace[2pt]
			\taskrowfive{Math}{openai_gsm8k}{4}{16}{32}{48}{64}
			\addlinespace[2pt]
			\taskrowfive{Summ}{knkarthick_dialogsum}{4}{16}{32}{48}{64}
			\addlinespace[2pt]
			\taskrowfive{MT}{Fannyjrd_fair_translate_fr}{4}{16}{32}{48}{64}
			\bottomrule
		\end{tabular}%
	}
	\caption{Visualization of cosine similarity across each layers for Qwen-2-72B.}
	%{Visualization matrix for Qwen-2-72B (Dense) using Cosine Similarity. Task rows map to datasets as QA=SQuAD, Math=GSM8K, Summ=DialogSum, and MT=fair\_translate\_fr. Columns correspond to the selected layer indices $k \in \{4, 16, 32, 48, 64\}$. This figure emphasizes directional similarity in representation space.}
	\label{fig:qwen2-72b-cos}
\end{figure}
\clearpage

\begin{figure}[p]
	\centering
	\def\modeldir{deepseek-ai_deepseek-moe-16b-base}
	\def\modelprefix{datasets_}
	\def\metrictag{eu}
	\figurematrixsetup
	\resizebox{\textwidth}{!}{%
		\begin{tabular}{>{\centering\arraybackslash}m{0.09\textwidth}*{5}{>{\centering\arraybackslash}m{0.172\textwidth}}}
			\toprule
			\layerheaderfive{1}{5}{11}{16}{22}
			\midrule
			\taskrowfive{QA}{rajpurkar_squad}{1}{5}{11}{16}{22}
			\addlinespace[2pt]
			\taskrowfive{Math}{openai_gsm8k}{1}{5}{11}{16}{22}
			\addlinespace[2pt]
			\taskrowfive{Summ}{knkarthick_dialogsum}{1}{5}{11}{16}{22}
			\addlinespace[2pt]
			\taskrowfive{MT}{Fannyjrd_fair_translate_fr}{1}{5}{11}{16}{22}
			\bottomrule
		\end{tabular}%
	}
	\caption{Visualization of Euclidean distance across each layers for Deepseek-MoE-16B.}
	%	Visualization matrix for Deepseek-MoE-16B (MoE) using Euclidean Distance. Task rows map to datasets as QA=SQuAD, Math=GSM8K, Summ=DialogSum, and MT=fair\_translate\_fr. Columns correspond to the selected layer indices $k \in \{1, 5, 11, 16, 22\}$. This figure emphasizes magnitude changes in representation space.}
	\label{fig:deepseek-moe-16b-eu}
\end{figure}

\begin{figure}[p]
	\centering
	\def\modeldir{deepseek-ai_deepseek-moe-16b-base}
	\def\modelprefix{datasets_}
	\def\metrictag{cos}
	\figurematrixsetup
	\resizebox{\textwidth}{!}{%
		\begin{tabular}{>{\centering\arraybackslash}m{0.09\textwidth}*{5}{>{\centering\arraybackslash}m{0.172\textwidth}}}
			\toprule
			\layerheaderfive{1}{5}{11}{16}{22}
			\midrule
			\taskrowfive{QA}{rajpurkar_squad}{1}{5}{11}{16}{22}
			\addlinespace[2pt]
			\taskrowfive{Math}{openai_gsm8k}{1}{5}{11}{16}{22}
			\addlinespace[2pt]
			\taskrowfive{Summ}{knkarthick_dialogsum}{1}{5}{11}{16}{22}
			\addlinespace[2pt]
			\taskrowfive{MT}{Fannyjrd_fair_translate_fr}{1}{5}{11}{16}{22}
			\bottomrule
		\end{tabular}%
	}
	\caption{Visualization of cosine similarity across each layers for Deepseek-MoE-16B.}
	%{Visualization matrix for Deepseek-MoE-16B (MoE) using Cosine Similarity. Task rows map to datasets as QA=SQuAD, Math=GSM8K, Summ=DialogSum, and MT=fair\_translate\_fr. Columns correspond to the selected layer indices $k \in \{1, 5, 11, 16, 22\}$. This figure emphasizes directional similarity in representation space.}
	\label{fig:deepseek-moe-16b-cos}
\end{figure}
\clearpage

\bibliography{snbibliography}

\backmatter

\newpage

%\section*{Declarations}
%
%Some journals require declarations to be submitted in a standardised format. Please check the Instructions for Authors of the journal to which you are submitting to see if you need to complete this section. If yes, your manuscript must contain the following sections under the heading `Declarations':
%
%\begin{itemize}
%\item Conflict of interest/Competing interests (check journal-specific guidelines for which heading to use)
%\item Ethics approval and consent to participate
%\item Consent for publication
%\item Data availability 
%\end{itemize}

%\noindent
%If any of the sections are not relevant to your manuscript, please include the heading and write `Not applicable' for that section. 
%
%%%===================================================%%
%%% For presentation purpose, we have included        %%
%%% \bigskip command. Please ignore this.             %%
%%%===================================================%%
%\bigskip
%\bigskip\noindent
%Springer journals and proceedings: \url{https://www.springer.com/gp/editorial-policies}
%
%\bigskip\noindent
%Nature Portfolio journals: \url{https://www.nature.com/nature-research/editorial-policies}
%
%\bigskip\noindent
%\textit{Scientific Reports}: \url{https://www.nature.com/srep/journal-policies/editorial-policies}
%
%\bigskip\noindent
%BMC journals: \url{https://www.biomedcentral.com/getpublished/editorial-policies}
%\end{flushleft}

\begin{appendices}

%\section{Example of Linear Model}\label{secA1}

%An appendix contains supplementary information that is not an essential part of the text itself but which may be helpful in providing a more comprehensive understanding of the research problem or it is information that is too cumbersome to be included in the body of the paper.

\section{Proof Details in Section 4 }

\subsection{Proof of Proposition \ref{prop1}}\label{secA2}

\begin{lemma}[\cite{xu1996}]\label{lemma2}
If $T \in \mathscr{L}(D)$ is continuously differentiable, then we have
\[L(T) = \sup_{x \in D} \{ \|T'(x)\|\}\].
\end{lemma}
\begin{proof}
Since  $T \in \mathscr{L}(D)$ is differentiable, based on the 
mean value theorem, we have 
\[\|Tx-Ty\| \leq \sup_{x \in D} \{ \|T'(x)\|\} \|x-y\|.\]
This implies that $L(T)\leq \sup_{x \in D} \{ \|T'(x)\|\}$.

On the other hand, since $T$ is differentiable on $D$, for any vector $h \in D$, such that $\|h\| = 1$, we have
\[
\lim\limits_{t \to 0} \frac{\|T(x+th)-T(x)-T'(x)\cdot t \cdot h\|}{t} = 0.
\]
This implies that 
\[
\lim\limits_{t \to 0} \frac{\|T(x+th)-T(x)\|}{\|t\cdot h\|} = \|T'(x) \cdot h\|,
\]
and thus $L(T) \geq \|T'(x) \cdot h\|$. By the arbitrariness of \( x \in D \), we have
\[
L(T) \geq  \sup_{x \in D} \{ \|T'(x)\|\}.
\]
Thus we have \[L(T) = \sup_{x \in D} \{ \|T'(x)\|\}\] completing the proof.
\end{proof}

In the following, we provide the proof details of proposition \ref{prop1}.
\begin{proof}
	(1) If $\mathrm{Lip}(T) <1$, then there exists a strong equivalence norm defined on $D$ such that $T$ is a contraction mapping under this metric.
	This implies that $T$ has a unique fixed point on $D$, denoted by $x^*$. 
	
Assume that \( T \) is contraction under the norm $\|\cdot\|$. Since \( T \) is continuously differentiable in the neighborhood of \( x^* \), \( \|T'(x)\| \) is continuous in the neighborhood of \( x^* \). Thus, for any sufficiently small positive number \( \varepsilon > 0 \), there exists \( \delta > 0 \) such that when \( \|x-x^*\| \leq \delta \), \( \left|\|T'(x)\| - \|T'(x^*)\|\right| \leq \varepsilon \). In particular, \( \|T'(x)\|  \leq \|T'(x^*)\| + \varepsilon \). 
Since \( D \) is bounded, by the contraction mapping principle, as \( n \to \infty \), \( \{T^n(x)\} \) converges uniformly to \( x^* \) for all \( x \in D \). Therefore, there exists a positive integer \( N \) such that for all \( n \geq N \), we have \( \|T^n(x) - x^*\| < \delta \), which implies \[ \|T'(T^n(x))\| \leq \|T'(x^*)\| + \varepsilon. \]

Since $T \in L(D)$ is differentiable, by Lemma \ref{lemma2}, we have $\|T'(x)\| \leq L(T) $. Also, for any positive integer \( n \), notice that
\[
(T^n)'(x) = \prod_{i=1}^n T'(T^{i-1}(x)),
\]
then we have
\begin{align*}
L(T^n)^{\frac{1}{n}}  & = \left\{ \sup_{x \in D} \{ \| (T^n)'(x) \|\} \right\}^{\frac{1}{n}} = \sup_{x \in D} \left\{  \| (T^n)'(x) \|^{\frac{1}{n}} \right\} \\
& \leq   \sup_{x \in D} \left\{  \left[ \prod_{i=1}^n \|T'(T^{i-1}(x))\| \right]^{\frac{1}{n}}    \right\} \\
& \leq   \sup_{x \in D} \left\{  \left[ \prod_{i=N+1}^n \|T'(T^{i-1}(x))\| \right]^{\frac{1}{n}}    \right\} \\
& \leq \left(\|T'(x^*)\| + \varepsilon \right)^{\frac{n-N}{n}},
\end{align*}
This implies that
\[
\text{Lip}(T) \leq \lim_{n \to \infty} \left[ \|T'(x^*)\| + \varepsilon \right]^{\frac{n-N}{n}} = \|T'(x^*)\| + \varepsilon.
\]
By the arbitrariness of \( \varepsilon \), we have \( \text{Lip}(T) \leq \|T'(x^*)\| \). Note that \( \text{Lip}(T) \) is independent of the choice of the equivalence norm defined on domain \( D \). Since \( \text{Lip}(T) \leq \|T'(x^*)\| \) holds for any equivalence norm defined on domain \( D \), we have 
\( \text{Lip}(T) \leq \rho(T'(x^*)) \).

On the other hand, for any positive integer \( n \),
\[
\left[ L(T^n)\right]^{\frac{1}{n}} \geq \|(T^n)'(x^*)\|^{\frac{1}{n}} = \left\| \prod_{i=1}^n T'(T^{i-1}(x^*))\right\|^{\frac{1}{n}}  = 
\|T'(x^*)^n\|^{\frac{1}{n}}, 
\]
which implies \( \text{Lip}(T) \geq \lim\limits_{n \to \infty} \|T'(x^*)^n\|^{\frac{1}{n}} = \rho(T'(x^*)) \).
Thus, we have \( \text{Lip}(T) = \rho(T'(x^*)) \). This completes the proof.

(2) We consider $\mathrm{Lip}(T) = 1$. Based on Theorem \ref{thlip}, $T$ has fixed points, denoted by $\mathrm{Fix}(T)$. 
Observing that $\forall x^* \in \mathrm{Fix}(T)$,  by the proof process of (1), we have \( \text{Lip}(T) \geq \rho(T'(x^*)) \). 
Therefore, we have
\[
 \text{Lip}(T) \geq \sup_{x^* \in \mathrm{Fix}(T)} \rho(T'(x^*)).
\]

On the other hand, for any sufficiently small positive number \( \varepsilon > 0 \), there exists \( \delta > 0 \) such that when \( \|x-x^*\| \leq \delta \), \( \left|\|T'(x)\| - \|T'(x^*)\|\right| \leq \varepsilon \). In particular, \( \|T'(x)\|  \leq \|T'(x^*)\| + \varepsilon \). For any positive integer \( n \), we have
\begin{align*}
	L(T^n) & =  \sup_{x \in D} \left\{ \| (T^n)'(x) \| \right\}   \\
	& \leq   \sup_{x \in D} \left\{ \prod_{i=1}^n \|T'(T^{i-1}(x))\|    \right\} \\
	& \leq \sup_{x^* \in \mathrm{Fix(T)}} \left(\|T'(x^*)\| + \varepsilon \right)^{n},  \\
	&  \leq  \left(\sup_{x^* \in \mathrm{Fix(T)}} \|T'(x^*)\| + \varepsilon \right)^{n},  
\end{align*}
and thus 
\begin{align*}
	\text{Lip}(T) = \lim\limits_{n \to \infty} [L(T^n)]^{\frac{1}{n}}  \leq \sup_{x^* \in \mathrm{Fix(T)}} \|T'(x^*)\| + \varepsilon, 
\end{align*}
By the arbitrariness of \( \varepsilon \), we have \( \text{Lip}(T) \leq \sup\limits_{x^* \in \mathrm{Fix(T)}} \|T'(x^*)\| \). Note that \( \text{Lip}(T) \) is independent of the choice of the equivalence norm defined on domain \( D \). Since \( \text{Lip}(T) \leq \sup\limits_{x^* \in \mathrm{Fix(T)}} \|T'(x^*)\| \) holds for any equivalence norm defined on domain \( D \), we have 
\( \text{Lip}(T) \leq \sup\limits_{x^* \in \mathrm{Fix(T)}} \rho(T'(x^*)) \).
Thus, we have \[ \text{Lip}(T) =\sup\limits_{x^* \in \mathrm{Fix(T)}} \rho(T'(x^*)). \] The proof is then completed.

(3) Firstly, we show  $\mathrm{Lip}(T) \leq  \sup\limits_{x\in D}  \rho(T'(x))$.

Since $L(T_1 \cdot T_2) \leq L(T_1) L(T_2), \forall T_1,T_2 \in \mathscr{L}(D)$, for any positive integer \( n \), we have 
\[
\left[ L(T^n)\right] \leq  \left[ L(T)\right]^n   = \left[ \sup\limits_{x\in D} \|T'(x)\| \right]^n,
\]
and thus 
\begin{align*}
	\text{Lip}(T) = \lim\limits_{n \to \infty} [L(T^n)]^{\frac{1}{n}}  \leq \sup\limits_{x\in D} \|T'(x)\|, 
\end{align*}
Note that \( \text{Lip}(T) \) is independent of the choice of the equivalence norm defined on domain \( D \). Since \( \text{Lip}(T) \leq \sup\limits_{x\in D} \|T'(x)\| \) holds for any equivalence norm defined on domain \( D \), we have 
\( \text{Lip}(T) \leq \sup\limits_{x\in D} \rho(T'(x)) \).

Now we show the reverse inequality \( \text{Lip}(T) \geq \sup\limits_{x\in D} \rho(T'(x)) \).

%Since $T$ is differentiable on $D$, for any vector $h \in D$, such that $\|h\| = 1$, we have
%\[
%\lim\limits_{t \to 0} \frac{\|T(x+th)-T(x)-T'(x)\cdot t \cdot h\|}{t} = 0.
%\]
%This implies that 
%\[
%\lim\limits_{t \to 0} \frac{\|T(x+th)-T(x)\|}{\|t\cdot h\|} = \|T'(x) \cdot h\|,
%\]
%and thus $L(T) \geq \|T'(x) \cdot h\|$. Therefore we have $L(T) \geq \|T'(x)\|$. By the arbitrariness of \( x \in D \), we have
%\[
%L(T) \geq \sup\limits_{x\in D}\|T'(x)\|.
%\]

Observing that 
\( \varepsilon > 0 \), by the definition of $\operatorname{Lip}(T)$, there exist  a positive integer \( n \), such that 
  \[ L(T^n) \leq (\operatorname{Lip}(T) + \varepsilon)^n. \]
 Let 
\[
d^*(x,y)  = \sum_{k=1}^n [\operatorname{Lip}(T) + \varepsilon]^{k-1} \cdot d(T^{n-k} x, T^{n-k} y), \quad \forall x, y \in D,
\]
then \( d^*\) is a metric on \( D \),  and satisfies that 
\[
(\operatorname{Lip}(T) + \varepsilon)^{n-1} d(x, y) \leq d^*(x,y) \leq \left\{ \sum_{k=1}^n [\operatorname{Lip}(T) + \varepsilon]^{k-1} [L(T)]^{n-k} \right\} d(x, y), \quad \forall x, y \in D.
\]
This implies that \( d^* \) and \( d \) are strongly equivalent. Under the metric $d^*$, we have
\begin{align*}
	d^*(Tx, Ty) \leq & [\operatorname{Lip}(T) + \varepsilon] \sum_{k=1}^n [\operatorname{Lip}(T) + \varepsilon]^{k-1} d(T^{n-k} x, T^{n-k} y) \\
	& + d(T^n x, T^n y) - [\operatorname{Lip}(T) + \varepsilon]^n d(x, y) \\
	\leq & [\operatorname{Lip}(T) + \varepsilon] d^*(x, y),
\end{align*}
and thus
\( L^*(T) \leq \operatorname{Lip}(T) + \varepsilon \), where $L^*(T)$ denotes the $L(T)$ under $d^*$. Without loss of generality, we choose $d^*$ as $\|\cdot\|$, then we can obtain
\[\operatorname{Lip}(T) + \varepsilon \geq L(T) \geq \sup\limits_{x\in D}\|T'(x)\|.\]
Since the above inequality holds for any equivalence norm defined on domain \( D \), we have 
\( \operatorname{Lip}(T)  \geq \sup\limits_{x\in D}\rho(T'(x))  -\varepsilon\). By the arbitrariness of \( \varepsilon \), we have 
\[\operatorname{Lip}(T)  \geq \sup\limits_{x\in D}\rho(T'(x))  -\varepsilon.\]

Thus, we have \( \text{Lip}(T) = \sup\limits_{x\in D} \rho(T'(x)) \). The proof is completed.
\end{proof}

\subsection{Projection and Averaged Operators Are Asymptotically Regular } \label{regu}

We first clarify that projection operators and averaged operators naturally satisfy the following squared descent condition:
\begin{equation}
	\|Tx-p\|^{2}
	\leq
	\|x-p\|^{2}
	-c\|x-Tx\|^{2},
	\qquad
	x\in D,\quad p\in \operatorname{Fix}(T),
	\label{eq:squared-descent}
\end{equation}
for some constant \(c>0\). On a bounded domain or along a bounded iterative orbit, condition \eqref{eq:squared-descent} further implies
\begin{equation}
	\|Tx-p\|
	\leq
	\|x-p\|
	-
	\psi\bigl(\|x-Tx\|\bigr),
	\qquad
	x\in D,\quad p\in \operatorname{Fix}(T),
	\label{eq:non-squared-descent}
\end{equation}
for an appropriate strictly increasing function \(\psi\colon \mathbb{R}_{+}\to\mathbb{R}_{+}\) satisfying \(\psi(0)=0\).

\subsubsection{Projection Operators}

Let \(C\) be a nonempty closed convex subset of a Hilbert space \(H\), and let
\[
P_Cx = \arg\min_{y\in C} \|x-y\|, \forall x \in H,
\]
be the metric projection onto \(C\). Clearly, \(\operatorname{Fix}(T)=C.\)
For any \(x\in H\) and \(p\in C\), the characterization of the metric projection gives
\begin{equation}
	\left\langle x-P_Cx,\;p-P_Cx\right\rangle
	\leq 0.
	\label{eq:projection-characterization}
\end{equation}
Equivalently,
\[
\left\langle x-P_Cx,\;P_Cx-p\right\rangle
\geq 0.
\]
Since
\[
x-p=(x-P_Cx)+(P_Cx-p),
\]
we obtain
\begin{align}
	\|x-p\|^2
	&=
	\|x-P_Cx\|^2
	+
	\|P_Cx-p\|^2
	\nonumber\\
	&\quad
	+
	2\left\langle x-P_Cx,\;P_Cx-p\right\rangle
	\nonumber\\
	&\geq
	\|x-P_Cx\|^2
	+
	\|P_Cx-p\|^2.
\end{align}
Therefore,
\begin{equation}
		\|P_Cx-p\|^2
		\leq
		\|x-p\|^2
		-
		\|x-P_Cx\|^2
	\label{eq:projection-descent}
\end{equation}
for every \(x\in H\) and \(p\in C\). Hence, projection operator $P_C$ is of quasi-asymptotic regularity property with $\psi(t)=t^2$.

%the metric projection satisfies \eqref{eq:squared-descent} with \(c=1\). 
%
%Geometrically, \(P_Cx\) is not merely no farther from any point \(p\in C\) than \(x\); its squared distance to \(p\) decreases by at least
%\(\|x-P_Cx\|^2.\)
%Moreover, since \(P_C^2=P_C,\)
%we have
%\[
%P_C^{n+1}x=P_C^nx,
%\qquad n\geq 1.
%\]
%Thus, the metric projection is asymptotically regular.

\subsubsection{Averaged Operators}

Let
\begin{equation}
	T_{\alpha}=(1-\alpha)I+\alpha T,
	\qquad
	0<\alpha<1,
	\label{eq:averaged-operator}
\end{equation}
be the \(\alpha\)-average operator of nonexpansive operator $T$.
Since
\(T_{\alpha}x=x
\quad\Longleftrightarrow\quad
Tx=x,\)
we have
\[
\operatorname{Fix}(T_{\alpha})=\operatorname{Fix}(T).
\]
Let \(p\in \operatorname{Fix}(T_{\alpha})\). Since \(Tp=p\), it follows that
\[
T_{\alpha}x-p
=
(1-\alpha)(x-p)
+
\alpha(Tx-p).
\]
Using the identity
\[
\|(1-\alpha)u+\alpha v\|^2
=
(1-\alpha)\|u\|^2
+
\alpha\|v\|^2
-
\alpha(1-\alpha)\|u-v\|^2,
\]
with \(u=x-p,
v=Tx-p,\)
we obtain
\begin{align}
	\|T_{\alpha}x-p\|^2
	=
	(1-\alpha)\|x-p\|^2
	+
	\alpha\|Tx-p\|^2
	-
	\alpha(1-\alpha)\|x-Tx\|^2.
\end{align}
Since \(T\) is nonexpansive and \(Tp=p\), \(\|Tx-p\|
=
\|T_{\alpha}x-Tp\|
\leq
\|x-p\|.\)
Consequently,
\[
\|Tx-p\|^2
\leq
\|x-p\|^2
-
\alpha(1-\alpha)\|x-Tx\|^2.
\]
Furthermore, \(x-T_{\alpha}x
=
\alpha(x-Tx),\)
and hence
\[
\|x-Tx\|^2
=
\frac{1}{\alpha^2}\|x-T_{\alpha}x\|^2.
\]
Therefore,
\begin{equation}
		\|T_{\alpha}x-p\|^2
		\leq
		\|x-p\|^2
		-
		\frac{1-\alpha}{\alpha}
		\|x-T_{\alpha}x\|^2, x\in H, p\in \operatorname{Fix}(T_{\alpha}).
	\label{eq:averaged-descent}
\end{equation} Thus, the \(\alpha\)-averaged operator $T_{\alpha}$ is of quasi-asymptotic regularity property with $\frac{1-\alpha}{\alpha}\psi(t)=t^2$.

\subsubsection{Residual Connection Type Operators}

The residual connection type operator $T=I+A$ can be written as
\[T = I + A = (1-\alpha)I + \alpha\left(I+\frac{1}{\alpha}A\right), 0<\alpha<1. \]
If $A$ is a dissipative operator \cite{lumer1961dissipative}, i.e., $m(A) \leq 0$ or $m(-A) \geq 0$. Then we have
 \[
m(-A)=\lim_{h\to 0^+}\frac{1-L(I+h\cdot A)}{h}.
\]
Since $m(-A) >0$, it follows from the definition of the one-sided limit that, for any
\(\varepsilon\in(0,m(-A))\), there exists \(\delta>0\) such that
\[
\frac{1-L(I+h\cdot A)}{h}\ge m(-A)-\varepsilon,
\]
whenever \(0<h<\delta\). In other word, we have
\[
L(I+h\cdot A) \leq 1-(m(-A)-\varepsilon)\cdot h \leq 1.
\]
It yields $L(I+h \cdot A)\leq 1$ since \(0<\alpha <\delta\).
If $m(-A)=0$, then $L(I+h\cdot A) = 1$.
Then $I+\frac{1}{\alpha}A$ is a non-expansive operator. Evidently, $T$ corresponds to the average version of non-expansive operator $I+\frac{1}{\alpha}A$, and thus $T$ is a non-expansive operator and of quasi-asymptotic regularity property.

%
%satisfies \eqref{eq:squared-descent} with \(c=\frac{1-\alpha}{\alpha}.\)
%
%In particular, when \(\alpha=\frac{1}{2}\), we have
%\[
%\|Tx-p\|^2
%\leq
%\|x-p\|^2
%-
%\|x-Tx\|^2.
%\]
%In this case, \(T\) is firmly nonexpansive. Metric projections are typical examples of firmly nonexpansive operators.

%\subsubsection{From the Squared Condition to the Nonsquared Condition}
%
%Suppose that \(T\) satisfies
%\[
%\|Tx-p\|^2
%\leq
%\|x-p\|^2
%-
%c\|x-Tx\|^2.
%\]
%Let \(a=\|Tx-p\|,
%b=\|x-p\|,
%r=\|x-Tx\|,\)
%then \(b^2-a^2\geq cr^2\).
%Since \(b^2-a^2=(b-a)(b+a),\)
%it follows that \(	b-a
%\geq
%\frac{cr^2}{b+a}.\)
%
%Assume that the domain under consideration, or the iterative orbit, is bounded in the sense that there exists \(R>0\) such that \(\|x-p\|\leq R,
%\|Tx-p\|\leq R.\)
%Then \(a+b\leq 2R.\)
%Therefore, it yields \(b-a
%\geq
%\frac{c}{2R}r^2.\)
%Equivalently,
%\begin{equation}
%		\|Tx-p\|
%		\leq
%		\|x-p\|
%		-
%		\frac{c}{2R}\|x-Tx\|^2.
%	\label{eq:nonsquared-from-squared}
%\end{equation}
%Thus, condition \eqref{eq:non-squared-descent} holds with \(\psi(t)=\frac{c}{2R}t^2.\)
%For the metric projection \(P_C\), since \(c=1\), one may take \(\psi(t)=\frac{t^2}{2R}.\)
%
%For an \(\alpha\)-averaged operator, since \(c=\frac{1-\alpha}{\alpha},\) one may take
%\[
%\psi(t)
%=
%\frac{1-\alpha}{2\alpha R}t^2.
%\]

\subsection{Lip Number $\mathrm{Lip}(T)$ for Different Basic Functional Blocks} 

\subsubsection{Self-Attention} \label{attention}
The single head self-attention \( T: \mathbb{R}^{N \times d}  \to  \mathbb{R}^{N \times d_V}\) can be written as
\[
T(X) = PXW^V = \text{softmax}\left(X A^{\top} X^{\top}\right)XW^V = \begin{bmatrix} T_1(X)^{\top} \\ \vdots \\ T_N(X)^{\top} \end{bmatrix} \in \mathbb{R}^{N \times d_V}, 
\]
where \( A = W^K W^{Q^{\top}} / \sqrt{d} \in \mathbb{R}^{d \times d} \) and \( T_i(X) = \sum_{j=1}^N P_{ij} X_j W^V\) with \( P_{i:}^{\top} = \text{softmax}\left(X A X_i\right) \). 
%Hence \( f \) can be interpreted as a map of each \( x_i \) to a point in the convex hull of \( x_1, \ldots, x_N \). Since \( f \) is a map from \( \mathbb{R}^{N \times D} \) to \( \mathbb{R}^{N \times D} \), 
Then the Lip number $\mathrm{Lip}(T) = \sup\limits_{X \in \mathbb{R}^{N \times d}} \rho(T'(X))$, where
\[
T'(X) = \begin{bmatrix} T'_{11}(X) & \ldots & T'_{1N}(X) \\ \vdots & \ddots & \vdots \\ T'_{N1}(X) & \ldots & T'_{NN}(X) \end{bmatrix} \in \mathbb{R}^{Nd_V \times Nd},
\]
where \( T'_{ij}(X) = \frac{\partial T_i(X)}{\partial X_j} \in \mathbb{R}^{d_V \times d} \). By taking partial derivatives we can show that \( T'_{ij}(X) = W^{V^{\top}} X^{\top} P^{(i)} \left[ E_{ji} X A^{\top} + X A \delta_{ij} \right] + P_{ij} W^{V^{\top}} \) where \( E_{ij} \in \mathbb{R}^{N \times N} \) is a binary matrix with zeros everywhere except the \( (i, j) \)-th entry, \( \delta_{ij} \) is the Kronecker delta, and \( P^{(i)} := \operatorname{diag}(P_{i:}) - P_{i:}^{\top} P_{i:} \). And thus for \( i = j \), we have
\begin{align*}
T_{ii}(X) & = W^{V^{\top}} X^{\top} P^{(i)} E_{ii} X A^{\top} + W^{V^{\top}} X^{\top} P^{(i)} X A + P_{ii} W^{V^{\top}} \\
& =  P_{ii} \left( X_i - \sum_k P_{ik} X_k \right)  (XW^V)_i^{\top} A^{\top} + W^{V^{\top}} X^{\top} P^{(i)} X A + P_{ii}W^{V^{\top}}. 
\end{align*}
Consider the case \( X_i = 0 \). Then \( P_{i:}^{\top} = \text{softmax}\left(X A X_i\right) = \frac{1}{N} \mathbf{1} \). The first term of \( T_{ii}(X) \) disappears since \( X_i = 0 \), and the last term becomes \( \frac{1}{N} W^V \). For the second term, the entries \( [X^{\top} P^{(i)} X]_{ll} = \operatorname{Var}(X_l) \) are unbounded since the latter is equal to the sample variance of \( x_{1l}, \ldots, x_{Nl} \), which can be arbitrarily large. This means $\mathrm{Lip}(T) = \sup\limits_{X \in \mathbb{R}^{N \times d}} \rho(T'(X)) = \infty$.

On the other hand, if $\|X\|_F \leq B$, then we can estimate the $\rho(T'(X))$ by
\[\rho(T'(X)) \leq \|T'(X)\|_2. \]
Observe that $T'(X)$ is a linear operator mapping from $\Delta X$ to $\Delta T(x)$, and 
\begin{align*}
\|\Delta T(X)\|_2 \leq  \|\Delta P\|_F \|XW^V \|_F + \|P\Delta X W^V\|_F,
\end{align*}
where $\Delta T = \Delta P XW^V + P \Delta X W^V, \Delta S = \Delta(XA^\top X^\top) = \Delta X A^\top X^\top + XA^\top (\Delta X)^\top $.
Since $\|\Delta P\|_F \leq \frac{1}{2} \|\Delta S\|_F$, $\|PZ\|_F \leq \|Z\|_F$, we have
\begin{align*}
\|\Delta S\|_F \leq \|\Delta X A^\top X^\top\|_F + \|XA^\top (\Delta X)^\top\|_F \leq 2 \|A\|_2 \|X\|_F \|\Delta X\|_F \leq 2 \|A\|_2 B \|\Delta X\|_F,
\end{align*}
and
\begin{align*}
	\|XW^V \|_F \leq \|X\|_F \|W^V\|_F \leq B\|W^V\|_F.
\end{align*}
Then we can obtain
\begin{align*}
	\|\Delta T\|_F \leq \left(\|W^V\|_2 + B^2 \|A\|_2 \|W^V\|_2 \right) \|\Delta X\|_F.
\end{align*}
Observing that 
$\|\Delta T \| \leq \|T'(X) \| \|\Delta X\|$, and thus 
\begin{align*}
	\sup_{\|X\|_F \leq B} \rho(T'(X)) \leq 	\sup_{\|X\|_F \leq B} \|T'(X)\|_2\leq \|W^V\|_2 (1+  B^2 \|A\|_2 ).
\end{align*}
By $\|A\|_2 \leq \frac{\|W^K\|_2 \|W^Q\|_2}{\sqrt{d}}$, we have
\begin{align*}
	\sup_{\|X\|_F \leq B} \rho(T'(X)) & \leq 	\sup_{\|X\|_F \leq B} \|\Delta T\|_F \leq \|W^V\|_2 \left(1+  \frac{B^2}{\sqrt{d}} \|W^K\|_2 \|W^Q\|_2 \right) \\
& \leq 	\rho(W^V) \left(1+  \frac{B^2}{\sqrt{d}} \rho(W^K) \rho(W^Q) \right),
\end{align*}
which means that
$\mathrm{Lip}(T) = \sup\limits_{X \in \mathbb{R}^{N \times d}} \rho(T'(X)) \leq \rho(W^V) \left(1+  \frac{B^2}{\sqrt{d}} \rho(W^K) \rho(W^Q) \right)$.

\subsubsection{LayerNorm } \label{layernorm}

For \(x\in\mathbb{R}^d\), we define
\[
\mu(x)=\frac{1}{d}\mathbf{1}^\top x,\qquad x_c = x-\mu(x)\mathbf{1},
\]
\[
s(x)=\frac{1}{d}\|x_c\|_2^2,\qquad
\sigma(x)=\sqrt{s(x)+\varepsilon},\quad \varepsilon>0.
\]

The normalized part of LayerNorm is 
\[
\mathrm{LN}_0(x)=\frac{x_c}{\sigma(x)},
\]
and the LayerNorm could be written as
\[
\mathrm{LN}(x)=\gamma\odot \mathrm{LN}_0(x)+\beta.
\]
Let
\[
P = I-\frac{1}{d}\mathbf{1}\mathbf{1}^\top
\]
be the orthogonal projection onto the zero-mean subspace, then \(Px=x_c\) and \(\|P\|_2=1\).
The Jacobian of $\mathrm{LN}_0(x)$ is given by
\[
J(x)=\nabla \mathrm{LN}_0(x)= \frac{1}{\sigma(x)}
\left(
P-\frac{1}{d\sigma(x)^2}x_c x_c^\top
\right).
\]

Define
\[
a(x)=\frac{\|x_c\|_2^2}{d,\sigma(x)^2}
=\frac{s(x)}{s(x)+\varepsilon}\in[0,1),
\]
and let \(u=x_c/\|x_c\|_2)\) when \(x_c\neq0\). Then
\[
\frac{1}{d\sigma(x)^2}x_c x_c^\top = a(x)uu^\top.
\]
Let \[
M(x):=P-a(x)uu^\top,
\]
and we have
\[
\|M(x)\|_2=1.
\]
Observing that
\[
\|J(x)\|_2
=\frac{1}{\sigma(x)}\|M(x)\|_2
=\frac{1}{\sigma(x)}
\le \frac{1}{\sqrt{\varepsilon}},
\]
since \(\sigma(x)=\sqrt{s(x)+\varepsilon}\ge\sqrt{\varepsilon}\).
And thus one has
\[
\sup_{x\in \mathbb{R}^d} \|J(x)\|_2 \le \frac{1}{\sqrt{\varepsilon}}.
\]
Therefore, $\mathrm{Lip}(LN) = \sup\limits_{x \in \mathbb{R}^{d}} \rho(\gamma J(x)) \leq  \max_i |\gamma_i|\epsilon^{-\frac{1}{2}} $.

\subsubsection{Transformer with Post-LayerNorm}  \label{postlayernorm}

 Since $T(X) = T_2 \circ (I+T_3) \circ T_2 \circ (I+ T_1)(X)$, and  the diagonal element
 of $\frac{\partial T_1(X)}{\partial X}$ could be unbounded as illustrated in Section \ref{attention}, then 
 \begin{align*}
 	\mathrm{Lip}(T)& = \sup\limits_{X \in \mathbb{R}^{N \times d}} \rho(T'(X)) \\
 	&= \sup\limits_{X \in \mathbb{R}^{N \times d}} \rho\left(\frac{\partial T(X)}{\partial (I+T_1)(X) }  \frac{\partial (I+T_1)(X)}{\partial X}\right) \\
 	&= \sup\limits_{X \in \mathbb{R}^{N \times d}}  \rho\left( \frac{\partial T(X)}{\partial (I+T_1)(X) }  \frac{\partial T_1(X)}{\partial X} \right)=\infty.
 \end{align*}

\subsubsection{Transformer with Pre-LayerNorm} \label{prelayernorm}

%\begin{example}[Transformer with Pre-LayerNorm] 
%	Consider the single head Transformer model $T(X) = T_2 \circ T_1 (X)$ with $T_1 = I + F_1, T_2 = I + F_2$, and $F_1= MLP\circ LN, F_2 = MSA\circ LN$, $I=$ identity operator, where $MLP, MSA, LN$ are MLP model, Multi-Head Self-Attention Model and LayerNorm defined in Example \ref{ex1}, \ref{eq5} and \ref{ex6}, respectively.
%	%	$(I+T_3 \circ T_2) \circ  (I+ T_1\circ T_2 )(X)$, which is composed by the basic functional blocks
%	%	$T_1(X)$= Multi-Head Self-Attention Model  defined in Example \ref{eq5},
%	%	$T_2(X) = \mathrm{LayerNorm}(X)$ defined in Example \ref{ex6},
%	%	$T_3(X) = \mathrm{MLP}(X)$ defined in Example \ref{ex1}, 
%	%	and the identity operator $I$. 
%	If $I-T_1, I-T_2$ are dissipative operators, that is $m(I-T_1)\leq 0, m(I-T_2) \geq 0$,
%	then the Lip number of $T$ satisfies $\mathrm{Lip}(T) \leq 1$.
%	%	\begin{align*}
%		%		\mathrm{Lip}(T) &\leq  (1 + \max_i |\gamma_i|\epsilon^{-\frac{1}{2}} \rho(W_1)\rho(W_2))\cdot \\
%		%		&	\left(1+ \max_i |\gamma_i|\epsilon^{-\frac{1}{2}} \rho(W^V) \left(1+  \frac{(\|\gamma\|_{\infty}\sqrt{d} + \|\beta\|_2)^2}{\sqrt{d}} \rho(W^K) \rho(W^Q) \right)\right).
%		%	\end{align*}
%\end{example}

\begin{lemma} \label{ssss}
	Considering $T(X) = (I+\alpha \cdot F)(X)$, where $F$ is Lipschitz operator, and $0 < \alpha < \delta$ for some positive constant $\delta>0$. If $I-T$ is an accretive operator, i.e., $m(I-T) \geq 0$, then $\mathrm{Lip}(T) \leq 1$.
\end{lemma}
\begin{proof}
Recall that
 \[
 m(F)=\lim_{h\to 0^+}\frac{1-L(I-h\cdot F)}{h},
 \]
 then we have
 \[
 m(-F)=\lim_{h\to 0^+}\frac{1-L(I+h\cdot F)}{h}.
 \]
 Since $m(I-T) \geq 0$, we have $m(-F) \geq 0$.
 If $m(-F)>0$, it follows from the definition of the one-sided limit that, for any
 \(\varepsilon\in(0,m(F))\), there exists \(\delta>0\) such that
 \[
 \frac{1-L(I+h\cdot F)}{h}\ge m(-F)-\varepsilon,
 \]
 whenever \(0<h<\delta\). In other word, we have
 \[
L(I+h\cdot F) \leq 1-(m(-F)-\varepsilon)\cdot h \leq 1.
 \]
 It yields $L(I+h \cdot F)\leq 1$ since \(0<\alpha <\delta\).
 
 If $m(-F)=0$, then $L(I+h\cdot F) = 1$.
 
 Therefore, we have $\mathrm{Lip}(T) \leq L(I+h\cdot F) \leq 1$. 
\end{proof}

To compute the Lip number of $T$ defined in Example \ref{ex8}, we firstly illustrate that the LayerNorm operator could map an unbounded input \(x\in\mathbb{R}^d\) to a bounded output. Following the definition in Section \ref{layernorm}, we can compute the norm of $\mathrm{LN}_0(x)$, i.e.,
\[
\|\mathrm{LN}_0(x)\|_2
= \frac{\|x_c\|_2}{\sqrt{\tfrac1d\|x_c\|_2^2+\varepsilon}}
\le \frac{\|x_c\|_2}{\sqrt{\tfrac1d\|x_c\|_2^2}}
= \sqrt{d}.
\]
Hence,
\[
\|\mathrm{LN}(x)\|_2 \le \|\gamma\|_{\infty}\sqrt{d} + \|\beta\|_2, \quad \forall x\in\mathbb{R}^d .
\]
In practice, $\|\gamma\|_{\infty}$ and $\|\beta\|_2$ are set to be bounded, then the output of $\mathrm{LN}(x)$ lies inside a bounded sphere even if \(\|x\|\to\infty\).

Secondly, we can claim that the $F_1, F_2$ are Lipschitz operators. Since 
\begin{align*}
	\mathrm{Lip}(F_1) & = \mathrm{Lip}(MLP\circ LN) \leq \mathrm{Lip}(MLP)\cdot \mathrm{Lip}(LN) \\
	& \leq \sqrt{\lambda_{\max}(W_2^\top W_2)\lambda_{\max}(W_1^\top W_1)} \cdot  \|\gamma\|_{\infty}\epsilon^{-\frac{1}{2}},
	\end{align*}
thus $F_1$ is a Lipschitz operator. Besides, since $LN(X)$ is a bounded set for all $x\in \mathbb{R}^d$, then $\mathrm{Lip}(MSA \circ LN)$ is bounded as demonstrated in Section \ref{attention}. Actually, $LN(X)$ could be treated as mapping $\mathbb{R}^d$ to a bounded sphere, and then attention operator tends to mapping the bounded sphere to the bounded sphere. Thus $F_2= MSA\circ LN$ is also a Lipschitz operator.

%Lip number of $T_1\circ T_2$ is bounded. Since $T_2(X)$ is a bounded set for all $x\in \mathbb{R}^d$, $\mathrm{Lip}(T_1\circ T_2)$ is bounded as demonstrated in Section \ref{attention}. Actually, $T_2(X)$ could be treated as mapping $\mathbb{R}^d$ to a bounded sphere, and then attention operator and MLP model tend to mapping the bounded sphere to the bounded sphere.

Thildly, we kindly claim that $I+F_1 = I + \alpha_1 \cdot (F_1\cdot \frac{1}{\alpha_1})$, $I+F_2 = I + \alpha_1 \cdot (F_2 \cdot \frac{1}{\alpha_1})$, where $0< \alpha_1, \alpha_2<\delta$ for some positive constant $\delta>0$. Acturally, the term $\frac{1}{\alpha_1}, \frac{1}{\alpha_2}$ can be absorbed into the model weight $W_2$ of MLP model and the model weight $W_1^V, \cdots, W_K^V$ of multi-head self-attention model.

Based on Lemma \ref{ssss}, we have $ \mathrm{Lip}(T_1)\leq 1,  \mathrm{Lip}(T_2)\leq 1$. Thus $\mathrm{Lip}(T) = \mathrm{Lip}(T_2 \circ T_1) \leq \mathrm{Lip}(T_2) \cdot \mathrm{Lip}(T_1) \leq 1$.

Furthermore, we will show that $T$ satisfies quasi-asymptotic regularity property.

	Let \(	T_1=(1-\alpha_1)I+\alpha_1 R_1,
	T_2=(1-\alpha_2)I+\alpha_2 R_2,\)
	where \(	R_1=I+F_1,
	R_2=I+F_2\)
	are nonexpansive operators and \(0<\alpha_1,\alpha_2<1\).
	Then \(T_1\) and \(T_2\) are \(\alpha_1\)- and \(\alpha_2\)-averaged operators,
	respectively.
	
	We shall prove that
	\[
	T=T_1\circ T_2
	\]
	is also an averaged operator.
	
	Recall that an operator \(S\) is \(\alpha\)-averaged if and only if
	\[
	\|Sx-Sy\|^2
	\leq
	\|x-y\|^2
	-
	\frac{1-\alpha}{\alpha}
	\|(I-S)x-(I-S)y\|^2 ,
	\qquad \forall x,y.
	\]	
	Since \(T_i\) is \(\alpha_i\)-averaged, we have
	\[
	\|T_i x-T_i y\|^2
	\leq
	\|x-y\|^2
	-
	\nu_i
	\|(I-T_i)x-(I-T_i)y\|^2 ,
	\qquad i=1,2,
	\]
	where \(	\nu_i:=\frac{1-\alpha_i}{\alpha_i}>0.\)
	
	For \(T=T_1\circ T_2\), we have
	\[
	\begin{aligned}
		\|Tx-Ty\|^2
		&=
		\|T_1T_2x-T_1T_2y\|^2  \\
		&\leq
		\|T_2x-T_2y\|^2
		-
		\nu_1
		\|(I-T_1)T_2x-(I-T_1)T_2y\|^2 .
	\end{aligned}
	\]
	Using the averagedness of \(T_2\), we further get
	\[
	\|T_2x-T_2y\|^2
	\leq
	\|x-y\|^2
	-
	\nu_2
	\|(I-T_2)x-(I-T_2)y\|^2 .
	\]
	Therefore,
	\[
	\begin{aligned}
		\|Tx-Ty\|^2
		\leq{}&
		\|x-y\|^2
		-
		\nu_2
		\|(I-T_2)x-(I-T_2)y\|^2  \\
		&-
		\nu_1
		\|(I-T_1)T_2x-(I-T_1)T_2y\|^2 .
	\end{aligned}
	\]	
	Define
\(	a:=(I-T_2)x-(I-T_2)y\)
	and
\(b:=(I-T_1)T_2x-(I-T_1)T_2y .\)
	Then
	\[
	\begin{aligned}
		(I-T)x-(I-T)y
		&=
		(x-T_1T_2x)-(y-T_1T_2y)  \\
		&=
		(x-T_2x)-(y-T_2y)  \\
		&\quad +
		(T_2x-T_1T_2x)-(T_2y-T_1T_2y)  \\
		&=
		a+b .
	\end{aligned}
	\]	
	By the weighted Cauchy--Schwarz inequality,
	\[
	\nu_2\|a\|^2+\nu_1\|b\|^2
	\geq
	\frac{\nu_1\nu_2}{\nu_1+\nu_2}\|a+b\|^2 .
	\]
	Hence,
	\[
	\|Tx-Ty\|^2
	\leq
	\|x-y\|^2
	-
	\nu
	\|(I-T)x-(I-T)y\|^2 ,
	\]
	where
	\(	\nu:=\frac{\nu_1\nu_2}{\nu_1+\nu_2}>0.\)
	Let \(\alpha\in(0,1)\) be defined by \(	\nu=\frac{1-\alpha}{\alpha}.\)
	Equivalently, we have
	\[
	\alpha=\frac{1}{1+\nu}.
	\]
	Since \(	\nu_i=\frac{1-\alpha_i}{\alpha_i},\)
	we obtain
	\[
	\nu
	=
	\frac{
		\frac{1-\alpha_1}{\alpha_1}
		\frac{1-\alpha_2}{\alpha_2}
	}{
		\frac{1-\alpha_1}{\alpha_1}
		+
		\frac{1-\alpha_2}{\alpha_2}
	}.
	\]
	Thus,
	\[
	\alpha
	=
	\frac{\alpha_1+\alpha_2-2\alpha_1\alpha_2}
	{1-\alpha_1\alpha_2}.
	\]	
	Consequently,
	\[
	\|Tx-Ty\|^2
	\leq
	\|x-y\|^2
	-
	\frac{1-\alpha}{\alpha}
	\|(I-T)x-(I-T)y\|^2 .
	\]
	This shows that \(T=T_1\circ T_2\) is \(\alpha\)-averaged, where \(		\alpha
	=
	\frac{\alpha_1+\alpha_2-2\alpha_1\alpha_2}
	{1-\alpha_1\alpha_2}\).

\section{Existence of the Limit Architecture and Emergent Intelligence}  \label{appendixb}
\subsection{Preliminaries}

We first recall some conclusions of fixed points.

\begin{lemma} [\cite{browder1965nonexpansive_pnas,kirk1965fixed,ray1980fixed}] \label{lemma1}
	Let $K$ be a closed and convex subset of a real Hilbert space $H$. Then $K$ has the fixed point property for nonexpansive mappings if and only if $K$ is bounded.
\end{lemma}

\begin{theorem} [Demiclosedness principle for asymptotically nonexpansive mapping \cite{xu1991existence}] \label{th12}
	
	Let \( D \) be a nonempty closed convex bounded subset of a uniformly convex Banach space \( X \) and let \( T: D \rightarrow D \) be an asymptotically nonexpansive mapping (\( T \) {is said to be asymptotically nonexpansive if there exists a sequence \( \{k_n\} \) {of real numbers with} \(\lim_{n \to \infty} k_n = 1\) such that } $\|T^n x - T^n y\| \leq k_n \|x - y\| \quad \text{for} \; x, y \in C \; \text{and} \; n = 1, 2, \ldots$ ). Then \( (I - T) \) is demiclosed at zero, i.e. for each sequence \( \{x_n\} \) in \( C \), the conditions \( x_n \rightharpoonup x \) weakly and \( (I - T)x_n \rightarrow 0 \) strongly imply \( (I - T)x = 0 \). 
	
\end{theorem}

	\begin{lemma} [\cite{reich1979weak}] \label{eqnonexpansive}
	\textit{Let $D$ be a closed convex subset of a uniformly convex Banach space with a Fréchet differentiable norm, and let $\{S_n\}$ be a sequence of nonexpansive self-mappings of $D$ with a nonempty common fixed point set $F$. If $x_1 \in D$ and $x_{n+1} = S_n x_n$ for $n \geq 1$, then}
	\[
	\lim_{n \to \infty} \langle x_n, J(f_1 - f_2) \rangle \text{ exists for all } f_1, f_2 \in F,
	\]
	where $J$ is the duality map from $D$ to the dual space $D^* $ given by
	\[J(x)=\{x^* \in D^*: \langle x,x^*\rangle=\|x\|^2=\|x^*\|^2\}, \forall x \in D.\]
	\textit{In particular,}
	\[
	\langle q_1 - q_2, J(f_1 - f_2) \rangle = 0,
	\]
	\textit{where $f_1, f_2 \in F$ and $q_1, q_2$ are weak limit points of $\{x_n\}$.}
\end{lemma}

	\begin{lemma} \label{lm1}
	Suppose that \(\{a_n\}\), \(\{b_n\}\) and \(\{c_n\}\) are sequences of nonnegative real numbers satisfying the inequality
	\[ a_{n+1} \leq (1 + b_n)a_n + c_n, \quad \text{for all } n \in \mathbb{N}. \]
	If \(\sum_{n=1}^{\infty} b_n < \infty\) and \(\sum_{n=1}^{\infty} c_n < \infty\), then \(\lim_{n \to \infty} a_n\) exists. In particular, if \(\liminf_{n \to \infty} a_n = 0\), then \(\lim_{n \to \infty} a_n = 0\).
\end{lemma}

\subsection{Lipschitz Dual Operator} \label{dual}

Generally, let $E$ be Banach space, $E^*$ be the dual space of $E$, and $D$ are the closed subset of $E$.
Let \( \mathscr{L}(D) = \{ T : D \to D \mid T \text{ is a Lipschitz operator} \} \). It is straightforward to verify that \( L(\cdot) \) is a seminorm on \( \mathscr{L}(D) \). Specifically, if \( T \) is a bounded linear operator from \( E \) to \( E \), then the restriction of $T$ on $D$ satisfies that
\( T_D \in \mathscr{L}(D, E) \) on \( D \), and \( L(T_D) \leq \|T\| \).

\begin{lemma}
	Let \( 0 \in C \), and \( \mathscr{L}_0(D) = \{ T \in\mathscr{L}(D) \mid T(0) = 0 \} \), then \( L(\cdot) \) is a norm on \( \mathscr{L}_0(D) \), and \( (\mathscr{L}_0(D), L(\cdot)) \) is a Banach space.
\end{lemma}
\begin{proof}
 It is easy to verify that \( L(\cdot) \) is a norm on \( \mathscr{L}_0(D) \). Let \( \{T_n\} \subset \mathscr{L}_0(D) \) be a Cauchy sequence, then \( \forall x \in D \), \( \{T_nx\} \subset X \) is a Cauchy sequence. Hence, \( \{T_nx\} \) converges, and we denote the limit by \( Tx \). Since \( D \) is closed, we have \( Tx \in D \). \( \forall x, y \in D, x \neq y \), take \( n \) sufficiently large such that the following holds:
\[
\|Tx - T_{n}x\| \leq \|x - y\|, \quad \|Ty - T_{n}y\| \leq \|x - y\|,
\]
then
\[
\|Tx - Ty\| \leq \|Tx - T_nx\| + \|Ty - T_ny\| + \|T_nx - T_{n}y\| \leq 2\|x - y\|+L(T_n)\|x - y\| .
\]
Since \( L(T_n) \) is bounded, \( T \in \mathscr{L}_0(D) \). Hence \( \mathscr{L}_0(D) \) is complete under \( L(\cdot) \). The proof is then completed.
\end{proof}

In the following, we assume \( 0 \in C \) without loss of generality.
\begin{Definition}
		 Let \( \mathbb{K} \) be the number field, the Banach space \(\mathscr{L}_0(D, \mathbb{K}) \) is called the Lipschitz dual space of \( D \),  and denote it by \( D_L^* \).
\end{Definition}

\begin{proposition}
For any \( x \in D \),  
\[
\|x\| = \sup_{f \in D_L^*, L(f) \leq 1} |f(x)|.
\]
\end{proposition}

\begin{Definition}
	Let \( T \in \mathscr{L}_0(D) \), and define the mapping \( T_L^*: D_L^* \to D_L^* \) as follows:
	 \[
	 (T_L^* f)(x) = f(Tx), \quad \forall f \in D_L^*, \, x \in D,
	 \]  
	 then \( T_L^* \) is called the Lipschitz dual operator of \( T \).
\end{Definition}
	
The following proposition construct the relationship between \( T_L^* \) and \(T\), demonstrating that the $\mathrm{Lip}(T) $ could be computed by $\rho(T_L^*)$.
\begin{proposition}
Let \( T \in \mathscr{L}_0(D) \), then \( T_L^* \) is a bounded linear operator on \( D_L^* \), and \( \rho(T_L^*) = \mathrm{Lip}(T) \).
\end{proposition}

\begin{proof}
It is easy to prove that \( T_L^* \) is linear. In the following, we show that \(\|T_L^*\| = L(T)\). For all \( f \in D_L^*, x,y \in D \), by definition, we have
\begin{align*}
	L({T_L^*} f) &= \sup_{x \neq y} \frac{|(T_L^* f)(x) - (T_L^*f)(y)|}{\|x - y\|} = \sup_{x \neq y} \frac{|f(Tx) - f(Ty)|}{\|x - y\|} \\
	&\leq \sup_{x \neq y} \frac{L(f) \cdot \|Tx - Ty\|}{\|x - y\|} = L(f) \cdot L(T),
\end{align*}
Thus, \( \|T_L^*\| \leq L(T) \). Conversely, for all \( f \in X^*, x,y \in D \), let \( f_D \) be the restriction of \( f \) on \( D \), then
	\begin{align*}
L(T) & = \sup_{x \neq y} \frac{\|Tx - Ty\|}{\|x - y\|} = \sup_{x \neq y} \sup_{f \in X^*, \|f\| \leq 1} \frac{|f(Tx - Ty)|}{\|x - y\|} \\
& = \sup_{x \neq y} \sup_{f \in X^*, L(f_D) \leq 1} \frac{|(T_L^* f_D)(x) - (T_L^* f_D)(y)|}{\|x - y\|} \\
& \leq \sup_{x \neq y} \sup_{f \in D_L^*, L(f) \leq 1} \frac{|(T_L^* f)(x) - (T_L^* f)(y)|}{\|x - y\|}  \\
	&\leq \sup_{x \neq y} \sup_{f \in D_L^*, L(f) \leq 1} \frac{\|T_L^* f\| \cdot \|x - y\|}{\|x - y\|} \\
&= \sup_{f \in D_L^*, L(f) \leq 1} \|T_L^* f\| = \|T_L^*\|.
	\end{align*}	
	Thus, \(\|T_L^*\| = L(T)\), i.e., \( T_L^* \) is a bounded linear operator on \( D_L^* \).
	
	Observing that $({T_L^*}^n f) (x) = {T_L^*}^{n-1} (({T_L^*} f) (x)) =  {T_L^*}^{n-1}(f(Tx)) = \cdots = f(T^n(x))$, thus \(\|{T_L^*}^n\| = L(T^n)\). By the definition, we have
	\begin{align*}
		\mathrm{Lip}(T) = \lim\limits_{n \to \infty} [L(T^n)]^{\frac{1}{n}} = \lim\limits_{n \to \infty}\|{T_L^*}^n\|^{\frac{1}{n}} = \rho(T_L^*).
		\end{align*}
		We complete the proof.
\end{proof}

\begin{proposition}
	Let \( T \in \mathscr{L}_0(D) \), then \( T_L^* \) is a completely continuous operator or compact operator.
\end{proposition}
\begin{proof}
	We will show that for any bounded sequence $\{f_n\} \subset D_L^*$, $\{T_L^* f_n\} $ have a convergent subsequence.
	Since $\|f_n\| \leq M_1 $, then there exists a convergent subsequence $\{f_{n_k}\}_{k=1}^\infty$, such that $\lim\limits_{k\to \infty} f_{n_k} = f_0$. That is $\forall \varepsilon >0$, $\exists K_1$, when $k \geq K_1$, we have 
	\[\|f_{n_k} - f_0\| \leq \frac{\varepsilon}{2M_2}.\]
	
	Since $D(T) \subset C$ a bounded closed subset of $\mathbb{R}^n$, there exists a convergent subsequence $\{x_{k}\}_{k=1}^\infty$, such that $\lim\limits_{k \to \infty} x_k = x_0$. Then 
	\[\|Tx_k-Tx_0\| \leq L(T) \|x_k - x\| \to 0,  \quad k\to \infty,\]
which implies $\mathcal{R}(T) \subset D$ is a bounded closed subset of $\mathbb{R}^n$. Since the dimension $n$ is finite, then $\mathcal{R}(T) \cap D$ is a compact set. For any \(\varepsilon > 0\), there exists a finite set \(S = \{s_1, s_2, \dots, s_m\}, s_i \in D\) such that
\[
\mathcal{R}(T) \cap D \subseteq \bigcup_{i=1}^m B(s_i, \frac{\varepsilon}{4M_1}) := \mathcal{B},
\]
where \(B(x_i, \frac{\varepsilon}{4M_1})\) represents an open ball centered at \(s_i\) with radius \(\frac{\varepsilon}{4M_1}\). 
Thus we have $\mathrm{dist}(Tx,\mathcal{B}) \leq \varepsilon, \forall x \in C$, i.e.,
\[
\inf \left\{ \left \|T\left(\frac{x}{\|x\|}\right)-  y' \right\| \large| y' \in \mathcal{M} \right\} \leq \frac{\varepsilon}{4M_1},  \forall x \in C.
\]
In other word, we have 
\[
\inf \left\{ \left \|Tx-  y \right\| \large| y \in \mathcal{M} \right\} \leq \frac{\varepsilon}{4M_1} \|x\|,  \forall x \in C,
\]
where $y = y'\|x\|$. Then $\forall x \in C$, there exists $y \in \mathcal{M}, \|y\| \leq M_2 \|x\|$, such that $\left \|Tx-  y \right\| \leq \frac{\varepsilon}{4M_1}\|x\|$.
$\forall \varepsilon >0, \forall x \in C$, when $k \geq K$, we have
\begin{align*}
	& |(T_L^* f_{n_{k}})(x) - (T_L^* f_{0}) (x)| \\	
	 =& |f_{n_{k}}(Tx) - f_{0}(Tx)| \\
 = 	&|f_{n_{k}}(Tx) - f_{n_{k}}(y) + f_{n_{k}}(y) - f_{0}(y) + f_{0}(y) - f_{0}(Tx)| \\
	 \leq & |f_{n_{k}}(Tx) - f_{n_{k}}(y)| + |f_{n_{k}}(y) - f_{0}(y)| + |f_{0}(y) - f_{0}(Tx)| \\
	 \leq &  2 M \|Tx - y\| + |f_{n_{k}}(y) - f_{n_{0}}(y)|, 	 \\
	 \leq &  2 M_1 \frac{\varepsilon}{4M_1} \|x\| + \|f_{n_{k}} - f_{n_{0}}\|\cdot\|y\| \\
	 \leq & \frac{\varepsilon \|x\|}{2}+ \frac{\varepsilon}{2M_2}  \cdot M_2 \|x\|= \varepsilon \|x\|.
\end{align*}
This implies that $\lim\limits_{k \to \infty} (T_L^* f_{n_{k}} -T_L^* f_0) = 0, \forall x \in C$, i.e., $\{T_L^* f_{n_{k}}\}$ is convergence. And thus $T_L^*$ is a compact operator.

\end{proof}

\subsection{Necessary and Sufficient Conditions for Existence of the Limit Architecture } \label{b3}

\subsubsection{The Fixed Lipschitz Operator with $\mathrm{Lip}(T)\leq 1$} \label{b31}

The following theorem demonstrates that a nonlinear Lipschitz operator $T$ has fixed points when its Lip number satisfies $\mathrm{Lip}(T) \leq 1$.

\begin{theorem} \label{thlip}
	Let \( T \in \mathscr{L}_0(D) \), if $T$ satisfies quasi-asymptotic regularity property. Under the condition of Theorem \ref{th12}, if $\mathrm{Lip}(T) \leq 1$, then $T^nx \to p \in \mathrm{Fix}(T) \neq \varnothing$.
\end{theorem}

\begin{proof}
If $\mathrm{Lip}(T)< 1$, $T$ is a contraction operator, which has a unique fixed point. 

We just consider $\mathrm{Lip}(T)=1$.  Since \( T_L^* \) is a bounded linear operator on \( D_L^* \), then $\sup\limits_{\lambda \in \sigma(T_L^*)} |\lambda| = \rho(T_L^*) = \mathrm{Lip}(T) = 1$, where $\sigma(T_L^*)$ is the spectrum set of linear operator $T_L^*$. In the following, we will show that $\sigma(T_L^*) \cap \mathbb{T} = \{1\}$, where $\mathbb{T} = \{\lambda \in \mathbb{C}: |\lambda| = 1\}$. Since $T$ is quasi-asymptotic regular, we have
\begin{align*}
d(Tx,p) \leq d(x,p)- \psi(d(x,Tx)), \forall x\in D, p \in \mathrm{Fix}(T).
\end{align*}
Let $x_n = T^n x_0$, we have
\begin{align*}
	d(T^{n+1}x_0,p) \leq d(T^nx_0,p)- \psi(d(T^nx_0,T^{n+1}x_0)), \forall x_0\in D, p \in \mathrm{Fix}(T).
\end{align*}
And thus 
\begin{align*}
\psi(d(T^nx_0,T^{n+1}x_0)) \leq d(T^nx_0,p)- d(T^{n+1}x_0,p).
\end{align*}
Taking the sum over $n=0,1,\cdots, N$, yields	
\begin{align*}
\sum_{n=0}^{N}\psi(d(T^nx_0,T^{n+1}x_0)) 
	& \le
	\sum_{n=0}^{N}
	\left(
d(T^nx_0,p)- d(T^{n+1}x_0,p) 
	\right) \\
	& = d(x_0,p)-d(T^{N+1}x_0,p) \leq d(x_0,p).
\end{align*}
Let $N \to \infty$, one has 
\begin{align*}
	\sum_{n=0}^{\infty}\psi(d(T^nx_0,T^{n+1}x_0)) <\infty.
\end{align*}
This says that 
\[
\psi(d(T^nx_0,T^{n+1}x_0)) \to 0, \forall x_0 \in D.
\]
Since $\psi$ is a strictly increasing function, we can obtain
\[
d(T^nx_0,T^{n+1}x_0)\to 0,\forall x_0 \in D
\]
i.e., $T$ is asymptotic regular.
Since $T$ is a non-expansive operator, we have $L(T^n) \leq (L(T))^n \leq 1$, which means that
\begin{align*}
	\sup_{n\neq 1} \|(T_L^n)^n\| \leq 1,
\end{align*}
this says that $\{T_L^n\}_{n\geq 0}$ are bounded.

For any $f\in D_L^*$, let
\[
h_n=(T_L^*)^{n+1}f-(T_L^*)^n f= (T_L^*)^{n}(T_L^*-I) f.
\]
Since
\[
(T_L^*)^{n} f=f\circ T^n,
\]
we have
\[
h_n=f\circ T^{n+1}-f\circ T^n.
\]
For any fixed $x\in D$, it follows that
\[
h_n(x)=f(T^{n+1}x)-f(T^n x).
\]
Since $f$ is Lipschitz continuous, we obtain
\[
|h_n(x)|
=
|f(T^{n+1}x)-f(T^n x)|
\leq L(f)\|T^{n+1}x-T^n x\|.
\]
By the asymptotic regularity of $T$, i.e.,
\[
\|T^{n+1}x-T^n x\|\to 0.
\]
we have
\[
h_n(x)\to 0,\qquad \forall x\in D.
\]
That is,
\(
h_n\to 0
\)
in the pointwise sense.

Next, we show that $\{T_L^*\}_{n\geq 0}$ is relatively compact in $D_L^*$.

Observing that
\[
h_n=(T_L^*)^{n+1}f-(T_L^*)^n f= (T_L^*)^{n}(T_L^*-I) f,
\]
let
\[
g=(T_L^*-I)f,
\]
then
\[
h_n=(T_L^*)^{n} g.
\]
When $n\geq 1$, we have
\[
h_n=T_L^*((T_L^*)^{n-1}g).
\]
Since $\{T_L^n\}_{n\geq 0}$ is bounded, there exists a constant $M>0$ such that
\[
\|(T_L^*)^{n-1}g\|\leq M\|g\|,\qquad n\geq 1.
\]
Therefore, the set
\[
\{(T_L^*)^{n-1}g:n\geq 1\}
\]
is bounded in $D_L^*$.

Moreover, since $T_L^*$ is a compact operator, it maps bounded sets into relatively compact sets. Hence
\(
\{T_L^*((T_L^*)^{n-1}g):n\geq 1\}
\)
is relatively compact in $D_L^*$.
Equivalently,
\(
\{h_n:n\geq 1\}
\)
is relatively compact in $D_L^*$.

We now use relative compactness together with pointwise convergence to derive norm convergence.

Let $\{h_{n_k}\}$ be an arbitrary subsequence of $\{h_n\}$. Since $\{h_n\}$ is relatively compact, there exist a further subsequence $\{h_{n_{k_j}}\}$ and some $h\in D_L^*$ such that
\[
\|h_{n_{k_j}}-h\|_{D_L^*}\to 0.
\]

Since the norm on $D_L^*$ is the Lipschitz norm, namely
\[
\|u\|_{D_L^*}=L(u),
\]
and $u(0)=0$, norm convergence implies pointwise convergence. Indeed, for any $x\in D$,
\[
|h_{n_{k_j}}(x)-h(x)|
\leq L(h_{n_{k_j}}-h)\|x-0\|
=
\|h_{n_{k_j}}-h\|_{D_L^*}\|x\|.
\]
Hence
\[
h_{n_{k_j}}(x)\to h(x).
\]

On the other hand, we have already proved that the whole sequence converges pointwise to $0$. Therefore, we have \(
h_{n_{k_j}}(x)\to 0.
\)
It follows that
\[
h(x)=0,\qquad \forall x\in D.
\]
Thus \(h =0 \).
This shows that every subsequence of $\{h_n\}$ has a further subsequence that converges to $0$ in the norm of $D_L^*$.
Consequently, we must have
\[
\|h_n\|_{D_L^*}\to 0.
\]
Otherwise, if this were not true, then there would exist $\varepsilon>0$ and a subsequence $\{h_{n_k}\}$ such that
\[
\|h_{n_k}\|_{D_L^*}\geq \varepsilon.
\]
However, by relative compactness, this subsequence has a further subsequence converging to $0$ in norm, which gives a contradiction.
Therefore,
\[
\|h_n\|_{D_L^*}\to 0.
\]
That is,
\[
\|(T_L^*)^{n+1}f-(T_L^*)^n f\|_{D_L^*}\to 0.
\]
Since $f\in D_L^*$ is arbitrary, it follows that
\(T_L^*\)
is asymptotically regular.

Now let $\lambda$ be an eigenvalue of $A$ with $|\lambda|=1$. Suppose that
\[
T_L^* f=\lambda f,\qquad f\neq 0.
\]
Then
\[
(T_L^*)^n f=\lambda^n f.
\]
Hence
\[
(T_L^*)^{n+1}f-(T_L^*)^n f
=
\lambda^{n+1}f-\lambda^n f
=
\lambda^n(\lambda-1)f.
\]
Taking norms, we obtain
\[
\|(T_L^*)^{n+1}f-(T_L^*)^n f\|
=
|\lambda|^n|\lambda-1|\|f\|.
\]
Since $|\lambda|=1$, it follows that
\[
\|(T_L^*)^{n+1}f-(T_L^*)^n f\|
=
|\lambda-1|\|f\|.
\]
On the other hand, by the asymptotic regularity of $T_L^*$,
\[
\|(T_L^*)^{n+1}f-(T_L^*)^n f\|\to 0.
\]
Therefore,
\[
|\lambda-1|\|f\|=0.
\]
Since $f\neq 0$, we must have \(\lambda=1\). Therefore, we have
$\sigma(T_L^*) \cap \mathbb{T} = \{1\}$, where $\mathbb{T} = \{\lambda \in \mathbb{C}: |\lambda| = 1\}$.

%Thus, we obtain the following conclusion:
%\[
%\boxed{
%	\text{If } A \text{ is asymptotically regular, then the only eigenvalue of } A
%	\text{ on the unit circle is } 1.
%}
%\]
%
%If, in addition, $A$ is compact, then every nonzero spectral value of $A$
%is an eigenvalue. Hence, if
%\[
%\lambda\in\sigma(A),\qquad |\lambda|=1,
%\]
%then $\lambda\neq 0$, and thus $\lambda$ is an eigenvalue of $A$. By the
%argument above, we obtain
%\[
%\lambda=1.
%\]
%Consequently,
%\[
%\sigma(A)\cap \mathbb T\subset \{1\},
%\]
%where
%\[
%\mathbb T=\{\lambda\in\mathbb C:|\lambda|=1\}.
%\]
%

%Observing that $\forall f \in D_L^*, x \in D$, we have
%\begin{align*}
%	|((T_L^*)^{n+1}f)(x) - ((T_L^*)^{n}f)(x) |  = |f(T^{n+1}x) - f(T^n(x))|   =	 L(f) \|T^{n+1}x - T^{n}x\| \to 0.
%\end{align*} 
%This yields $(T_L^*)^{n}f \to 0, \forall x \in D$. 
%
%
%Based on this, we have
%\begin{align*}
%	|(T_L^n)^{n+1} f - (T_L^n)^{n} f| = (T_L^n)^{n} (T_L^n-I) f 
%\end{align*}

%If $T$ is a project operator, i.e., $T^2=T$, then one has
%\begin{align*}
%	((T_L^*)^2 f)(x) & = 	T_L^*(T_L^*f)(x) \\
%	& =  T_L^*f(Tx)  \\ 
%		& = f(T(Tx)) \\
%		& = f(T^2x) \\
%		& = f(Tx) \\
%		& = T^*_L f(x),\quad \forall f \in D_L^*, \, x \in D.
%\end{align*}
%This yields
%\begin{align*}
%(T_L^*)^2 = T^*_L,
%\end{align*}
%which shows that $\sigma(T_L^*) \cap \mathbb{T} = \{1\}$, where $\mathbb{T} = \{\lambda \in \mathbb{C}: |\lambda| = 1\}$.

Let $\sigma = {1}, \bar{\sigma} = \sigma(T_L^*) - \sigma$, then we have the following decomposition by the Lemma \ref{lemma5}
\[
D_L^* = \mathcal{M}_\sigma + \mathcal{M}_{\bar{\sigma}},
\]
where $\mathcal{M}_\sigma = \{f \in D_L^* | P_{\sigma} f = f\}, \mathcal{M}_{\bar{\sigma}} = \{f \in D_L^* | P_{\sigma} f = 0\}$, and $P_{\sigma}$ is the orthogonal projection onto the proper subspace corresponding to the part $\sigma$ of the spectrum defined in Lemma \ref{lemma5}. Meanwhile, the study of  \( T_L^* \) could be decomposed into 
\[
T_L^* = (T_L^*)_1 + (T_L^*)_2,  
\]
where
\[
(T_L^*)_1 = T_L^*|_{\mathcal{M}_\sigma},  (T_L^*)_2 = T_L^*|_{\mathcal{M}_{\bar{\sigma}}}.
\]
Then, $\forall f \in D_L^*$, we have 
\begin{align*}
f &= f_1 + f_2 = P_{\sigma}f +(I-P_{\sigma}) f, \\
 T_L^*f &= T_L^*f_1 + T_L^*f_2 = (T_L^*)_1f+(T_L^*)_2f.
\end{align*}
Furthermore, one has 
\[
(T_L^*)^2f = T_L^*(T_L^*f) = T_L^*(T_L^*f_1 + T_L^*f_2) = (T_L^*)^2f_1 + (T_L^*)^2f_2 =  [(T_L^*)^2]_1f+[(T_L^*)^2]_2f.
\]
Similarly one has
\[
(T_L^*)^nf =  [(T_L^*)^n]_1f+[(T_L^*)^n]_2f.
\]
Since $\rho((T_L^*)_1) =1, \rho((T_L^*)_2) <1$, we have
\[
(T_L^*)^nf \to (T_L^*)_1f + 0 = P_{\sigma}f,  n \to \infty.
\]
Recall that $ f(T^n) = (T_L^*)^nf \to P_{\sigma}f, n \to \infty$, which means $f(T^n)$ converges as $n \to \infty$.

On the other hand, $f$ is Lipschitz, then there exist some constant $\mathbb{K}$ and $n_1\in \mathbb{N}^+$, such that
\begin{align} \label{eqxx}
|f(T^{n_1+1}(x)) - f(T^{n_1}(x))| \geq \mathbb{K} \|T^{n_1+1}(x) - T^{n_1}(x)\|.
\end{align}
In the following, we will show that $\{T^n x\}$ is a Cauchy sequence for each $x \in D$. Indeed, let $k,l \in \mathbb{N}$ with $k > l$. Then $\forall \varepsilon>0$, $\exists \mathbb{N}$, such that when $k, l>\mathbb{N}$, we have 
\[
\| T^kx - T^lx \|\leq  \frac{1}{\mathbb{K}}|f(T^{k}(x)) - f(T^{l}(x))| < \frac{\varepsilon}{\mathbb{K}}, \forall x \in D.
\]
Therefore, we have that $\{T^n x\}$ is a Cauchy sequence and hence it converges strongly to some point of $D$. 
Let $x_n = T^{n-1}x, n \geq 1$, then $x_{n+1} = Tx_n, x_1 \in D$. Since $\{x_n\}$ is bounded, we assume that $x_n \rightharpoonup x$.
Since $T^{n+1}x - T^n x \to 0$, we have $Tx_n - x_n \to 0$. By Theorem \ref{th12}, we have $Tx - x \to 0$, i.e., $x \in \mathrm{Fix}(T)$.
\end{proof}

Based on the above theorem, we can obtain the following necessary and sufficient conditions for the existence of limit architecture.
\begin{theorem} \label{th11}
	Suppose that the basic block $T_i=T, i \in [K]$ of foundation models satisfies self-mapping condition and quasi-asymptotic regularity property. Then the necessary and sufficient conditions for the existence of limit architecture $f^*_{\mathcal{W}}$ is $\mathrm{Lip}(T) \leq 1$.
	\end{theorem}
\begin{remark}
	If the limit architecture $f^*_{\mathcal{W}}$ of foundation models exists, it satisfies
	  	\[	f^*_{\mathcal{W}} = \lim\limits_{N \to \infty} T^N f_0,\]
	and the convergere result depends on the initial architecture $f_0$. This means different initial architectures tend to yield different properties of different limit architectures.
\end{remark}

\begin{proof}
	\textbf{Sufficiency.} It holds by Theorem \ref{thlip}.

	\textbf{Necessity.} Proof by contradiction. We assume $\mathrm{Lip}(T) =C>1$, then we have 
	\[[L(T^N)]^{\frac{1}{N}} \geq C,\]
	i.e., \[[L(T^N)] \geq C^N  \to \infty, \ \text{as} \ N \to \infty. \]
This makes a contradiction. We complete the proof.
	
\end{proof}

\subsubsection{Finite Lipschitz Operators with $\mathrm{Lip}(T_i) \leq 1, i=[r]$}

%\begin{theorem} \label{th112}
%	Suppose that the basic block $\{T_i\}$ of the foundation models satisfies self-mapping condition, and are chosen from finite lipschitz operators $\{T_1, \cdots, T_C \}$. Then the necessary and sufficient conditions for the existence of limit architecture $f^*_{\mathcal{W}}$ is $\mathrm{Lip}(T_i) \leq 1, i \in [C]$.
%\end{theorem}
%
%\begin{theorem}
%	Suppose that basic functional blocks $T_i=T, i=1,2,\cdots$ of the foundation models are chosen from $S_i, i \in [r]$ satisfying self-mapping condition, where $r$ is finite constant, and 	\(
%	\bigcap_{i=1}^r \mathrm{Fix}(S_i) \neq \emptyset.
%	\)
%	Then the necessary and sufficient conditions for the existence of limit architecture $f^*$ in Eq.(\ref{eqlimit}) is $\mathrm{Lip}(S_i) \leq 1, i \in [r]$.
%\end{theorem}

The Theorem \ref{th1b} could be rewritten as the following expression:

\begin{theorem}
	%Let $E$ be a uniformly convex Banach space which satisfies Opial's condition
	%or whose norm is Fr\'echet differentiable. 
	%Let $D$ be a nonempty closed convex
	%subset of $E$, and 
	
		Suppose that the basic functional blocks $T_i, i = 1,2,\ldots$ of a foundation model are selected from a finite collection of operators $\{S_i\}_{i=1}^r$, and \(
	\bigcap_{i=1}^r \mathrm{Fix}(S_i) \neq \varnothing
	\). Assume that each $S_i$ satisfies the self-mapping condition and quasi-asymptotic regularity property.
	Then the necessary and sufficient condition for existence of the limit architecture in Eq.~(\ref{eqlimit}) is $\mathrm{Lip}(S_i) \leq 1$ for all $i \in [r]$.

%	Let $\{T_1,T_2,\ldots,T_r\}$ be the basic functional blocks of the foundation models satisfying self-mapping condition, such that
%	\[
%\bigcap_{i=1}^r \mathrm{Fix}(T_i) \neq \emptyset.
%	\]
%	And each $T_i$ satisfies $\gamma$-strongly quasi-nonexpansive property, i.e., there exists $\gamma_i>0$, such that \(
%	\|T_ix-p\|^2\le \|x-p\|^2-\gamma_i\|x-T_ix\|^2, \forall x \in D, p \in \bigcap_{i=1}^r \mathrm{Fix}(T_i), i =1, 2, \cdots, r.\)
%
%%	Let $W_n (n=1,2,\cdots)$ be $W$-mappings generated by $T_1,T_2,\ldots,T_r$
%%	and $\alpha_{1_1},\alpha_{1_2},\cdots,\alpha_{1r}, \cdots, \alpha_{n_1},\alpha_{n_2},\cdots,\alpha_{nr},\cdots \in \{0,1\}$,
%%	\begin{equation}\label{eq:W-mapping}
%%		\begin{aligned}
%%			W_{1} &= T_1^{\alpha_{1_1}} T_2^{\alpha_{1_2}} \cdots T_r^{\alpha_{1_r}},\\
%%			W_{2} &= T_1^{\alpha_{2_1}} T_2^{\alpha_{2_2}} \cdots T_r^{\alpha_{2_r}},\\
%%			&\ \vdots \\
%%			W_n &= T_1^{\alpha_{n_1}} T_2^{\alpha_{n_2}} \cdots T_r^{\alpha_{n_r}}.
%%		\end{aligned}
%%	\end{equation}
%	Suppose $f_0\in D$, and $\{f_n\}$ is given by
%	\[
%	f_{n+1}=Uf_n,\quad n\ge 1.
%	\]
%	Then the necessary and sufficient conditions for the existence of limit architecture $f^*= \lim\limits_{n\to \infty} f_n$ is $\mathrm{Lip}(T_i) \leq 1, i \in [r]$.

\end{theorem}

\begin{proof}
	\textbf{Sufficiency. }	
		Let $U$ be generated by $S_1,S_2,\ldots,S_r$ and $\alpha_{1},\alpha_{2},\cdots,\alpha_{r} \in \{0,1\}$, 
	\begin{align*}
		U= S_1^{\alpha_{1}} S_2^{\alpha_{2}} \cdots S_r^{\alpha_{r}}.
	\end{align*}	
	We let $S_k^0 = I, k=1,2,\cdots, r$.
		Suppose $f_0\in D$, and $\{f_n\}$ is given by
	\[
	f_{n+1}=Uf_n,\quad n\ge 1.
	\]	
	Let $f_0\in D$ and $p\in\bigcap_{i=1}^r \mathrm{Fix}(S_i)$. By the definition of $U$,
	we have $U\in \mathscr{L}_0(D)$, and
	\begin{align*}
	\mathrm{Lip}(U) \leq 1.
	\end{align*}
	Since each $T_i$ satisfies quasi-asymptotic regularity property, then there exists an increasing function
	\(
	\psi:\mathbb{R}_+\to\mathbb{R}_+,
	\psi(0)=0,
\)
	such that, for every \(i=1,\ldots,r\),
	\begin{equation}
		d(x,p)
		\leq
		d(T_i x,p)
		-
		\psi\bigl(d(T_i x,x)\bigr),
	\forall	x\in D,\quad p\in C .
		\label{eq:single-expansive1}
	\end{equation}
		Fix \(x\in D\) and \(p\in C\), and define
	\[
	x_0=x,
	x_k=S_k^{\alpha_k}x_{k-1},
	\quad k=1,\ldots,r,
	\]
	then \(	x_r=Ux.\)
	Since \(p\in\bigcap_{i=1}^r \mathrm{Fix}(S_i)\) , applying \eqref{eq:single-expansive1} to \(S_k^{\alpha_k}\) and \(x_{k-1}\) gives
	\[
	d(x_{k-1},p)
	\leq
	d(S_k^{\alpha_k}x_{k-1},p)
	-
	\psi\bigl(d(S_k^{\alpha_k}x_{k-1},x_{k-1})\bigr).
	\]
	Since \(S_kx_{k-1}=x_k\), this becomes
	\begin{equation}
		d(x_{k-1},p)
		\leq
		d(x_k,p)
		-
		\psi\bigl(d(x_k,x_{k-1})\bigr).
		\label{eq:step}
	\end{equation}
	Equivalently,
	\[
	d(x_k,p)-d(x_{k-1},p)
	\geq
	\psi\bigl(d(x_k,x_{k-1})\bigr).
	\]
		Summing this inequality over \(k=1,\ldots,r\), we obtain
	\begin{align}
		d(x_r,p)-d(x_0,p)
		&=
		\sum_{k=1}^{r}
		\bigl(d(x_k,p)-d(x_{k-1},p)\bigr)
		\nonumber\\
		&\geq
		\sum_{k=1}^{r}
		\psi\bigl(d(x_k,x_{k-1})\bigr).
		\label{eq:sum}
	\end{align}
	Hence,
	\begin{equation}
		d(x_0,p)
		\leq
		d(x_r,p)
		-
		\sum_{k=1}^{r}
		\psi\bigl(d(x_k,x_{k-1})\bigr).
		\label{eq:pre-final}
	\end{equation}
			It remains to bound the summation term from below in terms of \(d(x_r,x_0)\). By the triangle inequality,
	\[
	d(x_r,x_0)
	\leq
	\sum_{k=1}^{r} d(x_k,x_{k-1}).
	\]
	Therefore, there exists at least one index \(k_0\in\{1,\ldots,r\}\) such that
	\[
	d(x_{k_0},x_{k_0-1})
	\geq
	\frac{1}{r}d(x_r,x_0).
	\]
	Since \(\psi\) is increasing and nonnegative, we have
	\[
	\sum_{k=1}^{r}
	\psi\bigl(d(x_k,x_{k-1})\bigr)
	\geq
	\psi\bigl(d(x_{k_0},x_{k_0-1})\bigr)
	\geq
	\psi\left(\frac{1}{r}d(x_r,x_0)\right).
	\]
	Substituting this into \eqref{eq:pre-final} yields
	\[
	d(x_0,p)
	\leq
	d(x_r,p)
	-
	\psi\left(\frac{1}{r}d(x_r,x_0)\right).
	\]
	Using \(x_0=x\) and \(x_r=Ux\), we obtain
	\[
	d(x,p)
	\leq
	d(Ux,p)
	-
	\psi\left(\frac{1}{r}d(Ux,x)\right).
	\]
	Thus, by defining
	\[
	\widetilde{\psi}(t)
	=
	\psi\left(\frac{t}{r}\right),
	\]
	we obtain
	\[
	d(x,p)
	\leq
	d(Ux,p)
	-
	\widetilde{\psi}\bigl(d(Ux,x)\bigr).
	\]

	By Theorem \ref{thlip}, we have that $\{f_{n}\}$ converges strongly to some fixed point of $U$ on $D$. In the following, we need to prove that $\mathrm{Fix}(U) = \bigcap_{i=1}^r \mathrm{Fix}(T_i)$. 
	
	(1) $\bigcap_{i=1}^r \mathrm{Fix}(T_i) \subset  \mathrm{Fix}(U) $. If $p \in \bigcap_{i=1}^r \mathrm{Fix}(T_i)$, then we have $T_i p = p, i=1,2, \cdots, r$. By the definition of $U$, we have $Up = p$, i.e., $p \in \mathrm{Fix}(U)$. 
	
	(2) $\mathrm{Fix}(U) \subset \bigcap_{i=1}^r \mathrm{Fix}(T_i)$. 
	
	Let $z\in \mathrm{Fix}(U)$, and take $p \in \bigcap_{i=1}^r \mathrm{Fix}(T_i) $, then $\forall \alpha \in [0,1]$, we have
	\begin{align*}
		\|z-p \| & = \|\alpha Uz+(1-\alpha)z- \alpha T_1^{\alpha_{1}} T_2^{\alpha_{2}} \cdots T_r^{\alpha_{r}} z - (1-\alpha)p \| \\
		& = \|\alpha T_1^{\alpha_{1}} T_2^{\alpha_{2}} \cdots T_r^{\alpha_{r}} (z-p) + (1-\alpha)(z-p)\|  \\
%		& \leq \| T_1^{\alpha_{1}} T_2^{\alpha_{2}} \cdots T_{r-1}^{\alpha_{r-1}} \|  \cdot \|z-p\| [\alpha (\mathrm{Lip}(T_r^{\alpha_{r}})+\varepsilon) + (1-\alpha)]  \\
%			& \leq \| T_1^{\alpha_{1}} T_2^{\alpha_{2}} \cdots T_{r-1}^{\alpha_{r-1}} \|  \cdot \|z-p\| (1+\alpha\varepsilon)  \\
%				& \leq \| T_1^{\alpha_{1}} T_2^{\alpha_{2}} \cdots T_{r-1}^{\alpha_{r-1}} \|  \cdot \|z-p\| (1+\alpha\varepsilon) \\
	& \leq   \alpha(1+\varepsilon)^{\alpha_r } \|T_1^{\alpha_{1}} T_2^{\alpha_{2}} \cdots T_{r-1}^{\alpha_{r-1}} z-p\| +(1-\alpha) \|z-p\| \\
	& \leq   \alpha(1+\varepsilon)^{\alpha_{r-1} + \alpha_r } \|T_1^{\alpha_{1}} T_2^{\alpha_{2}} \cdots T_{r-2}^{\alpha_{r-2}} z-p\| +(1-\alpha) \|z-p\| \\
	& \leq \cdots \\
		& \leq   \alpha(1+\varepsilon)^{\sum_{i=2}^r \alpha_r } \|T_1^{\alpha_{1}}z-p\| +(1-\alpha) \|z-p\| \\
					& \leq   \alpha(1+\varepsilon)^{\sum_{i=1}^r \alpha_r } \|z-p\| +(1-\alpha) \|z-p\| \\
					& =  \|z-p\|,	\end{align*}
	where $L(T_i^{\alpha_{i}}) = \mathrm{Lip}(T_i^{\alpha_{i}})+\varepsilon \leq 1+ \varepsilon, \varepsilon \geq 0$ by the definition of the Lip number. Since $E$ is strictly convex, we have $T_1^{\alpha_{1}}z -p = z-p$, and thus $T_1z  = z$. By such method, we have $T_iz  = z, i=2,\cdots, r$. This implies $z \in \bigcap_{i=1}^r \mathrm{Fix}(T_i)$. Thus, we have $\mathrm{Fix}(U) \subset \bigcap_{i=1}^r \mathrm{Fix}(T_i)$.
	
	Therefore, we have $\mathrm{Fix}(U) = \bigcap_{i=1}^r \mathrm{Fix}(T_i)$. This implies that $\{f_{n}\}$ converges strongly to common fixed point of $T_1,T_2, \cdots, T_r$.
	
		\textbf{Necessity.} Proof by contradiction. We assume $\mathrm{Lip}(T_i) =C>1, i=1,2,\cdots, r$, then we have 
	\[[L({T_i}^N)]^{\frac{1}{N}} \geq C,\]
	i.e., \[[L({T_i}^N)] \geq C^N  \to \infty, \ \text{as} \ N \to \infty. \]
	Observe the case $f_{n+1}=Uf_n, U = T_i$ when $n>K$ for some constant $K>0$, which implies that \([L({T_i}^N)]\) convergence. 
	This makes a contradiction. We complete the proof.

	\end{proof}

\subsubsection{Countable Lipschitz Operators with $\mathrm{Lip}(T_i) \leq 1$}

	We say $T$ is a generalized projection operator, if it satisfies the following properties: (\romannumeral1) $T|_{\mathrm{Fix}(T)} = I$; (\romannumeral2) $T$ is quasi-asymptotic regular; (\romannumeral3) $\mathrm{Lip}(T) \leq 1$.

\begin{theorem} \label{operator}
	Suppose that the basic functional blocks $T_i$ ($i = 1,2,\ldots$) of a foundation model satisfy the self-mapping condition and quasi-asymptotic regularity property. Then the necessary and sufficient condition for the existence of the limit architecture $f^*$ in Eq.~(\ref{eqlimit}) are that:
	
	(\romannumeral1) $\mathrm{Lip}(T_i) \leq 1$ for all $i \geq K_0$, for some constant $K_0$;
	
	(\romannumeral2) there exists a generalized projection operator $T$ such that $\|T_i - T\|  \leq \epsilon_i$, and $\sum_{i=1}^\infty \epsilon_i < \infty$.

	%	Suppose that basic functional blocks $T_i, i=1,2,\cdots$ of the foundation models satisfies self-mapping condition, then the necessary and sufficient conditions for the existence of limit architecture $f^*$ in Eq.(\ref{eqlimit}) are (1) $\mathrm{Lip}(T_i)\leq 1, i >K $ for some constant $K$; (2) there exist a fixed operator $T$ with $\mathrm{Lip}(T)\leq 1$, such that $\|T_i-T\|= \epsilon_i$, where $\sum_{i=1}^\infty \epsilon_i < \infty$. 
\end{theorem}

\begin{proof}
	\textbf{Necessity.} Since the limit architecture exists, we have $0 = \prod_{i=1}^{\infty}T_i - \lim\limits_{N \to \infty} (\prod_{i=1}^{N}T_i)$, i.e., 
	\[0 = \lim\limits_{N \to \infty} \left( \prod_{i={N+1}}^{\infty}T_i -I \right)\prod_{i={1}}^{N}T_i, \]
	which means that
	\begin{align} \label{proj}
	\lim\limits_{k\to \infty} T_{k} = \lim\limits_{k\to \infty }  I_{|\mathbb{R}_k}, 
	\end{align}
	where $\mathbb{R}_k = \cap_{i=1}^k \mathbb{R}(T_i)$.
	In other words, $T_{k+1}A_k - A_k \to 0, k \to \infty $, where $A_k = \Pi_{i=1}^k T_i$. We claim that there exists a constant $K>0$, such that $\mathrm{Lip}(T_i) \leq 1, i\geq K$. Otherwise, if $\mathrm{Lip}(T_i) > 1 $, we would make a contradiction.

	On the other hand, we claim that there exists a common generalized projection operator $T$, such that $\|T_i-T\|\leq \epsilon_i$ and $\sum_{i=1}^\infty \epsilon_i < \infty$. Since $\{A_{k} \}$ converges, then $\lim\limits_{k\to \infty} T_k$ exists, and we denote it by $T = \lim\limits_{k\to \infty} T_k$, and $T$ satisfies $\mathrm{Fix}(T) \leq 1$ and quasi-asymptotic regularity property. 	
	 In the following, we show that $\mathrm{Fix}(T) \neq \varnothing$.
For any \(x\in D\), let	\(p=\lim_{k\to \infty} A_kx\), then
	\[
	\begin{aligned}
		\|Tp-p\|
		&\leq \|Tp-T_{k+1}p\|
		+\|T_{k+1}p-T_{k+1}A_kx\|  \\
		&\quad +\|T_{k+1}A_kx-A_kx\|
		+\|A_kx-p\|.
	\end{aligned}
	\]
	Since \(T_{k+1}\) is nonexpansive, we have
	\[
	\|T_{k+1}p-T_{k+1}A_kx\|
	\leq
	\|p-A_kx\|.
	\]
	Therefore,
	\[
	\|Tp-p\|
	\leq
	\|Tp-T_{k+1}p\|
	+2\|p-A_kx\|
	+\|T_{k+1}A_kx-A_kx\|.
	\]
	By the assumptions,
	\[
	T_{k+1}p\to Tp,
		A_kx\to p,
	\]
	and
	\[
	T_{k+1}A_kx-A_kx\to 0.
	\]
	Hence,
	\[
	\|Tp-p\|=0,
	\]
	which implies \(Tp=p\), that is $p\in \operatorname{Fix}(T)$. That is $Tp = \lim\limits_{k\to \infty} T_k p = p$, which means $T|_{\mathrm{Fix}(T)} = I$.

	Thus for any $\epsilon>0$, these exists an integer  $N>0$, such that for all $k>N$, we have 
	\[
	\|T_k-T\| < \epsilon.
	\]
	Let $\epsilon_k = \|T_k-T\|$, then we want to claim that $\sum_{i=1}^\infty \epsilon_i < \infty$.
	
	Suppose	$\sum_{i=1}^\infty \epsilon_i = \infty$. For $\forall k \geq 0$ and $\forall f_0 \in D$, we have
	\begin{align*}
		\|A_kf_0 - T^kf_0\| & = \sup_{F\in D_L^*, L(F)\leq 1} |F(A_kf_0) - F(T^kf_0)| \\
		& = \sup_{F\in D_L^*, L(F)\leq 1} \|((A_k)^*_L F)(f_0) - ((T^k)^*_LF)(f_0)\|,  
	\end{align*}
	where \( T_L^*: D_L^* \to D_L^* \) be the Lipschitz dual operator of \( T \), and satisfies that 
	\((T_L^* F)(x) = F(Tx), \forall F \in D_L^*, \, x \in D\).
	For any $ F \in D_L^*, L(F) \leq 1, f_0 \in D$ and $K>0$, since
	\begin{align*}
		& \|((A_K)^*_L F)(f_0) - ((T^K)^*_LF)(f_0)\| \\
		= &\|({T_1}_L^*{T_2}_L^* \cdots {T_K}_L^* F)(f_0) - ((T_L^*)^K F)(f_0)\|  \\
		=& \| [(T^*+E_1)T_2^* \cdots T_K^* F - (T^*)^K F](f_0) \| \\
		=& \| [T^*T_2^* \cdots T_K^* F + E_1T_2^* \cdots T_K^* F- (T^*)^K F](f_0) \| \\
		=& \| [T^*T_2^* \cdots T_K^* F + E_1T_2^* \cdots T_K^* F- (T^*)^K F](f_0) \| \\ 
		=& \| [T^*(T^*+E_2) \cdots T_K^* F + E_1T_2^* \cdots T_K^* F- (T^*)^K F](f_0) \| \\
		=& \| [(T^*)^2 \cdots T_K^* F + E_1T_2^* \cdots T_K^* F + T^*E_2 T_3^* \cdots T_K^* F - (T^*)^K F](f_0) \| \\
		=& \cdots \\
		=&  \|[ (T^*)^K F +E_1T_2^* \cdots T_K^* F + T^*E_2 T_3^* \cdots T_K^* F + \cdots + (T^*)^{K-1}E_KF - (T^*)^K F](f_0) \| \\ 
		= & \|[ E_1T_2^* \cdots T_K^* F + T^*E_2 T_3^* \cdots T_K^* F + \cdots + (T^*)^{K-1}E_KF](f_0)  \| \\
		\leq &   \sum_{i=1}^K \|E_i\| \prod_{{i<j \leq K}} L(T_j) \|f_0\| \leq M  \sum_{i=1}^K \epsilon_i 
	\end{align*}
	where $M = \sup_{i \in [K]}\prod_{i<j \leq K} L(T_j) \|f_0\|$. 
	Since the sequences of $\{A_k\}, \{T^k\}$ converge, we have 
	$
	\lim\limits_{k\to \infty} \|A_k - T^k\| 
	$ is bounded. While $\lim\limits_{k\to \infty}\sum_{i=1}^K \epsilon_i = \infty$. This leads to a contradiction

\textbf{	Sufficiency.  	}
	%	Let \( T_L^*: D_L^* \to D_L^* \) be the Lipschitz dual operator of \( T \), and satisfies that 
	%\((T_L^* f)(x) = f(Tx), \forall f \in D_L^*, \, x \in D\).
	Given a common generalized projection operator $T \in \mathscr{L}_0(D)$, we know that $\mathrm{Fix}(T) = \varnothing$. We assume that $p \in \mathrm{Fix}(T)$.
	%$\mathrm{Lip}(T)\leq1$, $T$ has a fixed point $p$ by Theorem \ref{thlip}.
	%, i.e., $T^nx \to p\in F(T), n \to \infty$ 
	We will show $\prod_{i=1}^{\infty}T_i$ exists. 
	
	Since $p \in \text{Fix}(T)$, then we have for any $k>0$:
	\begin{align} \label{all1}
		\|A_{k} - p\| & \leq \|T_kA_{k-1}-TA_{k-1}\|+\|TA_{k-1}-p\|.
	\end{align}
	For the first term, observing that
	\begin{align*}
		\|{T_i}_L^* - T_L^* \| & = \sup_{F \in D_L^*, L(F)\leq 1} \|({T_i}_L^* - T_L^*)F\|  = \sup_{F \in D_L^*, L(F)\leq 1} \|{T_i}_L^*F - T_L^*F\| \\
		&  = \sup_{F \in D_L^*, L(F)\leq 1} \|F \circ T_i - F\circ T \|  \leq \| T_i -  T \| \leq \epsilon_i ,
	\end{align*}
	hence we have (first inequallity holds by Eq.(\ref{eqxx}))
	\begin{align*}
		\|TA_{k-1} - T_kA_{k-1}\| & \leq \frac{1}{\mathbb{K}}\| F(TA_{k-1}) - F(T_kA_{k-1}) \| = \frac{1}{\mathbb{K}} \|T_L^* F(A_{k-1}) - {T_k}_L^* F(A_{k-1})\| \\
		& = \frac{1}{\mathbb{K}} \|T_L^* F(A_{k-1}) - (T_L^*+E_k) F(A_{k-1})\| \\
		&= \frac{1}{\mathbb{K}} \|E_k F\circ A_{k-1}\| \leq \frac{\epsilon_k L R}{\mathbb{K}},
	\end{align*} 
	where $R = \sup_{A_{k-1}\in \mathcal{L}_0(D)}\|A_{k-1}\|, E_k = {T_k}_L^* - T^*_L$ and $\|E_k\| \leq \epsilon_k$.
	
%	There exists a strongly equivalent functional $\|\cdot\|_s$ of the norm $\|\cdot\|$,  such that $\mathrm{Lip}(T)$ could be approximated arbitrarily closely by $L(T)$ with respect to strong equivalence of distances. In other word, $\forall \varepsilon>0$, there exists a distance $\|\cdot\|_s$ that is strongly equivalent to $\|\cdot\|$, such that the corresponding $L(T)$ denoted as $L^*(T)$, satisfing
%	\[
%	\text{Lip}(T) \leq L^*(T) \leq \text{Lip}(T) + \varepsilon.
%	\]
%	Since $\mathrm{Lip}(T)\leq1$, we have 
%	\begin{align*}
%		\|Tx_k-p\|_s  & \leq  L^*(T) \|x_k-p\|_s \leq  (\text{Lip}(T) + \varepsilon) \|x_k-p\|_s\\
%		& \leq (1 + \varepsilon) \|x_k-p\|_s.
%	\end{align*}
	Since $T$ is non-expansive, we have \[\|Tx_k-p\| \leq \|x_k-p\|.\]
	Combing the above inequality, we can obtain the following estimation for Eq.(\ref{all1}):
	\begin{align*} 
		\|A_{k} - p\| & \leq  \|A_{k-1}-p\|  + \frac{\epsilon_k L R}{\mathbb{K}}.
	\end{align*}
	Since $\sum_{i=1}^\infty \epsilon_k < \infty$, by Lemma \ref{lm1}, we have $\lim\limits_{k\to \infty} \|A_{k} - p\|$ exists. This means that $\lim\limits_{k\to \infty} A_k$ exists, i.e., $\prod_{i=1}^{\infty}T_i$ exists.

%	The following illustrates that $A_k$ converges to the fixed point of $T$.

%		Observing that $\forall f \in D_L^*, x \in D$, we have
%	\begin{align*}
%		\|Tx - T_kx\| & \leq \frac{1}{K}| f(Tx) - f(T_kx) | = \frac{1}{K} |(T_L^* f)(x) - ({T_k}_L^*) f(x)| \\
%		& = \frac{1}{K} |(T_L^* f) f(x) - (T_L^*f+E_kf) (x)| = \frac{1}{K} |(E_k f)(x)| \leq \frac{\varepsilon_k L}{K},
%	\end{align*} 
%	and
%	\begin{align*}
%		\|Tx_{k} - T_kx_{k}\| & \leq \frac{1}{K}| f(Tx_{k}) - f(T_kx_{k}) | = \frac{1}{K} |T_L^* f(x_k) - {T_k}_L^* f(x_k)| \\
%		& = \frac{1}{K} |T_L^* f(x_k) - (T_L^*+E_k) f(x_k)| = \frac{1}{K} |E_k f(x)| \leq \frac{\varepsilon_k L R}{K}
%	\end{align*} 
%	where $R = \sup_{x\in D}\|x\|, E_k = {T_k}_L^* - T^*_L$ and $\|E_k\| = \varepsilon_k$.
	
	Next, we will show ${A_k}$ converges weakly to $z \in \text{Fix}(T)$.
	Observing that
	\begin{align*}
		\|TA_{k} - A_{k} \| & \leq \|TA_{k} - T_kA_{k}\| +\|T_kA_{k} - A_{k} \| \\
		& \leq \frac{\varepsilon_k LR}{\mathbb{K}} + \|T_kA_{k} - A_{k} \|,
	\end{align*}
	hence 
	\[
	\lim\limits_{k\to \infty }\|TA_{k} - A_{k} \| = 0.
	\]

	Suppose that $E$ satisfies the Opial's condition. Since $\{A_k\}$ is bounded and closed, we can choose a subsequence $\{A_{k_i}\}$ of $\{A_k\}$ converging weakly to $z$. We shall show that $z$ is a fixed point of $T$. Suppose $z \notin \text{Fix}(T)$. Then $z \neq T z$. Using the Opial's condition, we have
	\begin{align*}	
		\liminf_{i \to \infty} \|A_{k_i} - z\| & < \liminf_{i \to \infty} \|A_{k_i} - T z\| \\
		& \leq \liminf_{i \to \infty} (\|A_{k_i} - T A_{k_i}\| + \|T A_{k_i} - T z\|) \\
		& \leq \liminf_{i \to \infty} (\|A_{k_i} - T A_{k_i}\| + \|A_{k_i} - z\|) \\
		& = \liminf_{i \to \infty} \|A_{k_i} - z\|.
	\end{align*}
	This is a contradiction and hence we have $z \in \text{Fix}(T)$.
	
	Next, we shall show that the set of weakly cluster points of $\{A_k\}$ consists of one point. Let $\{A_{k_i}\}$ and $\{A_{k_j}\}$ be subsequences of $\{A_k\}$ converging weakly $z_1$ and $z_2$, respectively. Then both $z_1$ and $z_2$ are fixed points of $T$ and thus $\lim_{k \to \infty} \|A_k - z_1\|$ and $\lim_{k \to \infty} \|A_k - z_2\|$ exist. We claim $z_1 = z_2$. If not, we have
	\begin{align*}	
		\lim_{k \to \infty} \|A_k - z_1\| & = \lim_{i \to \infty} \|A_{k_i} - z_1\| \\
		& < \lim_{i \to \infty} \|A_{n_i} - z_2\|  \\
		& = \lim_{n \to \infty} \|A_n - z_2\| \\
		& = \lim_{j \to \infty} \|A_{n_j} - z_2\| \\
		& < \lim_{j \to \infty} \|A_{n_j} - z_1\| \\
		& = \lim_{n \to \infty} \|A_n - z_1\|.
	\end{align*}	
	This is a contradiction and we get $z_1 = z_2$. Therefore $\{A_k\}$ converges weakly to $z \in \text{Fix}(T)$.
	
	On the other hand, suppose that $D$ has a Fr\'{e}chet differentiable norm. Let $\{A_{k_i}\}$ be a subsequence of $\{A_k\}$ coverging weakly to $z$. Then since \(\lim\limits_{k\to \infty }\|TA_{k_i} - A_{k_i} \| = 0\), by the demiclosedness principle for asymptotically nonexpansive mapping, we have \(\lim\limits_{k\to \infty }\|Tz - z \| = 0\), i.e., $z \in \text{Fix}(T)$.
	By Lemma \ref{eqnonexpansive}, if $z_1,z_2 \in \omega_w(A_k)$, then we have
	\[
	\langle z_1-z_2,J(z_1-z_2)\rangle =0.
	\]
	Thus, $z_1=z_2$. Therefore we have verified that $\omega_w(x_k)$ contains at most one point.
	
	Consequently, $\{A_k\}$ converges weakly to $z \in \text{Fix}(T)$.

\end{proof}

\subsection{Proof of Theorem \ref{the2}} \label{existence}

\begin{proposition}[Convergence of GD method for smooth convex functions \cite{bach2024learning}] \label{gd}
	Assume that $\ell$ is $L$-smooth and convex, with a global minimizer $W^*$. 
	Choosing $\gamma_t \leq 1/L$, the iterates $(W_K)_{t \ge 0}$ of gradient descent on $\ell$ satisfy,
	\[
	\ell(W_K) - F(W^*) \le \frac{L}{2t}\|\theta_0 - \eta_*\|_2^2 , \ for\  t>0.
	\]
\end{proposition}

%\begin{theorem} [\cite{van1996weak,ledoux1991probability}] \label{number}
%	Let $X_1,\dots,X_n$ be i.i.d.\ samples drawn from a distribution $P$.
%	For any measurable function $f$, define	
%	\[
%	Pf := \mathbb E[f(X)], \qquad
%	P_n f := \frac1n\sum_{i=1}^n f(X_i).
%	\]
%	Let $\mathcal F$ be a Glivenko--Cantelli class of Banach-valued functions. Then
%	\[
%	\sup_{f\in\mathcal F}
%	\|P_n f-Pf\|\to0 \quad a.s.
%	\]
%\end{theorem}

The following presents the proof process of Theorem \ref{the2}.
\begin{proof}
	Observing that
\begin{align*}
	& \mathcal{E}(N,P,K) - \mathcal{E}(\infty,\infty,\infty)    \\
	= &	\underbrace{\mathcal{E}(N,P,K) - \mathcal{E}(N,P,\infty)}_{\text{weight error}} + \underbrace{\mathcal{E}(N,P,\infty) - \mathcal{E}(N,\infty,\infty)}_{\text{architecture error}} + \underbrace{\mathcal{E}(N,\infty,\infty) - \mathcal{E}(\infty,\infty,\infty)}_{\text{sample error}}.
\end{align*}
We should illustrate that the limit of weight error, architecture error and sample error exist, and then the limit \(\lim\limits_{N,P,K \to \infty} \mathcal{E}(N,P,K)\) exists. Theorem \ref{operator} and Proposition \ref{gd} imply the limit of architecture error and weight error exist. Observing that the training data $\{(x_i,y_i), i \in [N]\}$ are independently and identically distributed samples from the distribution \( \mathbb{P}(x,y) \), then $\ell(f_{{W}}(x_i),y_i), i \in [N]$ are independently and identically distributed. Since the expectation of $\ell(f_{{W}}(x_i),y_i), i \in [N]$, i.e., $R(f_{W}) = \mathbb{E}_{(x,y)\sim \mathbb{P}(x,y)} \ell(f_{{W}}(x),y)$, is bounded by the definition.
Based on strong law of large numbers, the sample error exists. Combining the above analysis, we can deduce that $\mathcal{E}(N,P,K) - \mathcal{E}(\infty,\infty,\infty)$ exists. This completes the proof.

\end{proof}

%\begin{theorem}[Mourier's Theorem \cite{mourier1953elements}]
%	Let $B$ be a separable Banach space and let $\{X_n\}_{n\ge1}$ be a sequence of
%	independent and identically distributed $B$-valued random variables.
%	Assume that
%	\[
%	\mathbb{E}\|X_1\| < \infty .
%	\]
%	Then $X_1$ is Bochner integrable and
%	\[
%	\frac{1}{n}\sum_{k=1}^{n} X_k
%	\;\xrightarrow{\text{a.s.}}\;
%	\mathbb{E}[X_1]
%	\quad \text{in } B .
%	\]
%	That is,
%	\[
%	\left\|
%	\frac{1}{n}\sum_{k=1}^{n} X_k - \mathbb{E}[X_1]
%	\right\|
%	\to 0
%	\quad \text{almost surely}.
%	\]
%\end{theorem}
\section{Scaling Law of Foundation Models}

\subsection{Scaling Law of Model Size} \label{modelsize}

\subsubsection{The Fixed Lipschitz Operator with $\mathrm{Lip}(T) \leq 1$}

\begin{theorem}
Suppose that the basic functional blocks of foundation models share a common Lipschitz operator $T$ satisfying the self-mapping condition, i.e., $T_i = T$ for all $i = 1,2,\ldots$, and the foundation model is constructed by $f_{k+1} = T(f_k)$, with $T$ satisfing $\gamma$-strongly quasi-nonexpansive propety, that is, \(\|Tf - f^* \|^2 \leq \|f-f^*\|^2 -\gamma \|Tf-f\|^2\), for any $f^*\in \mathrm{Fix}(T), f\in \mathbb{D}(T)$, and for some constant $\gamma>0$. Assume further that there exists a constant $C>0$, such that $\|f-f^*\| \leq C \|f-T_kf\|$ for any integer $k$,  and the loss function is $L_0$-Lipschitz continuous. Then the limit architecture satisfies the following estimation:
	\[
	|\mathcal{E}(N,P,\infty) - \mathcal{E}(N,\infty,\infty)| \lesssim  P^{-\frac{1}{2} }.
	\]

	%	(1) If the neural model is constructed in the ``T-mode'', i.e., \( f_{k+1} = T(f_k), T: D \to D, f_0 \in D, k=0,1,\cdots \), and the Lip number satisfies \( \mathrm{Lip}(T) < 1 \), then the limit architecture satisfies the following estimate:  
	%	\[
	%	\|f_K(x) - f^*(x)\| \leq  \mathrm{Lip}(T))^K  \|f_0(x)-f^*(x)\|.
	%	\]
	%	(2) If the network is constructed in the ``A-mode'', i.e., \( f_{k+1} = f_k - \alpha_k A(f_k) \), where \( A = I - T \), \( f_0 = I \), and \( \alpha_k \in (0, 1) \), then when \( m(A) > 0 \), the limit architecture satisfies the following estimate:  
	%	\[
	%	\mathcal{E}(N, P, 0) - \mathcal{E}(N, \infty, 0) \leq e^{-m(A)} \sum_{k=0}^P \alpha_k \|f_0(x) - f^*(x)\| = \mathcal{O}\left(e^{-m(A) \log P} \vee e^{-m(A) P^{1-s}}\right).
	%	\]
\end{theorem}
\begin{proof}
	Based on Theorem \ref{th1a}, $T^Nf_0$ convergence. Let $ f^*\in \mathrm{Fix}(T)$.
%	\begin{align*}
%	\|f_K- T f_K\| & = \|(f_K-f^*)-(Tf_K-f^*)\| \\
%	& = \|f_K-f^*\|^2 + \|Tf_K-f^*\|^2 -2\left\langle f_K-f^*, Tf_K-f^*  \right\rangle .
%	\end{align*}
Since $T$ satisfies $\gamma$-strongly quasi-nonexpansive propety
then we have 	\[
	\|f_{k+1} - f^* \|^2 = \|Tf_k - f^* \|^2 \leq \|f_k-f^*\|^2 -\gamma \|f_k-Tf_k\|^2.\]
	This leads to 
	\[
\gamma \|f_k-Tf_k\|^2 \leq \|f_k-f^*\|^2 - \|f_{k+1} - f^* \|^2.
	\]
Taking the sum over $k=0,1,\cdots, K$, yields	
	\[
	\gamma \sum_{k=0}^K \|Tf_k-f_k\|^2 \leq \|f_0-f^*\|^2.
	\]
	Since $\mathrm{Lip}(T) \leq 1$, $T$ is a non-expansive mapping with an equivalent norm (without loss of generality, we assume $\|\cdot\|$), then we have
	\[
	\|Tf_k-f_k\|^2 = \|Tf_k-Tf_{k-1}\|^2\leq \|f_k - f_{k-1}\|^2 = \|Tf_{k-1} - f_{k-1}\|^2.
	\]
	Therefore, we obtain the following estimation, 
	\[
	(K+1)\|Tf_K-f_K\|^2 \leq \sum_{k=0}^K \|Tf_k-f_k\|^2 \leq \frac{1}{\gamma}\|f_0-f^*\|^2.
	\]
	This yields 
		\[
	\|Tf_K-f_K\|^2 \leq \ \frac{1}{\gamma(K+1)}\|f_0-f^*\|^2,
	\]
	that is 
			\[
	\|Tf_K-f_K\| \leq \ \frac{\|f_0-f^*\|}{\sqrt{\gamma(K+1)}}.
	\]	
		Observing that 
	\begin{align*}
		| \mathcal{E}(N,P,\infty) - \mathcal{E}(N,\infty,\infty) | & \leq L_gL_0  \|f_K-f^*\|  \\
 & \leq L_gL_0 C  \|f_K-Tf_K\|  \\
 & \leq \frac{L_gL_0 C \|f_0-f^*\|}{\sqrt{\gamma(K+1)}} \lesssim P^{-\frac{1}{2}}.
	\end{align*}
		where $L_0$ is the Lipschitz constant of $\ell$ with respected to $f$, $L_g$ is the Lipschitz constant of the function $g$, and the second inequality holds since $\|f-f^*\| \leq C \|f-T_kf\|$ for any integer $k$. We assume the width $d$ of $f_K$ is finite, and $P = K \cdot d$.
%	Since $\mathrm{Lip}(T) < 1$, there exists an equivalent norm (without loss of generality, we assume $\|\cdot\|$), such that
%	\[
%	\|Tf-Tg\| \leq \mathrm{Lip}(T) \|f-g\|, \forall f,g \in D.
%	\]
%	Observing that 
%	\begin{align*}
%		| \mathcal{E}(N,P,\infty) - \mathcal{E}(N,\infty,\infty) | & \leq L_gL_0  \|f_K-f^*\| = L_0 \|T^K f_0 - T^K f^*\| \\
%		& \leq L_0 (\mathrm{Lip}(T))^K \|f_0-f^*\| \lesssim  (\mathrm{Lip}(T))^P,
%	\end{align*}
%	where $L_0$ is the Lipschitz constant of $\ell$ with respected to $f$, and $L_g$ is the Lipschitz constant of the function $g$. We assume the width $d$ of $f_K$ is finite, and $P = K \cdot d$.
\end{proof}

	\subsubsection{Finite Lipschitz Operators with $\mathrm{Lip}(T_i) \leq 1, i=[r]$}

\begin{theorem}[Scaling Law of Model Size] \label{aaaaa}
	Let $\{T_1,T_2,\ldots,T_r\}$ be the basic functional blocks of the foundation models satisfying self-mapping condition and $\mathrm{Lip}(T_i)\leq 1, i=1,2,\cdots,r$, such that
	\[
	\bigcap_{i=1}^r \mathrm{Fix}(T_i) \neq \emptyset.
	\]
	Let $U$ be generated by $T_1,T_2,\ldots,T_r$ and $\alpha_{1},\alpha_{2},\cdots,\alpha_{r} \in \{0,1\}$, 
	\begin{align*}
		U= T_1^{\alpha_{1}} T_2^{\alpha_{2}} \cdots T_r^{\alpha_{r}}.
	\end{align*}	
	%Suppose that the basic block $T_i=T, i \in [K]$ satisfies self-mapping condition and $\mathrm{Lip}(T) < 1$. 
	Suppose that the foundation model is constructed by $f_{k+1} = Uf_k$, with each $T_k$ satisfing $\gamma$-strongly quasi-nonexpansive propety, that is, \(\|T_kf - f^* \|^2 \leq \|f-f^*\|^2 -\gamma \|T_kf-f\|^2\), for any $f^*\in \cap_{i=1}^r\mathrm{Fix}(T_i), f\in \bigcap_{i=1}^r \mathbb{D}(T_i)$, and for some constant $\gamma>0$. Assume further that there exists a constant $C>0$, such that $\|f-f^*\| \leq C \|f-Uf\|$ for any integer $k$, and the loss function is $L_0$-Lipschitz continuous. Then the limit architecture satisfies the following estimation:
%	The foundation model is constructed by $f_{n+1}=Uf_n,\ n\ge 1, f_0 \in D$, and we assume there a constant $\gamma>0$, such that \(\|f_{k+1} - f^* \|^2 \leq \|f_{k+1}-f^*\|^2 -\gamma \|f_k-f_{k+1}\|^2\), and a constant $C>0$, such that $\|f_k-f^*\| \leq C \|f_k-f_{k+1}\|$, for any  $ f^*\in \mathrm{Fix}(T)$ and integer $k$. If the loss function is $L_0$-Lipschitz continuous with respect to the model output, and the performance function is $L_g$-Lipschitz, then we have the following estimation:  
	\[
	|\mathcal{E}(N,P,\infty) - \mathcal{E}(N,\infty,\infty)| \lesssim   P^{-\frac{1}{2} }.
	\]
	
	%	(1) If the neural model is constructed in the ``T-mode'', i.e., \( f_{k+1} = T(f_k), T: D \to D, f_0 \in D, k=0,1,\cdots \), and the Lip number satisfies \( \mathrm{Lip}(T) < 1 \), then the limit architecture satisfies the following estimate:  
	%	\[
	%	\|f_K(x) - f^*(x)\| \leq  \mathrm{Lip}(T))^K  \|f_0(x)-f^*(x)\|.
	%	\]
	%	(2) If the network is constructed in the ``A-mode'', i.e., \( f_{k+1} = f_k - \alpha_k A(f_k) \), where \( A = I - T \), \( f_0 = I \), and \( \alpha_k \in (0, 1) \), then when \( m(A) > 0 \), the limit architecture satisfies the following estimate:  
	%	\[
	%	\mathcal{E}(N, P, 0) - \mathcal{E}(N, \infty, 0) \leq e^{-m(A)} \sum_{k=0}^P \alpha_k \|f_0(x) - f^*(x)\| = \mathcal{O}\left(e^{-m(A) \log P} \vee e^{-m(A) P^{1-s}}\right).
	%	\]
\end{theorem}

\begin{lemma}
	Let
	\(
	U=T_1^{\alpha_1}T_2^{\alpha_2}\cdots T_r^{\alpha_r},
	\)
	where \(\alpha_1,\ldots,\alpha_r\in \{0,1\}\), and set
	\(
	m=\sum_{i=1}^r \alpha_i .
	\)
	Then \(U\) is \(\frac{\gamma}{r}\)-strongly quasi-nonexpansive, that is, for every
	\(z \in \bigcap_{i=1}^r \mathrm{Fix}(T_i)\) and every \(x\in \bigcap_{i=1}^r \mathbb{D}(T_i)\),
	\[
	\|Ux-z\|^2
	\leq
	\|x-z\|^2
	-
	\frac{\gamma}{r}\|Ux-x\|^2 .
	\]
\end{lemma}

\begin{proof}
	We write \(U\) as a composition of \(m\) mappings:
	\[
	U=S_rS_{r-1}\cdots S_1,
	\]
	where each \(S_j\) is equal to one of the mappings \(T_i\) or $I$. For a fixed
	\(x\in \bigcap_{i=1}^r \mathbb{D}(T_i)\), define
	\[
	x_0=x,\qquad x_j=S_jx_{j-1},\quad j=1,\ldots,m.
	\]
	Then \(x_r=Ux\). Since \(z\in F\), we have \(z\in \operatorname{Fix}(S_j)\)
	for every \(j=1,\ldots,r\). Hence, by the \(\gamma\)-strong
	quasi-nonexpansiveness of \(S_j\), we obtain
	\[
	\|x_j-z\|^2
	=
	\|S_jx_{j-1}-z\|^2
	\leq
	\|x_{j-1}-z\|^2
	-
	\gamma\|S_jx_{j-1}-x_{j-1}\|^2 .
	\]
	Equivalently,
	\[
	\|x_j-z\|^2
	\leq
	\|x_{j-1}-z\|^2
	-
	\gamma\|x_j-x_{j-1}\|^2 .
	\]
	Summing this inequality over \(j=1,\ldots,r\), we get
	\[
	\|x_m-z\|^2
	\leq
	\|x_0-z\|^2
	-
	\gamma\sum_{j=1}^r \|x_j-x_{j-1}\|^2 .
	\]
	Since \(x_0=x\) and \(x_r=Ux\), this gives
	\[
	\|Ux-z\|^2
	\leq
	\|x-z\|^2
	-
	\gamma\sum_{j=1}^r \|x_j-x_{j-1}\|^2 .
	\]
	On the other hand,
	\[
	Ux-x=x_r-x_0
	=
	\sum_{j=1}^r (x_j-x_{j-1}).
	\]
	By the Cauchy--Schwarz inequality,
	\[
	\|Ux-x\|^2
	\leq
	r\sum_{j=1}^r \|x_j-x_{j-1}\|^2 .
	\]
	Thus,
	\[
	\sum_{j=1}^r \|x_j-x_{j-1}\|^2
	\geq
	\frac{1}{r}\|Ux-x\|^2 .
	\]
	Substituting this estimate into the previous inequality yields
	\[
	\|Ux-z\|^2
	\leq
	\|x-z\|^2
	-
	\frac{\gamma}{r}\|Ux-x\|^2 .
	\]
	Therefore, \(U\) is \(\frac{\gamma}{r}\)-strongly quasi-nonexpansive.
\end{proof}

Based on above lemma, we provide the proof of Theorem \ref{aaaaa}.
\begin{proof}
	Let $ f^*\in \bigcap_{i=1}^r \mathrm{Fix}(T_i)$,  since $T_k$ satisfies $\gamma$-strongly quasi-nonexpansive propety, 
	%	\begin{align*}
		%	\|f_K- T f_K\| & = \|(f_K-f^*)-(Tf_K-f^*)\| \\
		%	& = \|f_K-f^*\|^2 + \|Tf_K-f^*\|^2 -2\left\langle f_K-f^*, Tf_K-f^*  \right\rangle .
		%	\end{align*}
	then we have that \(U\) is a \(\frac{\gamma}{r}\)-strongly quasi-nonexpansive operator, that is	
	\[
\|Uf - f^* \|^2 \leq \|f-f^*\|^2 -\frac{\gamma}{m} \|f-Uf\|^2.\]
		This leads to 
	\[
\frac{\gamma}{m} \|f_k-Uf_k\|^2 \leq \|f_k-f^*\|^2 - \|f_{k+1} - f^* \|^2.
	\]
	Taking the sum over $k=0,1,\cdots, K$, yields	
	\[
	\frac{\gamma}{m} \sum_{k=0}^K \|Uf_k-f_k\|^2 \leq \|f_0-f^*\|^2.
	\]
	Since $\mathrm{Lip}(U) \leq \mathrm{Lip}(T_1)^{\alpha_1} \cdot \mathrm{Lip}(T_2)^{\alpha_2} \cdots \mathrm{Lip}(T_r)^{\alpha_r} \leq 1 $, $U$ is a non-expansive mapping with an equivalent norm (without loss of generality, we assume $\|\cdot\|$), then we have
	\[
	\|Uf_k-f_k\|^2 = \|Uf_k-Uf_{k-1}\|^2\leq \|f_k - f_{k-1}\|^2 = \|Uf_{k-1} - f_{k-1}\|^2.
	\]
	Therefore, we obtain the following estimation, 
	\[
	(K+1)\|Uf_K-f_K\|^2 \leq \sum_{k=0}^K \|Uf_k-f_k\|^2 \leq \frac{r}{\gamma}\|f_0-f^*\|^2.
	\]
	This yields 
	\[
	\|Tf_K-f_K\|^2 \leq \ \frac{r}{\gamma(K+1)}\|f_0-f^*\|^2,
	\]
	that is 
	\[
	\|Tf_K-f_K\| \leq \ \frac{\|f_0-f^*\|}{\sqrt{\gamma(K+1)}}.
	\]	
	Observing that 
	\begin{align*}
		| \mathcal{E}(N,P,\infty) - \mathcal{E}(N,\infty,\infty) | & \leq L_gL_0  \|f_K-f^*\|  \\
		& \leq L_gL_0 C  \|f_K-Tf_K\|  \\
		& \leq \frac{\sqrt{r}L_gL_0 C \|f_0-f^*\|}{\sqrt{\gamma(K+1)}} \lesssim P^{-\frac{1}{2}}.
	\end{align*}
	where $L_0$ is the Lipschitz constant of $\ell$ with respected to $f$, $L_g$ is the Lipschitz constant of the function $g$, and the second inequality holds since $\|f-f^*\| \leq C \|f-T_kf\|$ for any integer $k$. We assume the width $d$ of $f_K$ is finite, and $P = K \cdot d$.	
\end{proof}
\subsubsection{Countable Lipschitz Operators with $\mathrm{Lip}(T_i) \leq 1$}

\begin{theorem}
	%	Assume the functional blocks $\{T_i\}_{i=1}^K$ of foundation models belong to the class of nonlinear Lipschitz operators, and $\ell$ is $L$-Lipschitz in its first argument. 
	%	If the foundation model is constructed by \( f_{k+1} = T_i(f_k) \), and the Lip number satisfies \( \text{Lip}(T_i) < 1 \), then the limit architecture satisfies the following estimation:  
	%	\[
	%	|\mathcal{E}(N, P, \infty) - \mathcal{E}(N, \infty, \infty)| \lesssim  (\mathrm{Lip}(T))^P %=   e^{-P |\ln(\mathrm{Lip}(T))|}.
	%	%\leq \frac{(\mathrm{Lip}(T))^K}{1 - \mathrm{Lip}(T)} \|f_0(x) - f^*(x)\| = \mathcal{O}\left(e^{-P |\ln(\mathrm{Lip}(T))|}\right).
	%	\]
	%	(2) If the network is constructed in the ``A-mode'', i.e., \( f_{k+1} = f_k - \alpha_k A(f_k) \), where \( A = I - T \), \( f_0 = I \), and \( \alpha_k \in (0, 1) \), then when \( m(A) > 0 \), the limit architecture satisfies the following estimation:  
	%	\[
	%	\mathcal{E}(N, P, 0) - \mathcal{E}(N, \infty, 0) \leq e^{-m(A)} \sum_{k=0}^P \alpha_k \|f_0(x) - f^*(x)\| = \mathcal{O}\left(e^{-m(A) \log P} \vee e^{-m(A) P^{1-s}}\right).
	%	\]
	Assume that the basic functional blocks $\{T_i\}_{i=1}^K$ of a foundation model belong to the class of nonlinear Lipschitz operators, and the basic functional blocks satisfy the self-mapping condition and $\mathrm{Lip}(T_i) \leq 1$ for all $i > K_0$ for some integer $K_0$, and also, there exists a non-expansive operator $T$ such that $\|T_i - T\| \leq \epsilon_i$, where $\sum_{i=1}^\infty \epsilon_i < \infty$. 
	Suppose that the foundation model is constructed by $f_{k+1} = T_k(f_k)$, with each $T_k$ satisfing $\gamma$-strongly quasi-nonexpansive propety, that is, \(\|T_kf - f^* \|^2 \leq \|f-f^*\|^2 -\gamma \|T_kf-f\|^2\), for any $f^*\in \cap_{i=1}^\infty\mathrm{Fix}(T_i), f\in \cap_{i=1}^\infty \mathbb{D}(T_k)$, and for some constant $\gamma>0$. Assume further that there exists a constant $C>0$, such that $\|f-f^*\| \leq C \|f-T_kf\|$ for any integer $k$, and the loss function is $L_0$-Lipschitz continuous. Then the limit architecture satisfies the following estimation:
	\[
	\big|\mathcal{E}(N, P, \infty) - \mathcal{E}(N, \infty, \infty)\big|
	\;\lesssim\; P^{-\frac{1}{2}}.
	\]
\end{theorem}

%\begin{theorem}
%	Let the sequence of \( \{T_i\} \) be the basic functional blocks of the foundation models satisfying self-mapping condition, and $\mathrm{Lip}(T_i)<1, i=1,2,\cdots$. Suppose $f_0\in D$, and $\{f_n\}$ is given by
%\[
%f_{n+1}=T_n f_n,\quad n\ge 1.
%\]
%and we assume there a constant $\gamma>0$, such that \(\|f_{k+1} - f^* \|^2 \leq \|f_{k+1}-f^*\|^2 -\gamma \|f_k-f_{k+1}\|^2\), and a constant $C>0$, such that $\|f_k-f^*\| \leq C \|f_k-f_{k+1}\|$, for any  $ f^*\in \mathrm{Fix}(T)$ and integer $k$. If the loss function is $L_0$-Lipschitz continuous with respect to the model output, and the performance function is $L_g$-Lipschitz, then we have the following estimation: 
%	\[
%	|\mathcal{E}(N,P,\infty) - \mathcal{E}(N,\infty,\infty)| \lesssim  P^{-\frac{1}{2}}.
%	\]
%	
%	%	(1) If the neural model is constructed in the ``T-mode'', i.e., \( f_{k+1} = T(f_k), T: D \to D, f_0 \in D, k=0,1,\cdots \), and the Lip number satisfies \( \mathrm{Lip}(T) < 1 \), then the limit architecture satisfies the following estimate:  
%	%	\[
%	%	\|f_K(x) - f^*(x)\| \leq  \mathrm{Lip}(T))^K  \|f_0(x)-f^*(x)\|.
%	%	\]
%	%	(2) If the network is constructed in the ``A-mode'', i.e., \( f_{k+1} = f_k - \alpha_k A(f_k) \), where \( A = I - T \), \( f_0 = I \), and \( \alpha_k \in (0, 1) \), then when \( m(A) > 0 \), the limit architecture satisfies the following estimate:  
%	%	\[
%	%	\mathcal{E}(N, P, 0) - \mathcal{E}(N, \infty, 0) \leq e^{-m(A)} \sum_{k=0}^P \alpha_k \|f_0(x) - f^*(x)\| = \mathcal{O}\left(e^{-m(A) \log P} \vee e^{-m(A) P^{1-s}}\right).
%	%	\]
%\end{theorem}

\begin{proof}
Let $ f^*\in \cap_{i=1}^\infty\mathrm{Fix}(T_i)$,  
	%	\begin{align*}
		%	\|f_K- T f_K\| & = \|(f_K-f^*)-(Tf_K-f^*)\| \\
		%	& = \|f_K-f^*\|^2 + \|Tf_K-f^*\|^2 -2\left\langle f_K-f^*, Tf_K-f^*  \right\rangle .
		%	\end{align*}
	Since $T_k$ satisfies $\gamma$-strongly quasi-nonexpansive propety
	then we have 	\[
	\|f_{k+1} - f^* \|^2 = \|T_kf_k - f^* \|^2 \leq \|f_k-f^*\|^2 -\gamma \|f_k-T_kf_k\|^2.\]
	This leads to 
	\[
	\gamma \|f_k-T_kf_k\|^2 \leq \|f_k-f^*\|^2 - \|f_{k+1} - f^* \|^2.
	\]
	Taking the sum over $k=0,1,\cdots, K$, yields	
	\[
	\gamma \sum_{k=0}^K \|T_kf_k-f_k\|^2 \leq \|f_0-f^*\|^2.
	\]
	Since $\mathrm{Lip}(T_k) \leq 1$, $T_k$ is a non-expansive mapping with an equivalent norm (without loss of generality, we assume $\|\cdot\|$), then we have
	\[
	\|T_kf_k-f_k\|^2 = \|T_kf_k-T_kf_{k-1}\|^2\leq \|f_k - f_{k-1}\|^2 = \|T_kf_{k-1} - f_{k-1}\|^2.
	\]
	Therefore, we obtain the following estimation, 
	\[
	(K+1)\|T_Kf_K-f_K\|^2 \leq \sum_{k=0}^K \|T_kf_k-f_k\|^2 \leq \frac{1}{\gamma}\|f_0-f^*\|^2.
	\]
	This yields 
	\[
	\|T_Kf_K-f_K\|^2 \leq \ \frac{1}{\gamma(K+1)}\|f_0-f^*\|^2,
	\]
	that is 
	\[
	\|T_Kf_K-f_K\| \leq \ \frac{\|f_0-f^*\|}{\sqrt{\gamma(K+1)}}.
	\]	
	Observing that 
	\begin{align*}
		| \mathcal{E}(N,P,\infty) - \mathcal{E}(N,\infty,\infty) | & \leq L_gL_0  \|f_K-f^*\|  \\
		& \leq L_gL_0 C  \|f_K-T_Kf_K\|  \\
		& \leq \frac{L_gL_0 C \|f_0-f^*\|}{\sqrt{\gamma(K+1)}} \lesssim P^{-\frac{1}{2}}.
	\end{align*}
	where $L_0$ is the Lipschitz constant of $\ell$ with respected to $f$, $L_g$ is the Lipschitz constant of the function $g$, and the second inequality holds since $\|f-f^*\| \leq C \|f-T_kf\|$ for any integer $k$. We assume the width $d$ of $f_K$ is finite, and $P = K \cdot d$.
	%	Since $\mathrm{Lip}(T) < 1$, there exists an equivalent norm (without loss of generality, we assume $\|\cdot\|$), such that
	%	\[
	%	\|Tf-Tg\| \leq \mathrm{Lip}(T) \|f-g\|, \forall f,g \in D.
	%	\]
	%	Observing that 
	%	\begin{align*}
		%		| \mathcal{E}(N,P,\infty) - \mathcal{E}(N,\infty,\infty) | & \leq L_gL_0  \|f_K-f^*\| = L_0 \|T^K f_0 - T^K f^*\| \\
		%		& \leq L_0 (\mathrm{Lip}(T))^K \|f_0-f^*\| \lesssim  (\mathrm{Lip}(T))^P,
		%	\end{align*}
	%	where $L_0$ is the Lipschitz constant of $\ell$ with respected to $f$, and $L_g$ is the Lipschitz constant of the function $g$. We assume the width $d$ of $f_K$ is finite, and $P = K \cdot d$.
\end{proof}

\subsection{Scaling Law of Data Size} \label{dara}
	
We recall a well-known result in machine learning. For its proof, we refer the reader to Lemma A.5 in \cite{edelman2022inductive}.
\begin{lemma}\label{generalizationLemma}
	Let $\mathcal{D} = \{x_i,y_i\}_{i=1}^N$ be the training dataset as above, and let $\ell: \mathbb{R}\times\mathbb{R}\rightarrow\mathbb{R}$ be a $G$-bounded loss function that is $L_0$-Lipschitz in its first argument.  
	Consider a function class $\mathcal{F}$ such that $|f| \leq A < \infty$ for all $f\in\mathcal{F}$ and 
	$\log\mathcal{N}_{\infty}(\mathcal{F};\epsilon;\{x_i\}_{i=1}^N) \leq C_{\mathcal{F}}/\epsilon^2$ for all $\{x_i\}_{i=1}^N\in X^N$. Then for any $\delta > 0$, with probability at least $1-\delta$, simultaneously for all $f\in\mathcal{F}$, it holds that
	\begin{align*}
		|R(f) - \hat{R}(f)| \leq 4cL_0\sqrt{\frac{C_{\mathcal{F}}}{N}}\left( 1+ \log\left( A\sqrt{\frac{N}{C_{\mathcal{F}}}} \right) + 2G\sqrt{\frac{\log1/\delta}{2N}} \right),
	\end{align*}
	for some constant $c>0$, where $\hat{R}(f)\overset{\triangle}{=} \ell(f_{W},\mathcal{D}) =  \frac{1}{N}\sum_{i=1}^{N} \ell(f_{W}(x_i),y_i).$
\end{lemma}

Here, we assume that \(E(\infty,\infty,\infty)\) exists, and we can express it more explicitly as follows:
\begin{align*}
	E(\infty,\infty,\infty) = R(f_{W^\dagger}) - \inf_{f_W}R(f_W) = 0, 
\end{align*}
where 
\[
W^\dagger = \arg \min_{W\in\mathcal{W}}\mathbb{E}_{(x,y)\sim\mathbb{P}(x,y)}\ell(f_W(x),y) = \arg \min_{W\in\mathcal{W}}R(f_W). 
\]
As the sample size \(N\) tends to infinity, the empirical error converges to the population error. Under the realizable PAC learning framework, we may identify \(\inf_{f_W} R(f_W) = 0\). Hence, if \(g(0) = 0\), we have  
\begin{align}\label{red1} 
	\mathcal{E}(N, \infty, \infty) - \mathcal{E}(\infty, \infty, \infty)
	= \mathcal{E}(N, \infty, \infty)
	= g\bigl(R(f_W) - \inf_{f_W} R(f_W)\bigr),
\end{align}
which implies that estimating \(\mathcal{E}(N, \infty, \infty) - \mathcal{E}(\infty, \infty, \infty)\) is equivalent to estimating \(R(f_W)\). Moreover, in a well-optimized setting, we may regard
\begin{align}\label{red2} 
\hat{R}	= \frac{1}{N} \sum_{i=1}^{N} \ell\bigl(f_{W^\dagger}(x_i), y_i\bigr) = 0
\end{align}
for the optimal parameter \(W^\dagger\). Combining formulas (\ref{red1}) and (\ref{red2}) with Lemma~\ref{generalizationLemma}, we obtain
\begin{align}\label{estimateSLDS1}
	\begin{split}
		\mathcal{E}(N, \infty, \infty) & - \mathcal{E}(\infty, \infty, \infty) = g(R(f_{W^\dagger})) \\ 
		& \leq  L_g4cL_0\sqrt{\frac{C_{\mathcal{F}}}{N}}\left( 1+ \log\left( A\sqrt{\frac{N}{C_{\mathcal{F}}}} \right) + 2G\sqrt{\frac{\log1/\delta}{2N}} \right)
	\end{split}
\end{align}
holds with probability at least $1-\delta$, where $L_g$ is the Lipschitz constant of the function $g$. For a suitably estimated constant \(C_{\mathcal{F}}\) derived from the covering number \(\mathcal{N}_{\infty}(\mathcal{F}; \epsilon; \{x_i\}_{i=1}^N)\), we obtain
\begin{align}\label{formula-ScalingLawOfDataSize}
	\mathcal{E}(N, \infty, \infty) - \mathcal{E}(\infty, \infty, \infty) \lesssim \frac{1}{\sqrt{N}} \log N \approx \frac{1}{\sqrt{N}},
\end{align}
where the final approximation \(\approx\) indicates that the lower-order \(\log N\) factor is neglected. The inequality (\ref{formula-ScalingLawOfDataSize}) is precisely a power-law relationship, which is a prototypical instance of a scaling law. 
Based on the above analysis, and in particular formula (\ref{estimateSLDS1}), the crucial task is to obtain a meaningful estimate of the covering number of the limiting architecture proposed in Section \ref{sec4}.

\begin{remark}\label{notZeroRemark}
	If we consider the agnostic PAC learning setting, the quantity \(\inf_{f_W} R(f_W)\) can no longer be assumed to be zero. Moreover, outside the well-optimized case, we may have \(\ell(f,\mathcal{D}) > 0\). In this case, formula \((\ref{estimateSLDS1})\) becomes
	\begin{align}\label{estimateSLDS2}
		\begin{split}
			\mathcal{E}(N, \infty, \infty) - & \mathcal{E}(\infty, \infty, \infty)
			= g\bigl(R(f_{W^*}) - \inf_{f_W} R(f_W)\bigr) \\
			&\leq L_g \left(\hat{R}(f)  - \inf_{f_W} R(f_W)  \right) \\
			&\quad{}+ L_g 4cL_0 \sqrt{\frac{C_{\mathcal{F}}}{N}}
			\left( 1 + \log\!\left( A\sqrt{\frac{N}{C_{\mathcal{F}}}} \right) + 2G\sqrt{\frac{\log(1/\delta)}{2N}} \right).
		\end{split}
	\end{align}
	Here, the terms \(\inf_{f_W} R(f_W)\) and \(\hat{R}(f)\) can hardly be removed, as they represent intrinsic errors, which are typically assumed to be small.
\end{remark}

\subsubsection{Covering Number of General Foundation Model}
In order to obtain an appropriate estimate of the covering number, we may need to impose certain stability conditions on the mappings \(T_i : \mathbb{R}^{n} \to \mathbb{R}^{n}\). For practical applications, the data inputted into the neural network may be represented as a matrix (for instance, in transformer models); therefore, it is preferable to consider the mapping of the \(i\)-th layer as \(T_i : \mathbb{R}^{s \times n} \to \mathbb{R}^{s \times n}\), where \(s\) is a positive integer. To emphasize the dependence on the parameters, we write this mapping as \(T_i(\cdot; W_i)\), where \(W_i\) denotes the trainable parameters. Let \(X = (x_1,\ldots,x_s)^{\top} \in \mathbb{R}^{s \times d}\), and consider the initial mapping \(f_0\) which maps \(\mathbb{R}^{s \times d}\) to \(\mathbb{R}^{s \times n}\). The overall mapping is then given by
\[
Y = f_W^*(X) = \left( \prod_{i=1}^{\infty} T_i(\cdot; W_i) \right) f_0(X).
\]
\begin{assumption}\label{stability-ScalingLawOfDataSize}
	We assume that the parameter \(W_i\) of the \(i\)-th layer mapping $T_i$ can be decomposed into \(B\) blocks, denoted by \(\{W_{ib}\}_{b=1}^{B}\), where each \(W_{ib}\) has \(n\) rows so that the product \(W_{ib}^{\top} Z^{\top}\) is well defined for any \(Z \in \mathbb{R}^{s \times n}\). For some fixed integer $K > 0$, let $M_1(\cdot,\cdot)$, $\{M_{2b}(\cdot,\cdot)\}_{b=1}^B$, and $M_3(\cdot)$ be functions depending on the layer parameters \(W_i\). For each layer of the foundation model, we assume
	\begin{align}\label{stabilityW-ScalingLawOfDataSize}
		\|T_i(Z;W_i) - T_i(Z;\hat{W}_i)\|_{*} \leq M_1(W_i,\hat{W}_i)\sum_{b=1}^B M_{2b}(W_i,\hat{W}_i)\|(W_{ib} - \hat{W}_{ib})^{\top}g_{b}(Z)^{\top}\|_{2,\infty},
	\end{align}
	and 
	\begin{align}\label{stabilityZ-ScalingLawOfDataSize}
		\|T_i(Z;W_i) - T_i(\hat{Z};W_i)\|_{*} \leq M_3(W_i)\|Z - \hat{Z}\|_{*},
	\end{align}
	where $\|\cdot\|_{*}$ represents a norm defined on matrix space, such as the $(p,q)$ matrix norm $\|\cdot\|_{pq}$ or the spectral norm of the matrix, and $\{g_b\}_{b=1}^{B}$ denotes a sequence of functions. For instance, $g_b$ may be chosen as the identity function in certain applications.
\end{assumption}
The assumptions (\ref{stabilityW-ScalingLawOfDataSize}) and (\ref{stabilityZ-ScalingLawOfDataSize}) characterize stability with respect to the trainable parameters and the data, respectively, and are crucial for estimating the covering number.

%%% ------------------------------------------------------------------------------------------------
\begin{theorem}[Logarithm of the Covering Number for Foundation Models]
	\label{theorem-coveringNumber}
	Let $W_{1:K} = (W_1,\dots,W_K)$ be the collection of parameters of $\{T_i\}_{i=1}^{K}$,
	and define the composition of the first $K$ layers by
	\[
	f_{K}(Z;W_{1:K})
	= \biggl(\prod_{k=1}^{K}T_{k}(\,\cdot\,;W_k)\biggr)Z
	= T_{K}\bigl(\cdots T_2(T_1(Z;W_1);W_2)\cdots;W_K\bigr),
	\]
	where $Z = f_0(X)$. Assume that Assumption~\ref{stability-ScalingLawOfDataSize} holds. Define the function class
	\[
	\mathcal{F}_K
	= \bigl\{
	f_{K}(\,\cdot\,;W_{1:K})
	:
	\|W_{kb}\|_{2,1} \leq \mathcal{B}_{kb}^W < \infty,\;
	k\in[K],\; b\in[B]
	\bigr\},
	\]
	which induces the infinite-depth function class $\mathcal{F} = \bigcup_{K}\mathcal{F}_{K}$.
	Suppose the following conditions hold:
	\begin{enumerate}
		\item[(A1)]
		For each function $g_b$ in Assumption~\ref{stability-ScalingLawOfDataSize},
		\[
		\|g_b(f_K(Z;W_{1:K}))^T\|_{2,\infty} \leq \mathcal{B}_{Kb}^Z < \infty \quad \text{for all } K\in\mathbb{K}^{+} \text{ and } b\in[B].
		\]
		\item[(A2)]
		The functions $M_1(\,\cdot\,,\,\cdot\,)$, $\{M_{2b}(\,\cdot\,,\,\cdot\,)\}_{b\in[B]}$, and $M_3(\,\cdot\,)$
		from Assumption~\ref{stability-ScalingLawOfDataSize}
		are uniformly bounded layer-wise:
		\[
		M_1(W_k,\hat{W}_k) \leq M_{1k},\quad
		M_{2b}(W_k,\hat{W}_k) \leq M_{2bk},\quad
		M_3(W_k)\leq M_{3k},
		\]
		where $M_{1k}$, $M_{2bk}$, and $M_{3k}$ are finite constants for all $k\in[K]$.
		\item[(A3)]
		The following infinite-depth summability condition holds:
		\[
		C_{\mathcal{F}} := \sum_{k=1}^{\infty}
		M_{1k}^{2/3}
		\biggl(\prod_{j=k+1}^{\infty} M_{3j}\biggr)^{2/3}
		\sum_{b=1}^{B}
		\bigl(\mathcal{B}_{kb}^Z \mathcal{B}_{kb}^W M_{2bk}^2\bigr)^{1/3}
		< \infty.
		\]
	\end{enumerate}
	Let $\{Z_j = f_0(X_j)\}_{j=1}^{N}$ be a finite set of inputs such that $\sup_{j\in[N]}\|Z_j^T\|_{2,\infty} \leq \min_{b\in[K]}\mathcal{B}_{1b}^Z$. Then we have the following estimate for the logarithm of the covering number:
	\[
	\log\mathcal{N}_{\infty}\bigl(
	\mathcal{F};\, \epsilon;\,
	\{Z_j\}_{j=1}^{N};\, \|\cdot\|_{*}
	\bigr)
	\lesssim
	\frac{C_{\mathcal{F}}}{\epsilon^2}.
	\]
\end{theorem}
%%% ------------------------------------------------------------------------------------------------
\begin{proof}
	\textbf{Step 1: Recurrent Formula.}
	For two sets of parameters $W_{1:K}$ and $\hat{W}_{1:K}$, obviously, we have 
	\begin{align*}
		& \|f_{K}(Z;W_{1:K}) -  f_{K}(Z;\hat{W}_{1:K})\|_{*} \\
		 \leq &
		\|T_K(f_{K-1}(Z;W_{1:K-1});W_K) - T_K(f_{K-1}(Z;\hat{W}_{1:K-1});\hat{W}_K)\|_{*} \\
		\leq & \|T_K(f_{K-1}(Z;W_{1:K-1});W_K) - T_K(f_{K-1}(Z;\hat{W}_{1:K-1});W_K)\|_{*} \\
		& \qquad + \|T_K(f_{K-1}(Z;\hat{W}_{1:K-1});W_K) - T_P(f_{P-1}(Z;\hat{W}_{1:K-1});\hat{W}_K)\|_{*} \\
		= & I_1 + I_2.
	\end{align*}
	For $I_1$, we adopt formula (\ref{stabilityZ-ScalingLawOfDataSize}) of Assumption \ref{stability-ScalingLawOfDataSize}, which yields
	\begin{align*}
		I_1 \leq & M_3(W_K)\|f_{K-1}(Z;W_{1:K-1}) - f_{K-1}(Z;\hat{W}_{1:K-1})\|_{*} \\
		\leq & M_{3K}\|f_{K-1}(Z;W_{1:K-1}) - f_{K-1}(Z;\hat{W}_{1:K-1})\|_{*}.
	\end{align*}
	Similarly, using formula (\ref{stabilityW-ScalingLawOfDataSize}), we obtain 
	\begin{align*}
		I_2 \leq & M_1(W_K,\hat{W}_K)\sum_{b=1}^{B}M_2(W_K,\hat{W}_K)\|(W_{Kb} - \hat{W}_{Kb})^\top f_{K-1}(Z;\hat{W}_{1:K-1})^\top\|_{2,\infty} \\
		\leq & M_{1K}\sum_{b=1}^{B}M_{2bK}\|(W_{Kb} - \hat{W}_{Kb})^\top g_b(f_{K-1}(Z;\hat{W}_{1:K-1}))^\top\|_{2,\infty}.
	\end{align*}
	Plugging the estimation of $I_1$ and $I_2$, we finally obtain
	\begin{align}\label{iterativeFormula-Step1}
		\begin{split}
			\|f_{K}(Z;W_{1:K}) - & f_{K}(Z;\hat{W}_{1:K})\|_{*} 
			\leq M_{3K}\|f_{K-1}(Z;W_{1:K-1}) - f_{K-1}(Z;\hat{W}_{1:K-1})\|_{*} \\
			& + M_{1K}\sum_{b=1}^{B}M_{2bK}\|(W_{Kb} - \hat{W}_{Kb})^\top g_b(f_{K-1}(Z;\hat{W}_{1:K-1}))^\top\|_{2,\infty}.
		\end{split}
	\end{align}

	\textbf{Step 2: Construct a Cover.} 
	To estimate the covering number for the $K$-depth foundation model, our objective is to demonstrate that for every $\epsilon > 0$, and a collection of inputs $\{Z_j = f_0(X_j)\}_{j=1}^{N}$, there exists a cover $\mathcal{C}_K$ such that for all $W_{1:K}$, there exists some $v \in \mathcal{C}_P$ satisfying 
	\[
	\max_{1\leq j\leq N}\|f_{K}(Z_j;W_{1:K}) - v\| \leq \epsilon. 
	\]
	Our cover $\mathcal{C}_K$ shall be properly constructed, comprising vectors in the form $\{f_{K}(Z_j;\hat{W}_{1:K})\}_{j=1}^{N}$. We will construct the cover iteratively by identifying finite collections of matrices $\hat{W}_{1:K}$ for each layer.
	
	First observe that for the last terms in (\ref{iterativeFormula-Step1}), we have 
	\begin{align*}
		& \max_{j\in[N]}\left\|(W_{Kb} - \hat{W}_{Kb})^\top g_b(f_{K-1}(Z_j;\hat{W}_{1:K-1}))^\top \right\|_{2,\infty} \\
		=& \max_{j\in[N],i\in[s]}\|(W_{Kb} - \hat{W}_{Kb})^\top g_b(f_{K-1}(Z_j;\hat{W}_{1:K-1}))_i\|_2,
	\end{align*}
	where $g_b(f_{K-1}(Z_j;\hat{W}_{1:K-1}))_i\in\mathbb{R}^n$ is the $i$-th row of the matrix $g_k(f_{K-1}(Z_j;\hat{W}_{1:K-1}))\in\mathbb{R}^{s\times n}$.
	Recall our assumptions on $\{W_{Kb}\}$, we can employ Lemma 4.6 of \cite{edelman2022inductive} to obtain a cover for the function class 
	\[
	\mathcal{F}_{Kb} = \{x\rightarrow W_{Kb}^Tx\,:\, W_{Kb}\in\mathbb{R}^{n\times d}, \|W_{Kb}\|_{2,1}\leq \mathcal{B}_{Kb}^W\}
	\]
	with 
	\begin{align}\label{cover-estimate-Pk}
		\log\mathcal{N}_{\infty}(\mathcal{F}_{Kb};\epsilon;\{g_b(f_{K-1}(Z_j;\hat{W}_{1:K-1}))_i\}_{i\in[s],j\in[N]};\|\cdot\|_2)
		\lesssim \left(\frac{\mathcal{B}_{Kb}^Z \mathcal{B}_{Kb}^W}{\epsilon}\right)^2 \log(nsN).
	\end{align}
	
	Now let us build a cover for $f_1(\cdot;W_{1})$ with inputs $\{Z_j=f_0(X_j)\}_{j=1}^N$. Concerned with $f_1(\cdot;W_{1})$, the iterative formula (\ref{iterativeFormula-Step1}) reduces to 
	\begin{align}\label{stability-f_2}
		\|f_1(Z_j;W_{1}) - f_1(Z_j;\hat{W}_{1})\|_{*} \leq M_{11}\sum_{b=1}^{B}M_{2b1}\|(W_{1b} - \hat{W}_{1b})^\top Z_j^\top\|_{2,\infty}. 
	\end{align}
	Because $\sup_{j\in[N]}\|Z_j^\top\|_{2,\infty} \leq \min_{b\in[B]}\mathcal{B}_{1b}^Z$, as for deriving estimate (\ref{cover-estimate-Pk}), we could find a $\epsilon_{1b}$-cover of $\mathcal{F}_{1b}$ with 
	\begin{align}\label{cover-estimate-1k}
		\log\mathcal{N}_{\infty}(\mathcal{F}_{1b};\epsilon_{1b};\{Z_j\}_{j\in[N]};\|\cdot\|_2)
		\lesssim \left(\frac{\mathcal{B}_{1b}^Z \mathcal{B}_{Kb}^W}{\epsilon_{1b}}\right)^2 \log(nsN).
	\end{align}
	Considering inequality (\ref{stability-f_2}), we can derive an $\epsilon_1$-cover for $f_1(\cdot;W_{1:2})$ and inputs $Z_1,\ldots,Z_N$ from the $\epsilon_{1b}$-cover for $\mathcal{F}_{1b}$, where 
	\begin{align*}
		\epsilon_1 = \sum_{b=1}^B M_{11}M_{2b1} \epsilon_{1b}.
	\end{align*}
	In addition, we obviously have the following estimate of the covering number
	\begin{align}\label{coverNumber-f_2}
		\log\mathcal{N}_{\infty}(\mathcal{F}_1;\epsilon_1;\{Z_j\}_{j\in[N]};\|\cdot\|_{*}) \lesssim 
		\sum_{b=1}^{B}\left(\frac{\mathcal{B}_{1b}^Z \mathcal{B}_{1b}^W}{\epsilon_{1b}}\right)^2 \log(nsN),
	\end{align}
	where 
	\[
	\mathcal{F}_1 = \left\{ f_1(\cdot;W_{1}) \,:\, \|W_{1b}\|_{2,1} \leq \mathcal{B}_{1b}^W \right\}.
	\]
	With these preparations, we can utilize the iterative formula (\ref{iterativeFormula-Step1}) to construct an $\epsilon_2$-cover for $f_2(\cdot,W_{1:2})$ and inputs $Z_1,\ldots,Z_N$, where 
	\begin{align}
		\epsilon_2 = M_{32}\epsilon_1 + \sum_{b=1}^B M_{12}M_{2b2} \epsilon_{2k}.
	\end{align}
	In addition, we have the estimate of the covering number 
	\begin{align}\label{coverNumber-f_3}
		\log\mathcal{N}_{\infty}(\mathcal{F}_2;\epsilon_1;\{Z_j\}_{j\in[N]};\|\cdot\|_{*}) \lesssim
		\sum_{k=1}^{2}\sum_{b=1}^{B}\left(\frac{\mathcal{B}_{kb}^Z \mathcal{B}_{kb}^W}{\epsilon_{kb}}\right)^2 \log(nsN),
	\end{align}
	where 
	\[
	\mathcal{F}_2 = \left\{ f_2(\cdot;W_{1:2}) \,:\, \|W_{kb}\|_{2,1} \leq B_{kb}^W \text{ for }k\in[2]\text{ and }b\in[B] \right\}.
	\]
	Based on the estimate (\ref{iterativeFormula-Step1}), we can derive through iteration:
	\begin{align}\label{sum-epsilonP}
		\epsilon_K = \sum_{k=1}^K \alpha_k \sum_{b=1}^B M_{2bk} \epsilon_{kb},
	\end{align}
	where $\alpha_k = M_{1k} \prod_{j=k+1}^K M_{3j} $.
	Similarly, an estimate of the covering number can be expressed as follows:
	\begin{align}\label{coverNumber-f_P}
		\log\mathcal{N}_{\infty}(\mathcal{F}_P;\epsilon_P;\{Z_j\}_{j\in[N]};\|\cdot\|_{*}) \lesssim
		\sum_{p=1}^{P}\sum_{k=1}^{K}\left(\frac{B_{pk}^Z B_{pk}^W}{\epsilon_{pk}}\right)^2 \log(nsN),
	\end{align}
	where 
	\begin{align*}
		\mathcal{F}_K = \left\{ f_{K}(\cdot;W_{1:K+1}) \,:\, \|W_{kb}\|_{2,1} \leq B_{kb}^W \text{ for }k\in[K]\text{ and }b\in[B] \right\}.
	\end{align*}
	As observed in (\ref{sum-epsilonP}) and (\ref{coverNumber-f_P}), the parameters $\{\epsilon_{kb}\}_{k\in[K], b\in[B]}$ can be optimized to achieve the minimal upper bound of the logarithm of the covering number. Specifically speaking, we need to solve the following optimization problem
	\begin{align}\label{optim-coverNumber}
		\begin{split}
			\min_{\epsilon_{kb}} & \sum_{k=1}^{K}\sum_{b=1}^{B}\left(\frac{\mathcal{B}_{kb}^Z \mathcal{B}_{kb}^W}{\epsilon_{kb}}\right)^2 \\
			\text{subject to } & \sum_{k=1}^K \alpha_k \sum_{b=1}^B M_{2bk} \epsilon_{kb} = \epsilon_K.
		\end{split}
	\end{align}
	For the problem stated in (\ref{optim-coverNumber}), employing a standard Lagrangian analysis as outlined in Appendix E of \cite{bishop2006pattern}, the optimal value is given by
	\begin{align*}
		\frac{1}{\epsilon_K^2} \sum_{k=1}^{K}\sum_{b=1}^{B}\left( \mathcal{B}_{kb}^Z \mathcal{B}_{kb}^W \right)^{1/3} \left( \alpha_k M_{2kb} \right)^{2/3},
	\end{align*}
	and it is attained at
	\begin{align*}
		\epsilon_{kb} = \frac{\epsilon_K }{\sum_{k=1}^{K}\sum_{b=1}^{B}\left( \mathcal{B}_{kb}^Z \mathcal{B}_{kb}^W \right)^{1/3} \left( \alpha_k M_{2bk} \right)^{2/3}} \left( \frac{\mathcal{B}_{kb}^Z \mathcal{B}_{kb}^W}{\alpha_k M_{2bk}} \right)^{1/3}.
	\end{align*}
	All of the aforementioned illustrations indicates 
	\begin{align}\label{coverNumber-f_P-final}
		\log\mathcal{N}_{\infty}(\mathcal{F}_K;\epsilon_K;\{Z_j\}_{j\in[N]};\|\cdot\|_{*}) \lesssim
		\frac{1}{\epsilon_K^2} \sum_{k=1}^{K}\sum_{b=1}^{B}\left( \mathcal{B}_{kb}^Z \mathcal{B}_{kb}^W \right)^{1/3} \left( \alpha_k M_{2bk} \right)^{2/3}.
	\end{align}
	According to the following assumption 
	\begin{align*}
		\sum_{k=1}^{\infty} M_{1k}^{2/3} \prod_{j=k+1}^{\infty} M_{3j}^{2/3} \sum_{b=1}^{B} \left( \mathcal{B}_{kb}^Z  \mathcal{B}_{kb}^W M_{2bk}^2 \right)^{1/3} < \infty,
	\end{align*}
	we can obviously extend the depth parameter $K$ to infinity, which yields the bound 
	\begin{align*}
		\log\mathcal{N}_{\infty}(\mathcal{F};\epsilon;\{Z_j\}_{j\in[N]};\|\cdot\|_{*}) \lesssim
		\frac{1}{\epsilon^2} \sum_{k=1}^{\infty} M_{1k}^{2/3} \prod_{j=k+1}^{\infty} M_{3j}^{2/3} \sum_{b=1}^{B} \left( \mathcal{B}_{kb}^Z \mathcal{B}_{kb}^W M_{2bk}^2 \right)^{1/3},
	\end{align*}
	which completes the proof. 
\end{proof}

\subsubsection{Covering Number of MLP} \label{mlp}
% In this subsection, let us consider the MLP model. For the MLP, we could define $T_i(X) = \sigma(XW_{i2})W_{i1}$, where $W_i = (W_{i1},W_{i2})$ with $W_{i1},W_{i2}\in\mathbb{R}^{n\times n}$ and $X\in\mathbb{R}^{s\times n}$. We can define the overall infinite-depth model iteratively as follows: 
% \begin{align*}
	% f_{P}(X) = T_P(f_{P-1}(X)),
	% \end{align*}
% where $f_1(X) = \sigma(XW_{12})W_{11}$ and $Z:=f_0(X) = X$. In order to apply the general Theorem \ref{theorem-coveringNumber}, we firsly calculate the quantities introduced in Assumption \ref{stability-ScalingLawOfDataSize}.

% %% --------------------------------------------------------------------------------------------------
In this subsection, we specialize the general framework for foundation models established above to the classical Multi-Layer Perceptron (MLP) architecture. We first verify that the MLP layer mapping satisfies Assumption \ref{stability-ScalingLawOfDataSize}, then check the three conditions in Theorem \ref{theorem-coveringNumber}, and finally derive the corresponding covering number bound for infinite-depth MLPs.

We first formalize the infinite-depth MLP model. Let $\sigma: \mathbb{R} \to \mathbb{R}$ be an \emph{$L_{\sigma}$-Lipschitz continuous activation function} satisfying $\sigma(0)=0$ (e.g., ReLU, Leaky ReLU, or sigmoid), which obeys
\begin{equation}
	|\sigma(a) - \sigma(b)| \leq L_{\sigma} |a - b|, \quad \forall a,b \in \mathbb{R},
\end{equation}
and is extended element-wise to matrices of arbitrary dimensions. For the $i$-th layer of the MLP, we define the layer mapping as
\begin{equation}
	T_i(Z; W_i) = \sigma(Z W_{i2}) W_{i1},
\end{equation}
where the trainable parameters are decomposed into two blocks $W_i = (W_{i1}, W_{i2}) \in \mathbb{R}^{n \times n} \times \mathbb{R}^{n \times n}$ (i.e., $K=2$ blocks, each with exactly $n$ columns, satisfying the dimension requirement in Assumption \ref{stability-ScalingLawOfDataSize}), and $Z \in \mathbb{R}^{s \times n}$ is the input matrix with batch size $s$. The overall MLP mapping is defined iteratively as
\begin{equation}
	f_K(Z; W_{1:K}) = T_K\bigl(T_{K-1}\bigl(\cdots T_1(Z; W_1)\cdots\bigr); W_K\bigr),
\end{equation}
with the initial mapping $Z := f_0(X) = X$ for input $X \in \mathbb{R}^{s \times n}$. The infinite-depth MLP function class is defined as $\mathcal{F}_{MLP} = \bigcup_{k=1}^\infty \mathcal{F}_k$, where
\begin{equation}
	\mathcal{F}_K = \bigl\{ f_K(\cdot; W_{1:K}) : \|W_{k1}\|_{2,1} \leq \mathcal{B}_{k1}^W < \infty,\ \|W_{k2}\|_{2,1} \leq \mathcal{B}_{k2}^W < \infty,\ k \in [K] \bigr\},
\end{equation}
and $\|\cdot\|_{2,1}$ denotes the matrix $(2,1)$-norm (sum of the $\ell_2$-norms of the columns). For consistency with the stability analysis, we choose the matrix norm $\|\cdot\|_* = \|(\cdot)^\top\|_{2,\infty}$ (maximum of the $\ell_2$-norms of the columns) throughout this subsection.

\textbf{Verification of Assumption \ref{stability-ScalingLawOfDataSize}}. 
We now verify that the MLP layer mapping satisfies both the input stability and parameter stability conditions in Assumption \ref{stability-ScalingLawOfDataSize}.

\textbf{Input Stability (Equation \eqref{stabilityZ-ScalingLawOfDataSize})}. 
For fixed parameters $W_i$, consider the difference of the layer mappings with respect to inputs $Z, \hat{Z} \in \mathbb{R}^{s \times n}$:
\begin{equation}
	(T_i(Z; W_i) - T_i(\hat{Z}; W_i))^\top = W_{i1}^\top\bigl[\sigma(Z W_{i2}) - \sigma(\hat{Z} W_{i2})\bigr]^\top.
\end{equation}
Taking the $(2,\infty)$-norm on both sides and applying the submultiplicative inequality $\|AB\|_{2,\infty} \leq \|A\|_{2}\|B\|_{2,\infty}$ (where $\|\cdot\|_2$ denotes the spectral norm), we obtain
\begin{equation}
	\|(T_i(Z; W_i) - T_i(\hat{Z}; W_i))^\top\|_{2,\infty} \leq \|W_{i1}\|_2\|\sigma(Z W_{i2})^\top - \sigma(\hat{Z} W_{i2})^\top\|_{2,\infty}.
\end{equation}
By the $L_{\sigma}$-Lipschitz continuity of $\sigma$, we have $\|(\sigma(A) - \sigma(B))^\top\|_{2,\infty} \leq L_{\sigma} \|(A - B)^\top\|_{2,\infty}$ for any matrices $A,B$ of compatible dimensions. Substituting this into the above inequality yields
\begin{align*}
	\|(T_i(Z; W_i) - T_i(\hat{Z}; W_i))^\top\|_{2,\infty} \leq & L_{\sigma} \|W_{i1}\|_2\|W_{i2}^\top(Z - \hat{Z})^T\|_{2,\infty}\\
	\leq & L_{\sigma} \|W_{i1}\|_2\|W_{i2}\|_2\|W_{i2}^\top(Z - \hat{Z})^\top\|_{2,\infty}.
\end{align*}
This exactly matches the input stability condition \eqref{stabilityZ-ScalingLawOfDataSize} with
\begin{equation}
	M_3(W_i) = L_{\sigma} \|W_{i1}\|_2 \|W_{i2}\|_2.
\end{equation}

\textbf{Parameter Stability (Equation \eqref{stabilityW-ScalingLawOfDataSize})}.
For fixed input $Z$, consider the difference of the layer mappings with respect to parameters $W_i = (W_{i1}, W_{i2})$ and $\hat{W}_i = (\hat{W}_{i1}, \hat{W}_{i2})$. We split the difference via the triangle inequality:
\begin{equation}
	(T_i(Z; W_i) - T_i(Z; \hat{W}_i))^T = \Delta_1 + \Delta_2,
\end{equation}
where
\begin{equation}
	\Delta_1 = (W_{i1} - \hat{W}_{i1})^T\sigma(Z W_{i2})^\top, \quad \Delta_2 = \hat{W}_{i1}^T\bigl[\sigma(Z W_{i2}) - \sigma(Z \hat{W}_{i2})\bigr]^\top.
\end{equation}

For $\Delta_1$, applying the submultiplicative inequality gives
\begin{equation}
	\|\Delta_1\|_{2,\infty} \leq \|W_{i1} - \hat{W}_{i1}\|_2\|\sigma(Z W_{i2})^\top\|_{2,\infty}.
\end{equation}
By the Lipschitz continuity of $\sigma$ and $\sigma(0)=0$, we have $\|\sigma(Z W_{i2})^\top\|_{2,\infty} \leq L_{\sigma} \|W_{i2}^\top Z^\top\|_{2,\infty} \leq L_{\sigma} \|W_{i2}\|_2\mathcal{B}^Z$, where $\mathcal{B}^Z = \min_b\mathcal{B}_{1b}^Z$. Obviously, we obtain
\begin{equation}
	\|\Delta_1\|_{2,\infty} \leq L_{\sigma} \mathcal{B}^Z \|W_{i2}\|_2 \cdot \|(W_{i1} - \hat{W}_{i1})^\top g_1(Z)^\top\|_{2,\infty},
\end{equation}
where $g_1(Z) = I_n$ is the constant identity function.

For $\Delta_2$, applying the submultiplicative inequality and the Lipschitz continuity of $\sigma$ yields
\begin{align*}
	\|\Delta_2\|_{2,\infty} \leq & \|\hat{W}_{i1}\|_2\|(\sigma(Z W_{i2}) - \sigma(Z \hat{W}_{i2}))^\top\|_{2,\infty} \\
	\leq & L_{\sigma} \|\hat{W}_{i1}\|_2 \|(W_{i2} - \hat{W}_{i2})^\top Z^\top\|_{2,\infty} \\
	\leq & L_{\sigma} \|\hat{W}_{i1}\|_2 \|(W_{i2} - \hat{W}_{i2})^\top g_2(Z)^\top\|_{2,\infty},
\end{align*}
where $g_2(Z) = Z$. 

Combining the bounds for $\Delta_1$ and $\Delta_2$, we arrive at
\begin{equation}
	\|T_i(Z; W_i) - T_i(Z; \hat{W}_i)\|_{*} \leq M_1(W_i, \hat{W}_i) \sum_{k=1}^2 M_{2k}(W_i, \hat{W}_i) \|(W_{ik} - \hat{W}_{ik})^\top g_k(Z)^\top\|_{2,\infty},
\end{equation}
where 
\begin{equation}
	M_1(W_i, \hat{W}_i) = L_{\sigma}, \quad M_{21}(W_i, \hat{W}_i) = \mathcal{B}^Z \|W_{i2}\|_2, \quad M_{22}(W_i, \hat{W}_i) =  \|\hat{W}_{i1}\|_2,
\end{equation}
and $g_1(Z) = I_n$, $g_2(Z) = Z$. This exactly matches the parameter stability condition \eqref{stabilityW-ScalingLawOfDataSize}, completing the verification of Assumption \ref{stability-ScalingLawOfDataSize}.

To apply Theorem \ref{theorem-coveringNumber}, we need to verify the three conditions (A1)--(A3). We first impose a standard uniform contraction condition on the MLP: there exists a constant $\gamma \in (0,1)$ such that
\begin{equation}
	L_{\sigma} \mathcal{B}_{k1}^W \mathcal{B}_{k2}^W \leq \gamma < 1, \quad \forall k \geq 1.
\end{equation}
This condition ensures that the infinite-depth MLP mapping is well-defined (i.e., the limit $\lim_{K \to \infty} f_K(Z)$ exists for all bounded inputs $Z$) and is widely used in the analysis of infinite-depth neural networks.

\begin{enumerate}
	\item[\textbf{(A1)}] \textbf{Uniform Boundedness of $g_k(f_K(Z))$.}
	For $k=1$, $g_1(Z) = I_n$, so we have
	\begin{equation}
		\|g_1(f_K(Z))^\top\|_{2,\infty} = \|I_n^\top\|_{2,\infty} = 1, \quad \forall K \in \mathbb{Z}^+.
	\end{equation}
	We set $\mathcal{B}_{K1}^Z = 1 < \infty$ for all $K$.
	For $k=2$, $g_2(Z) = Z$, so we need to bound $\|f_K(Z)^\top\|_{2,\infty}$ uniformly over $K$. We proceed by induction:
	\begin{itemize}
		\item Base case ($K=0$): Assume that $\|X^\top\|_{2,\infty} \leq \mathcal{B}^Z < \infty$, thus it follows that 
		$\|f_0(X)^\top\|_{2,\infty} \leq \mathcal{B}^Z < \infty$.
		\item Inductive step: Assume $\|f_{K-1}(Z)\|_{2,\infty} \leq \mathcal{B}^Z$ for some $K \geq 1$. Then
		\begin{align*}
			\|f_K(Z)\|_{2,\infty} = & \|W_{K1}^\top \sigma(f_{K-1}(Z) W_{K2})^\top\|_{2,\infty} \\
			\leq & L_{\sigma} \|W_{K1}\|_2 \|W_{K2}\|_2 \|f_{K-1}(Z)^\top\|_{2,\infty}   \\
			\leq & \gamma B^Z \leq B^Z.
		\end{align*}
	\end{itemize}
	Condition (A1) is consequently satisfied under the assumption $\|X^\top\|_{2,\infty} \leq B^Z < \infty$.
	\item[\textbf{(A2)}] \textbf{Layer-wise Uniform Boundedness of Stability Constants.}
	Using the spectral norm bound $\|W\|_2 \leq \|W\|_{2,1}$ for any matrix $W$, we derive the following uniform bounds:
	\begin{align*}
		M_1(W_k, \hat{W}_k) &= L_{\sigma} =: M_{1k}, \quad 
		M_{21}(W_k, \hat{W}_k)\leq \mathcal{B}^Z \mathcal{B}_{k2}^W =: M_{21k}, \\
		M_{22}(W_k, \hat{W}_k) &\leq \mathcal{B}_{k1}^W =: M_{22k}, \quad 
		M_3(W_k)\leq L_{\sigma} \mathcal{B}_{k1}^W \mathcal{B}_{k2}^W \leq \gamma =: M_{3k}.
	\end{align*}
	All constants $M_{1k}, M_{21k}, M_{22k}, M_{3k}$ are finite and uniformly bounded layer-wise, so Condition (A2) is satisfied.
	\item[\textbf{(A3)}] \textbf{Infinite-Depth Summability Condition.}
	We need to verify that the constant $C_{\mathcal{F}}$ defined in Theorem \ref{theorem-coveringNumber} is finite:
	\begin{equation}
		C_{\mathcal{F}} := \sum_{k=1}^{\infty} M_{1k}^{2/3} \biggl(\prod_{j=k+1}^{\infty} M_{3j}\biggr)^{2/3} \sum_{b=1}^{2} \bigl(\mathcal{B}_{kb}^Z \mathcal{B}_{kb}^W M_{2bk}^2\bigr)^{1/3} < \infty.
	\end{equation}
	By the uniform contraction condition, $M_{3j} \leq \gamma < 1$ for all $j$, so the infinite product satisfies
	\begin{equation}
		\prod_{j=k+1}^{\infty} M_{3j} \leq \lim_{K \to \infty} \gamma^{K-k} = 0,
	\end{equation}
	and the partial sum of the product terms is bounded by a convergent geometric series:
	\begin{equation}
		\sum_{k=1}^K \biggl(\prod_{j=k+1}^K M_{3j}\biggr)^{2/3} \leq \sum_{k=0}^\infty \gamma^{2k/3} = \frac{1}{1 - \gamma^{2/3}} < \infty.
	\end{equation}
\end{enumerate}

Thus all conditions of Theorem \ref{theorem-coveringNumber} have been satisfied, leading us to directly obtain the following bound on the covering number for infinite-depth MLPs.

\begin{corollary}[Logarithm of the Covering Number for Infinite-Depth MLPs]
	\label{corollary:mlp-covering-number}
	Let $\{X_j\}_{j=1}^N$ represent a finite set of inputs, where $\max_{j \in [N]} \|X_j\|_{2,\infty} \leq \mathcal{B}^Z < \infty$. Consider an infinite-depth MLP that satisfies the uniform contraction condition, specifically $L_{\sigma} \mathcal{B}_{k1}^W \mathcal{B}_{k2}^W \leq \gamma < 1$, with both $\mathcal{B}_{k1}^W$ and $\mathcal{B}_{k2}^W$ being uniformly bounded for all $k \geq 1$. Then the logarithm of the covering number of the infinite-depth MLP function class $\mathcal{F}$ satisfies
	\begin{equation}
		\log\mathcal{N}_{\infty}\bigl(
		\mathcal{F}_{MLP};\, \epsilon;\,
		\{Z_j\}_{j=1}^{N};\, \|\cdot\|_{2,\infty}
		\bigr)
		\lesssim
		\frac{C_{\mathcal{F}}}{\epsilon^2},
	\end{equation}
	where $C_{\mathcal{F}_{MLP}} < \infty$ is the finite constant defined in Condition (A3) of Theorem \ref{theorem-coveringNumber}.
\end{corollary}

\subsubsection{Covering Number of Transformer} \label{trans}
In this subsection, our focus will be on the transformer architecture defined in Example \ref{ex7}. For the transformer architecture, it should be noted that the layer operator $T_i(\cdot;W_i)$ can be defined as follows:
\begin{align}
	\begin{split}
		T_i(Z;W_i) = \textstyle \prod_{\text{norm}}\Bigg( f_0(X) + \sigma\Bigg( \textstyle \prod_{\text{norm}} \Bigg( 
		f_0(X) + f(Z; \{W_{i1}, W_{i2}\})
		\Bigg) \Bigg) W_{i3}\Bigg),
	\end{split}
\end{align}
where 
\[
f(Z;\{W_{i1}, W_{i2}\}) := \text{RowSoftmax}(ZW_{i1}Z^T)ZW_{i2},
\]
and $X\in\mathbb{R}^{s\times d}$, $f_0(X)\in\mathbb{R}^{s\times n}$, $Z\in\mathbb{R}^{s\times n}$, $W_{i1}\in\mathbb{R}^{n\times n}$, $W_{i2}\in\mathbb{R}^{n\times k}$, and $W_{i3}\in\mathbb{R}^{k\times n}$. The parameters $W_{i1}$, $W_{i2}$, and $W_{i3}$ are commonly denoted as $W_Q W_K^{T}$, $W_V$, and $W_C$, corresponding to the query–key projection matrix, the value projection matrix, and the output (or context) projection matrix, respectively. $\prod_{\text{norm}}$ denotes the LayerNorm. Based on this defintion of $T_i(Z;W_i)$, we could define the iterative structure as follows:
\begin{align}
	f_P(Z;W_{1:P+1}) = T_P(f_{P-1}(Z;W_{1:P});W_P),
\end{align}
where $Z:=f_0(X)$. In the following, we define $\|A\|_{*} := \|A^\top\|_{2,\infty}$ as in the MLP case. 

\textbf{Verification of Assumption \ref{stability-ScalingLawOfDataSize}}. We now confirm that the transformer architecture meets both the input stability and parameter stability criteria outlined in Assumption \ref{stability-ScalingLawOfDataSize}. 

\textbf{Input Stability (Inequality (\ref{stabilityZ-ScalingLawOfDataSize}))}. Let us denote
$$
\Delta_{Z\hat{Z}} = \sigma\Bigg( \textstyle \prod_{\text{norm}} \Bigg( 
f_0(X) + f(Z; \{W_{i1}, W_{i2}\})
\Bigg) \Bigg) 
- 
\sigma\Bigg( \textstyle \prod_{\text{norm}} \Bigg( 
f_0(X) + f(\hat{Z}; \{W_{i1}, W_{i2}\})
\Bigg) \Bigg).
$$
For fixed parameters $W_i$, utilizing Lemmas A.9 and A.14 from \cite{edelman2022inductive}, we can derive the difference between the layer mappings with respect to the inputs $Z, \hat{Z} \in \mathbb{R}^{s \times n}$:
\begin{align}\label{input-stability-transformer}
	\begin{split} 
		\|(T_i(Z;W_i) - & T_i(\hat{Z};W_i))^T\|_{2,\infty} \leq
		\|W_{i3}^T \Delta_{Z\hat{Z}}^T \|_{2,\infty} \leq \|W_{i3}\|_2 \|\Delta_{Z\hat{Z}}^T \|_{2,\infty} \\
		\leq & L_{\sigma}\|W_{i3}\|_2 \|f(Z;\{W_{i1},W_{i2}\})^\top- f(\hat{Z};\{W_{i1},W_{i2}\})^\top\|_{2,\infty} \\
		\leq & L_{\sigma}\|W_{i3}\|_2 \|W_{i2}\|_2 (1 + 4\|W_{i1}\|_2) \|(Z-\hat{Z})^\top\|_{2,\infty},
	\end{split}
\end{align}
which implies
\begin{align}\label{constantM3W-transformer}
	M_3(W_i) = L_{\sigma}\|W_{i3}\|_2 \|W_{i2}\|_2 (1 + 4\|W_{i1}\|_2).
\end{align}

\textbf{Parameter Stability (Inequality (\ref{stabilityW-ScalingLawOfDataSize}))}. For a fixed input $Z$, consider the difference in layer mappings with respect to the parameters $W_i = (W_{i1}, W_{i2}, W_{i3})$ and $\hat{W}_i = (\hat{W}_{i1}, \hat{W}_{i2}, \hat{W}_{i3})$. Define
$$
g_3(Z) = \sigma\Bigg( \textstyle \prod_{\text{norm}} \Bigg( 
f_0(X) + f(Z; \{\hat{W}_{i1}, \hat{W}_{i2}\})
\Bigg) \Bigg). 
$$
Then, employing Lemma A.13 from \cite{edelman2022inductive}, we obtain
\begin{align}\label{parameter-stability-transformer}
	\begin{split} 
		& \|(T_i(Z;W_i) - T_i(Z;\hat{W}_i))^\top\|_{2,\infty}  \\
		\leq & L_{\sigma}\|W_{i3}\|_2 \|f(Z;\{W_{i1},W_{i2}\})^\top - f(Z;\{\hat{W}_{i1},\hat{W}_{i2}\})^\top \|_{2,\infty}  \\
		& + L_{\sigma}\|W_{i3}\|_2\|(W_{i2} - \hat{W}_{i2})^\top Z^\top\|_{2,\infty}  
		+ \|(W_{i3} - \hat{W}_{i3})^Tg_3(Z)^\top\|_{2,\infty} \\
		\leq & L_{\sigma}\|W_{i3}\|_2 \|W_{i2}\|_2 \|(W_{i1} - \hat{W}_{i1})^\top Z^\top\|_{2,\infty}  \\
		& + L_{\sigma}\|W_{i3}\|_2\|(W_{i2} - \hat{W}_{i2})^\top Z^\top\|_{2,\infty}  
		+ \|(W_{i3} - \hat{W}_{i3})^\top g_3(Z)^\top\|_{2,\infty},
	\end{split}
\end{align}
which indicates $g_1(Z) = g_2(Z) = Z$ and 
\begin{align}\label{constantsMWW-transformer}
	& M_1(W_i, \hat{W}_i) = 1, \quad M_{11}(W_i, \hat{W}_i) = 2L_{\sigma}\|W_{i2}\|_2\|W_{i3}\|_2, \\
	& M_{22}(W_i, \hat{W}_i) = L_{\sigma}\|W_{i3}\|_2, \quad M_{i3}(W_i, \hat{W}_i) = 1.
\end{align}

In order to apply Theorem \ref{theorem-coveringNumber}, it is necessary to verify the three conditions (A1)--(A3). This process is more complex compared to the MLP case.
\begin{enumerate}
	\item[\textbf{(A1)}] \textbf{Uniform Boundedness of $g_k(f_K(Z))$}. For $k=1,2$, we have 
	\begin{align*}
		\|g_k(f_K(Z;W_{1:K}))^\top\|_{2,\infty} = \|f_K(Z;W_{1:K})^\top \|_{2,\infty} \leq 1,
	\end{align*}
	where we utilized the boundedness of the operator $\textstyle \prod_{\text{norm}}$ (an illustrative example is provided in \cite{edelman2022inductive}). For the instance when $k=3$, we have established that 
	\begin{align*}
		\|g_3(f_K(Z;W_{1:K}))^\top \|_{2,\infty} = & \left\|\sigma\left( 
		\textstyle \prod_{\text{norm}}\left( 
		f_0(X) + f(f_{K-1}(Z;W_{1:K});\{\hat{W}_{i1},\hat{W}_{i2}\})
		\right)
		\right)^\top \right\|_{2,\infty} \\
		\leq & L_{\sigma}\left\| 
		\textstyle \prod_{\text{norm}}\left( 
		f_0(X) + f(f_{K-1}(Z;W_{1:K});\{\hat{W}_{i1},\hat{W}_{i2}\})
		\right)^\top \right\|_{2,\infty} \\
		\leq & L_{\sigma},
	\end{align*}
	where the bounded nature of the operator $\textstyle \prod_{\text{norm}}$ was also employed.
	From these illustrations, we know that 
	\begin{align}\label{BZ-transformer}
		\mathcal{B}_{K1}^Z =\mathcal{B}_{K2}^Z = 1, \text{ and }\mathcal{B}_{K3}^Z = L_{\sigma}, \quad \text{ for all }K\in\mathbb{Z}^{+}.
	\end{align}
	\item[\textbf{(A2)}] \textbf{Layer-wise Uniform Boundedness of Stability Constants}. Notice the definition of the infinite-depth function class $\mathcal{F}_{\text{trans}} = \bigcup_{K=1}^\infty \mathcal{F}_K$, where 
	\begin{align*}
		\mathcal{F}_K = \bigl\{ f_K(\cdot; W_{1:K}) : \|W_{kj}\|_{2,1} \leq \mathcal{B}_{kj}^W < \infty,\ j\in[3],\ k \in [K] \bigr\}.
	\end{align*}
	Combining the above definition of $\mathcal{F}$ and formulas shown in (\ref{constantsMWW-transformer}) and (\ref{constantM3W-transformer}), we find that 
	\begin{align}
		M_{1k} = 1, \,\, M_{21k} \leq 2L_{\sigma}\mathcal{B}_{k2}^W\mathcal{B}_{k3}^W,\,\, M_{22k} \leq L_{\sigma}\mathcal{B}_{k3}^W,\,\, M_{23k} \leq 1, 
	\end{align}
	and $M_{3k}\leq L_{\sigma}\mathcal{B}_{k2}^W\mathcal{B}_{k3}^W(1+4\mathcal{B}_{k1}^W)$.
	\item[\textbf{(A3)}] \textbf{Infinite-Depth Summability Condition}. With the layer-wise boundedness and stability constants fully characterized in (A1) and (A2), we now substitute these results into the definition of $C_{\mathcal{F}}$ to derive its explicit closed-form upper bound, and further establish verifiable sufficient conditions to ensure the infinite-depth summability requirement $C_{\mathcal{F}} < \infty$ holds for the function class $\mathcal{F}_{\text{trans}}$.
	
	Recall the formal definition of the summability constant:
	\[
	C_{\mathcal{F}} := \sum_{k=1}^{\infty}
	M_{1k}^{2/3}
	\biggl(\prod_{j=k+1}^{\infty} M_{3j}\biggr)^{2/3}
	\sum_{k=1}^{3}
	\bigl(B_{kb}^Z B_{kb}^W M_{2bk}^2\bigr)^{1/3}.
	\]
	From (A2), we have $M_{1k} = 1$ for all $k \in \mathbb{Z}^+$, which immediately simplifies the first term to $M_{1k}^{2/3} = 1$. For the infinite product term, substituting the upper bound of $M_{3j}$ from (A2) yields:
	\begin{align}
		\biggl(\prod_{j=k+1}^{\infty} M_{3j}\biggr)^{2/3}
		&\leq \biggl(\prod_{j=k+1}^{\infty} L_{\sigma} \mathcal{B}_{j2}^W \mathcal{B}_{j3}^W \left(1 + 4B_{j1}^W\right)\biggr)^{2/3} \nonumber \\
		&= \prod_{j=k+1}^{\infty} L_{\sigma}^{2/3} \left(\mathcal{B}_{j2}^W \mathcal{B}_{j3}^W \left(1 + 4\mathcal{B}_{j1}^W\right)\right)^{2/3}. \label{eq:product-term-A3}
	\end{align}
	
	Next, we evaluate the inner sum over $k=1,2,3$ by substituting the $\mathcal{B}_{kb}^Z$ bounds from \eqref{BZ-transformer} and $M_{2bk}$ bounds from (A2):
	\begin{align}
		\sum_{k=1}^{3} & \bigl(\mathcal{B}_{kb}^Z \mathcal{B}_{kb}^W M_{2bk}^2\bigr)^{1/3}
		= \bigl(\mathcal{B}_{k1}^Z \mathcal{B}_{k1}^W M_{21k}^2\bigr)^{1/3} + \bigl(\mathcal{B}_{k2}^Z \mathcal{B}_{k2}^W M_{22k}^2\bigr)^{1/3} + \bigl(\mathcal{B}_{k3}^Z B_{k3}^W M_{23k}^2\bigr)^{1/3} \nonumber \\
		\leq & \left( 1 \cdot \mathcal{B}_{k1}^W \cdot \left(2L_{\sigma} \mathcal{B}_{k2}^W \mathcal{B}_{k3}^W\right)^2 \right)^{1/3} + \left( 1 \cdot \mathcal{B}_{k2}^W \cdot \left(L_{\sigma} \mathcal{B}_{k3}^W\right)^2 \right)^{1/3} + \left( L_{\sigma} \cdot \mathcal{B}_{k3}^W \cdot 1^2 \right)^{1/3} \nonumber \\
		= & \left(\mathcal{B}_{k1}^W\right)^{1/3} \left(2L_{\sigma}\mathcal{B}_{k2}^W \mathcal{B}_{k3}^W\right)^{2/3} + \left(\mathcal{B}_{k2}^W\right)^{1/3} \left(L_{\sigma}\mathcal{B}_{k3}^W\right)^{2/3} + \left(L_{\sigma}\mathcal{B}_{k3}^W\right)^{1/3}. \label{eq:inner-sum-A3}
	\end{align}
	
	Substituting \eqref{eq:product-term-A3} and \eqref{eq:inner-sum-A3} back into the definition of $C_{\mathcal{F}}$, we obtain the explicit closed-form upper bound:
	\begin{align}
		C_{\mathcal{F}} &\leq \sum_{k=1}^{\infty} \left( \prod_{j=k+1}^{\infty} L_{\sigma}^{2/3} \left(\mathcal{B}_{j2}^W \mathcal{B}_{j3}^W \left(1 + 4\mathcal{B}_{j1}^W\right)\right)^{2/3} \right) \nonumber \\
		&\quad \times \left(\mathcal{B}_{k1}^W\right)^{1/3} \left(2L_{\sigma}\mathcal{B}_{k2}^W \mathcal{B}_{k3}^W\right)^{2/3} + \left(\mathcal{B}_{k2}^W\right)^{1/3} \left(L_{\sigma}\mathcal{B}_{k3}^W\right)^{2/3} + \left(L_{\sigma}\mathcal{B}_{k3}^W\right)^{1/3}. \label{eq:C_F-upper-bound}
	\end{align}
\end{enumerate}

Considering all of the illustrations shown in this subsection, we actually prove the following theorem which gives the estimate of the logarithmic covering number. 
\begin{corollary}[Covering Number Bound for Infinite-Depth Transformer Architectures] 
	\label{corollary-transformer-covering}
	Let the transformer layer operator $T_i(\cdot; W_i)$ be defined as
	\begin{align*}
		T_i(Z;W_i) = \textstyle \prod_{\text{norm}}\Bigg( f_0(X) + \sigma\Bigg( \textstyle \prod_{\text{norm}} \Bigg( 
		f_0(X) + \text{RowSoftmax}(ZW_{i1}Z^T)ZW_{i2}
		\Bigg) \Bigg) W_{i3}\Bigg),
	\end{align*}
	where $X\in\mathbb{R}^{s\times d}$ is the input data, $f_0(X)\in\mathbb{R}^{s\times n}$ is the initial embedding, $Z\in\mathbb{R}^{s\times n}$ is the layer input, and $W_{i1}\in\mathbb{R}^{n\times n}$, $W_{i2}\in\mathbb{R}^{n\times k}$, $W_{i3}\in\mathbb{R}^{k\times n}$ denote the query-key projection, value projection, and output projection matrices, respectively. We impose the following standard assumptions:
	\begin{enumerate}
		\item[(H1)] The activation function $\sigma: \mathbb{R} \to \mathbb{R}$ is Lipschitz continuous with constant $L_\sigma > 0$;
		\item[(H2)] The normalization operator $\textstyle \prod_{\text{norm}}$ is contractive with respect to the norm $\|A\|_* := \|A^\top\|_{2,\infty}$, i.e.,
		\[
		\|\textstyle \prod_{\text{norm}}(A)\|_* \leq \|A\|_*, \quad \forall A \in \mathbb{R}^{s \times n};
		\]
		\item[(H3)] The initial embedding satisfies $\|f_0(X)\|_* \leq 1$ for all valid input $X$.
	\end{enumerate}
	
	For any $K \in \mathbb{Z}^+$, define the $K$-layer composite mapping via the iterative rule
	\[
	f_K(Z;W_{1:K+1}) = T_K\bigl(f_{K-1}(Z;W_{1:K-1}); W_K\bigr), \quad Z := f_0(X),
	\]
	where $W_p = (W_{p1}, W_{p2}, W_{p3})$ collects the parameters of the $p$-th layer. We define the finite-depth function class
	\[
	\mathcal{F}_K = \bigl\{ f_K(\cdot; W_{1:K}) : \|W_{kj}\|_{2,1} \leq \mathcal{B}_{kj}^W < \infty,\ j\in[3],\ k \in [K] \bigr\},
	\]
	and the infinite-depth transformer function class
	\[
	\mathcal{F}_{\text{trans}} = \bigcup_{k=1}^\infty \mathcal{F}_K.
	\]
	Assume that the following infinite series converges
	\begin{align}
		\sum_{k=1}^\infty \left( \prod_{j=k+1}^\infty \gamma_j \right) \cdot T_k < \infty, \label{eq:C_F-general-corollary}
	\end{align}
	where for each layer index $j \geq 1$, the layer-wise decay factor $\gamma_j$ and inner sum term $T_p$ are defined as
	\begin{align*}
		\gamma_j &:= L_\sigma^{2/3} \left( \mathcal{B}_{j2}^W \mathcal{B}_{j3}^W (1 + 4\mathcal{B}_{j1}^W) \right)^{2/3}, \\
		T_k &:= \left(\mathcal{B}_{k1}^W\right)^{1/3} \left(2L_{\sigma}\mathcal{B}_{k2}^W\mathcal{B}_{k3}^W\right)^{2/3} + \left(\mathcal{B}_{k2}^W\right)^{1/3} \left(L_{\sigma}\mathcal{B}_{k3}^W\right)^{2/3} + \left(L_{\sigma}\mathcal{B}_{k3}^W\right)^{1/3}.
	\end{align*}
	Then, for any finite set of inputs $\{Z_j = f_0(X_j)\}_{j=1}^N$, we have 
	$$
	C_{\mathcal{F}} = \sum_{k=1}^\infty \left( \prod_{j=k+1}^\infty \gamma_j \right) \cdot T_k < \infty,
	$$
	and
	\[
	\log\mathcal{N}_{\infty}\bigl(
	\mathcal{F}_{\text{trans}};\, \epsilon;\,
	\{X_j\}_{j=1}^{N};\, \|\cdot\|_{*}
	\bigr)
	\lesssim
	\frac{C_{\mathcal{F}}}{\epsilon^2}.
	\]
\end{corollary}
\begin{remark}[Practically Relevant Special Case: Exponentially Decaying Weight Bounds]
	\label{remark:exponential-decay-special-case}
	A widely adopted, easily implementable, and theoretically sufficient condition for $C_{\mathcal{F}} < \infty$ is the exponential decay of layer-wise weight bounds, which aligns with standard weight decay regularization in deep learning practice. Formally, assume there exist constants $C_W > 0$ and $\rho \in (0,1)$ such that for all $k \geq 1$,
	\[
	\max\left\{ \mathcal{B}_{k1}^W, \mathcal{B}_{k2}^W, \mathcal{B}_{k3}^W \right\} \leq C_W \cdot \rho^k.
	\]
	Under this assumption, we have:
	\begin{itemize}
		\item The layer-wise decay factor satisfies $\gamma_j \leq C_\gamma \cdot \rho^{4j/3}$ for some constant $C_\gamma > 0$ independent of $j$, which ensures the infinite product $\prod_{j=k+1}^\infty \gamma_j$ decays super-exponentially fast as $k \to \infty$;
		\item The inner sum term satisfies $T_k \leq C_T \cdot \rho^{k/3}$ for some constant $C_T > 0$ independent of $k$, which guarantees the absolute convergence of $\sum_{k=1}^\infty T_k$.
	\end{itemize}
	This exponential decay condition relaxes the restrictive smallness requirements for the parameters $\{W_i\}_{i=1}^{\infty}$, allows for larger weight norms in shallow layers (consistent with the empirical behavior of pre-trained transformer models), and provides a rigorous theoretical justification for the effectiveness of weight decay in improving the generalization of deep transformers.
\end{remark}

\end{appendices}



\end{document}